\documentclass{article}
\pdfoutput=1
\usepackage{iclr2027_conference,times}
\iclrfinalcopy 

\usepackage[utf8]{inputenc}
\usepackage[T1]{fontenc}
\usepackage{hyperref}
\usepackage{url}
\usepackage{fontawesome}
\hypersetup{colorlinks=true,allcolors=black,pdftitle={Missing-Data-Induced Phase Transitions in Spectral Partial Least Squares},pdfauthor={Anders Gj\o lbye, Emma Kargaard, Ida Kargaard, Lina Skerath, Hiba Nassar, Lars Kai Hansen}}
\usepackage{booktabs}
\usepackage{array}
\usepackage{amsfonts}
\usepackage{microtype}
\usepackage{xcolor}
\newcommand{\concept}[1]{\textbf{#1}}
\usepackage{amsmath}
\usepackage{amssymb}
\usepackage{mathtools}
\usepackage{amsthm}
\usepackage{graphicx}
\usepackage{flafter}
\usepackage{subcaption}

\usepackage{framed}
\definecolor{statementbar}{gray}{0.72}
\newenvironment{statementframe}{%
  \def\FrameCommand{{\color{statementbar}\vrule width 1.6pt}\hspace{7pt}}%
  \MakeFramed{\advance\hsize-\width\FrameRestore}}%
 {\endMakeFramed}
\theoremstyle{plain}
\newtheorem{theoreminner}{Theorem}[section]

\theoremstyle{definition}

\newtheorem{example}[theoreminner]{Example}
\theoremstyle{remark}
\newtheorem{remark}[theoreminner]{Remark}
\newcommand{\DeclareFramedTheorem}[1]{%
  \NewDocumentEnvironment{#1}{o}{%
    \begin{statementframe}%
    \IfNoValueTF{##1}{\begin{#1inner}}{\begin{#1inner}[##1]}}%
   {\end{#1inner}\end{statementframe}}}
\DeclareFramedTheorem{theorem}
\DeclareFramedTheorem{proposition}
\DeclareFramedTheorem{lemma}
\DeclareFramedTheorem{corollary}
\DeclareFramedTheorem{conjecture}
\DeclareFramedTheorem{reduction}
\DeclareFramedTheorem{definition}
\DeclareFramedTheorem{assumption}

\def\theoremautorefname{Theorem}
\def\lemmaautorefname{Lemma}
\def\propositionautorefname{Proposition}
\def\corollaryautorefname{Corollary}
\def\definitionautorefname{Definition}
\def\assumptionautorefname{Assumption}
\def\exampleautorefname{Example}
\def\remarkautorefname{Remark}
\def\conjectureautorefname{Conjecture}
\def\reductionautorefname{Proposition}

\makeatletter
\let\old@autoref\autoref
\newcommand{\@checkprefixandref}[2]{%
  \ifcsname @prefixmatch@#2\endcsname
    \csname @prefixmatch@#2\endcsname{#1}%
  \else
    \old@autoref{#1}%
  \fi
}
\expandafter\def\csname @prefixmatch@lem\endcsname#1{\hyperref[#1]{\lemmaautorefname~\ref*{#1}}}
\expandafter\def\csname @prefixmatch@app\endcsname#1{\hyperref[#1]{Appendix~\ref*{#1}}}
\expandafter\def\csname @prefixmatch@prop\endcsname#1{\hyperref[#1]{\propositionautorefname~\ref*{#1}}}
\expandafter\def\csname @prefixmatch@cor\endcsname#1{\hyperref[#1]{\corollaryautorefname~\ref*{#1}}}
\expandafter\def\csname @prefixmatch@def\endcsname#1{\hyperref[#1]{\definitionautorefname~\ref*{#1}}}
\expandafter\def\csname @prefixmatch@ass\endcsname#1{\hyperref[#1]{\assumptionautorefname~\ref*{#1}}}
\expandafter\def\csname @prefixmatch@rem\endcsname#1{\hyperref[#1]{\remarkautorefname~\ref*{#1}}}
\expandafter\def\csname @prefixmatch@conj\endcsname#1{\hyperref[#1]{\conjectureautorefname~\ref*{#1}}}
\expandafter\def\csname @prefixmatch@red\endcsname#1{\hyperref[#1]{\reductionautorefname~\ref*{#1}}}
\expandafter\def\csname @prefixmatch@ex\endcsname#1{\hyperref[#1]{\exampleautorefname~\ref*{#1}}}
\expandafter\def\csname @prefixmatch@thm\endcsname#1{\hyperref[#1]{\theoremautorefname~\ref*{#1}}}

\def\@extractprefix#1:#2\@nil{#1}
\def\@extractsuffix#1:#2\@nil{#2}

\renewcommand{\autoref}[1]{%
  \@ifcontainscolon{#1}{%
    \edef\@currentprefix{\@extractprefix#1:\@nil}%
    \@checkprefixandref{#1}{\@currentprefix}%
  }{%
    \old@autoref{#1}%
  }%
}

\def\@ifcontainscolon#1{%
  \@ifcontainscolon@aux#1:\@marker:\@nil
}
\def\@ifcontainscolon@aux#1:#2:#3\@nil{%
  \ifx\@marker#2%
    \expandafter\@secondoftwo
  \else
    \expandafter\@firstoftwo
  \fi
}
\makeatother

\usepackage{bm}

\newcommand{\E}{\mathbb{E}}

\newcommand{\R}{\mathbb{R}}
\newcommand{\Prob}{\mathbb{P}}
\newcommand{\Var}{\operatorname{Var}}

\newcommand{\Tr}{\operatorname{Tr}}
\newcommand{\diag}{\operatorname{diag}}

\newcommand{\range}{\operatorname{range}}
\newcommand{\supp}{\operatorname{supp}}
\newcommand{\spec}{\operatorname{spec}}
\DeclareMathOperator*{\stat}{stat}

\newcommand{\opnorm}[1]{\left\lVert #1 \right\rVert_{\mathrm{op}}}
\newcommand{\norm}[1]{\left\lVert #1 \right\rVert}
\newcommand{\abs}[1]{\left\lvert #1 \right\rvert}
\newcommand{\ip}[2]{{#1}{}^{\top}#2}
\newcommand{\avg}[1]{\left\langle #1 \right\rangle}
\newcommand{\one}{\mathbf{1}}
\newcommand{\dd}{\mathrm{d}}
\newcommand{\eqd}{\overset{d}{=}}
\newcommand{\toP}{\overset{P}{\longrightarrow}}
\newcommand{\iid}{\overset{\mathrm{iid}}{\sim}}

\newcommand{\Xs}{X_{\star}}
\newcommand{\Ys}{Y_{\star}}
\newcommand{\Dx}{D_x}
\newcommand{\Dy}{D_y}

\newcommand{\rhox}{\rho_x}
\newcommand{\rhoy}{\rho_y}
\newcommand{\rhoxy}{\rho}
\newcommand{\hp}{\widehat{p}}
\newcommand{\hu}{\widehat{u}}
\newcommand{\hv}{\widehat{v}}
\newcommand{\hz}{\widehat{z}}
\newcommand{\WG}{W_{\mathrm G}}
\newcommand{\han}{\widehat{a}}
\newcommand{\Cbar}{\bar{C}}
\newcommand{\Rres}{\mathcal{R}}

\newcommand{\mone}{\bar{m}_1}
\newcommand{\mcross}{\bar{m}_1^{\times}}
\newcommand{\mtwo}{\bar{m}_2}
\newcommand{\Dbar}{\bar{D}}
\newcommand{\Phix}{\Phi^{\times}}

\newcommand{\thmref}[1]{Theorem~\ref{#1}}

\newcounter{proofstep}

\newenvironment{compactitemize}
  {\begin{list}{\textbullet}{%
     \setlength{\itemsep}{1pt}\setlength{\parsep}{0pt}%
     \setlength{\topsep}{3pt}\setlength{\partopsep}{0pt}%
     \setlength{\leftmargin}{1.4em}\setlength{\labelwidth}{1em}%
     \setlength{\labelsep}{0.4em}}}
  {\end{list}}

\allowdisplaybreaks[1]
\title{Missing-Data-Induced Phase Transitions\\in Spectral Partial Least Squares}

\author{Anders Gj{\o}lbye\hspace{2em}Emma Kargaard\hspace{2em}Ida Kargaard\hspace{2em}Lina Skerath\\
\textbf{Hiba Nassar\hspace{2em}Lars Kai Hansen}\\[4pt]
Department of Applied Mathematics and Computer Science\\
Technical University of Denmark\\
\texttt{agjma@dtu.dk}}

\begin{document}

\maketitle
\lhead{Preprint}
\linespread{0.982}\selectfont 

\begin{abstract}
  Spectral partial least squares (PLS-SVD) estimates the directions shared by two views of the same samples. We analyze it using a rank-one regression model, with entries of both views missing completely at random and filled with zeros. Missing response entries weaken the signal, whereas missing design entries also tilt the recovered direction. The setting is the proportional regime, with Gaussian response noise and a design that is column-orthogonal before masking. For incoherent designs and planted directions, we prove that as the signal diverges, the squared overlap with the truth approaches $1$ on the response side but only a ceiling on the design side. This ceiling is strictly below $1$ whenever design entries are missing, unless samples of vanishing leverage carry the signal. For this estimator, a stronger signal therefore cannot remove the design-side error. Under equal leverage and an explicit spectral condition on the masked design, which we prove for designs taken from nested Hadamard matrices, we establish a phase transition and derive the recovery curves. Under the same conditions, the critical signal strength is exactly proportional to $\rho_y^{-1/2}$, with $\rho_y$ the fraction of response entries retained. A heuristic replica derivation reproduces the equations behind these curves, and simulations follow them, including on random orthogonal designs outside the verified class. In semi-synthetic experiments on real biological designs with unequal leverage, the predicted ceiling varies across directions, and high-signal recovery tracks it. These results bring the theory of PLS-SVD closer to what is known for PCA under homogeneous missingness and show that the loss of signal-to-noise ratio observed there does not, by itself, describe missing design entries.
\end{abstract}

\section{Introduction}
Partial least squares (PLS) and canonical correlation analysis (CCA) \citep{hotelling1936cca} are classical linear methods for extracting the low-dimensional structure shared by two views $(X,Y)$ of the same samples, and PLS remains widely used in high dimensions \citep{rosipal2006pls,boulesteix2007pls}. Its spectral form, PLS-SVD, takes the leading singular vectors of the cross-covariance $X^\top Y$ \citep{leger2025pls}. Two-view data are often incomplete, with entries lost to acquisition failures, measurement dropout, and heterogeneous protocols, and the simplest estimator fills them with zeros. Under homogeneous missing-completely-at-random (MCAR) masking \citep{little2002missing,schafer2002missing}, inverse-retention weighting only rescales the zero-filled cross-covariance, so it gives the same singular vectors. We ask how missing entries change the recovery threshold and accuracy of zero-filled PLS-SVD, and whether a stronger signal can compensate.\footnote{Code is available at \href{https://github.com/gjoelbye/Phase-Transitions-in-Spectral-PLS}{\faGithub\,gjoelbye/Phase-Transitions-in-Spectral-PLS}.}

For PCA under homogeneous missingness, the loss of entries acts as a reduced signal-to-noise ratio \citep{ipsen2019pca_missing_phase}, which a stronger signal restores. For PLS-SVD, the threshold and the overlaps are known only for complete data \citep{leger2025pls,mergny2025pls}. With missing entries, the mask on $X$ sets both the signal direction and the noise covariance of $X^\top Y$. The two are dependent, a case that the standard theory of low-rank perturbations \citep{benaychgeorges2012singular} does not cover, so the noise spectrum alone no longer determines recovery.

We show that the answer depends on which view is incomplete. We study a rank-one regression model, $\Ys=\theta(\Xs u_0)v_0^\top+Z$, with signal strength $\theta$, Gaussian noise $Z$, and a complete design $\Xs$ that is column-orthogonal. Entries of $\Xs$ and $\Ys$ are retained independently with probabilities $\rho_x$ and $\rho_y$, the dimensions grow proportionally, and $R_x^2$, $R_y^2$ denote the squared overlaps of the leading singular pair with $u_0$, $v_0$.

An exact decomposition of the masked cross-covariance (\autoref{red:deformed}) separates the roles of the two masks. Response masking attenuates the signal, as in PCA. Design masking also tilts the recovered direction. As the signal diverges, $R_y^2$ tends to $1$ while $R_x^2$ approaches a ceiling, which under equal leverage is $\kappa=\alpha_x\rho_x/(\alpha_x\rho_x+1-\rho_x)$ with $\alpha_x$ the limit of the sample-to-feature ratio $N/D_x$. At $\alpha_x=2$ and $\rho_x=0.3$, $\kappa\approx0.46$ (\autoref{fig:theory_validation}(a)). For this estimator, a stronger signal cannot remove the design-side error. Our contributions are as follows.
\begin{compactitemize}
\item \textbf{High-signal ceiling} (\autoref{thm:ceiling}). For incoherent designs and planted directions, we prove that the design-side overlap approaches a ceiling as the signal diverges, without any spectral assumptions. The ceiling is independent of $\rho_y$, depends on the direction through the leverage of the samples carrying the signal, and is below $1$ when design entries are missing unless that leverage vanishes.
\item \textbf{Phase transition} (\autoref{thm:finitesignal}). Under equal leverage and an explicit spectral condition on the masked design, we prove that both overlaps vanish below a critical signal strength $\theta_c$ and follow explicit curves above it. These curves depend on $\theta$ and $\rho_y$ only through $\theta\sqrt{\rho_y}$, so $\theta_c\propto\rho_y^{-1/2}$ exactly. We prove the condition for designs taken from nested Hadamard matrices (\autoref{cor:flatdesigns}).
\item \textbf{Replica derivation} (\autoref{app:replica}). A heuristic replica calculation reproduces the equations behind these curves via an independent route that makes no further assumptions about the design.
\end{compactitemize}
Experiments test these predictions. Simulations follow the predicted transition, recovery curves, and collapse in $\theta\sqrt{\rho_y}$, also on random orthogonal designs outside the verified class (\autoref{fig:theory_validation}). In semi-synthetic tests on real biological designs with unequal leverage, the direction-specific ceiling predicts high-signal recovery within $0.02$ on average, against $0.05$ to $0.08$ under equal leverage. We cover one signal component, homogeneous MCAR masking, and a column-orthogonal complete design ($D_x\le N$). The asymmetry is specific to this model, where only the response is noisy.

\begin{figure}[!t]
\centering
\includegraphics[width=\textwidth]{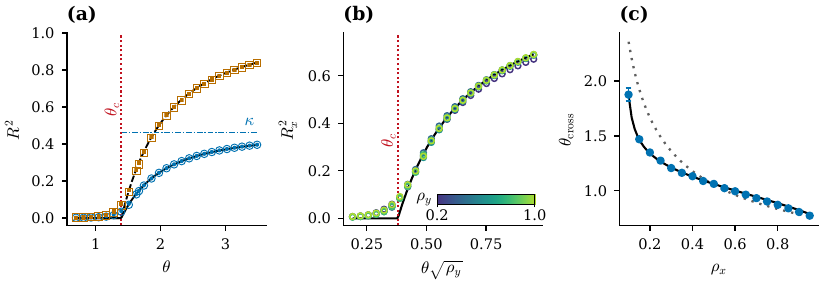}
\caption{%
\textbf{Recovery under masking in both views.} Black curves are the predictions of \autoref{thm:finitesignal}, and red marks the threshold $\theta_c$ in (a) and its rescaled value in (b). \textbf{(a)} Squared overlaps $R_x^2$ (blue) and $R_y^2$ (orange) against the signal strength $\theta$, on nested Walsh--Hadamard designs (filled markers, the verified class) and on random orthogonal designs (open markers, outside it), with $N=2048$, $\Dx=\Dy=1024$ and retentions $(\rhox,\rhoy)=(0.3,0.6)$, for which $\theta_c=1.39$. The dash-dot line is the ceiling $\kappa=0.462$ that $R_x^2$ approaches as $\theta$ grows. Markers average $100$ trials, and standard errors are below marker size. \textbf{(b)} $R_x^2$ at five response retentions $\rhoy$ falls on one curve against $\theta\sqrt{\rhoy}$ (random orthogonal designs, $N=2000$, $\Dx=400$, $\Dy=100$, $\rhox=0.7$, $60$ trials per point). \textbf{(c)} The statistic $\theta_{\mathrm{cross}}$, the mean of the signal strengths at which the mean $R_x^2$ curve last crosses three small positive levels, with $95\%$ bootstrap intervals (random orthogonal designs, $N=2000$, $\Dx=\Dy=1000$, $\rhoy=1$, $100$ trials per point). The same statistic is shown for our prediction (black) and for the isotropic reference (gray dotted), which treats masking as a loss of signal-to-noise ratio alone. A crossing of a positive level lies above $\theta_c$ (\autoref{app:belowiso}). }
\label{fig:theory_validation}
\end{figure}
\section{Related work}
\paragraph{Spectral phase transitions.} The transition of \autoref{thm:finitesignal} is of BBP type \citep{baik2005bbp}. Principal-component asymptotics cover isotropic \citep{bloemendal2016} and non-isotropic \citep{dingyang2021} covariance spikes, and \citet{benaychgeorges2012singular} give the singular-vector limits for rectangular low-rank perturbations independent of the noise. For fully observed spiked cross-covariance models, \citet{leger2025pls} and \citet{mergny2025pls} characterize the phase transition and the overlaps rigorously in the proportional regime, and \citet{bao2019} treat finite-rank canonical correlations. We add \emph{entry-wise missingness in both views} to a regression model with a noise-free design. The proofs condition on the design mask, which fixes the signal direction and the noise covariance. With noise in both views, this fails, so that case is open (\autoref{app:limitations}).

\paragraph{Missing data.} In the closest single-view result, \citet{ipsen2019pca_missing_phase} show that homogeneous missingness in PCA acts as a reduced signal-to-noise ratio. Under our spectral condition on the masked design, response masking has this interpretation in the limiting recovery curves, while design masking also changes the signal direction and the noise spectrum. Other spectral approaches correct for missingness, in PCA under heterogeneous missingness, debiased Gram-matrix dimension reduction and covariance estimation from partial observations \citep{zhu2022missingpca,ling2021missingdr,lounici2014}, whereas we analyze the uncorrected estimator. Signal-plus-noise reductions support prediction from incomplete data \citep{dobriban2020prediction}. Likelihood methods fit probabilistic PCA and PLS by EM \citep{dempster1977em,roweis1998empca,tipping1999ppca,elb2018ppls,elb2022po2pls}, and generalized CCA has been adapted to missing data \citep{vandevelden2012gcca_missing}.

\paragraph{Random-matrix tools and replica methods.} After masking, the design is a deterministic matrix plus independent noise with a variance profile. This is the setting of information-plus-noise models \citep{dozier2007,capitaine2014}, with deterministic equivalents for variance profiles \citep{hachem2007,hachem2013bilinear} and deterministic signals \citep{vallet2012}. Our proofs need spectral limits along the signal directions, with edge and outlier control. The noise edge is based on classical results for sample covariance matrices \citep{baisilverstein2010}. For the nested Hadamard class, information-plus-noise theory \citep{dozier2007,dozier2007analysis,baisilverstein2012} gives the limiting spectrum of the masked design and excludes eigenvalues outside its support. Replica methods derive thresholds and overlaps in spectral inference \citep{biehl1993unsupervised,hoyle2007multiple} and connect to message passing \citep{bayati2011amp,lesieur2017}. Our heuristic replica derivation follows this line (\autoref{sec:proof_sketch}).
\section{Theory of PLS with missing data}
\label{sec:pls}

\subsection{Problem setup}
\label{sec:setup}
Consider PLS-SVD in a two-view model where both views may have missing entries. Let $\Xs\in\R^{N\times\Dx}$ denote the complete $X$-view, a deterministic design whitened to have orthogonal columns,
\begin{equation}
  \Xs^\top\Xs = N I_{\Dx}.
  \label{eq:whitening}
\end{equation}
The normalization forces $\Dx\le N$. It also places dimension reduction and whitening before the missingness mechanism, which is a modeling choice. Let $u_0\in\R^{\Dx}$ and $v_0\in\R^{\Dy}$ be fixed unit vectors, and model the complete $Y$-view as
\begin{equation}
  \Ys = \theta (\Xs u_0)v_0^\top + Z,
  \quad
  Z_{ij}\iid\mathcal{N}(0,1),
  \quad
  Z\perp \Xs,
  \label{eq:model}
\end{equation}
where $\theta\ge0$ is the signal strength. Write $b:=\Xs u_0$ for the vector of latent scores, whose norm satisfies $\norm{b}^2=N$ exactly by \eqref{eq:whitening}.

Entries are removed by MCAR masks $S_x\in\{0,1\}^{N\times\Dx}$ and $S_y\in\{0,1\}^{N\times\Dy}$, whose entries are independent Bernoulli variables with fixed retention probabilities $\rhox,\rhoy\in(0,1]$. The masks are independent of each other and of $Z$, and we write $\rhoxy:=\rhox\rhoy$ for the joint retention. Missing entries are recorded as zeros, so the observed views are $X := S_x\odot\Xs$ and $Y := S_y\odot\Ys$. PLS-SVD returns the top singular vectors of the empirical cross-covariance. These vectors are invariant under positive rescaling of the matrix, so we work with the rescaled matrix and its top singular pair,
\begin{equation}
  C := \frac{1}{N\sqrt{\rhoxy}}X^\top Y \in\R^{\Dx\times\Dy},
  \qquad
  (\hu,\hv)\in\arg\max_{\norm{u}=\norm{v}=1} \ip{u}{Cv},
  \label{eq:crosscov}
\end{equation}
and measure recovery by the squared overlaps with the planted directions,
\begin{equation}
  R_x^2 := (\ip{\hu}{u_0})^2,
  \qquad
  R_y^2 := (\ip{\hv}{v_0})^2.
  \label{eq:overlapsdef}
\end{equation}

\paragraph{Asymptotic regime and assumptions.} All limits are taken in the proportional regime $N,\Dx,\Dy\to\infty$, in which the aspect ratios satisfy $N/\Dx\to\alpha_x\in[1,\infty)$ and $N/\Dy\to\alpha_y\in(0,\infty)$. All convergences hold in probability~($\toP$). Throughout, we assume the incoherence conditions of \autoref{ass:incoherence}. They require, for a constant $C_0$, that
\[
  \norm{u_0}_\infty\le C_0\frac{\log\Dx}{\sqrt{\Dx}},
  \quad
  \norm{v_0}_\infty\le C_0\frac{\log\Dy}{\sqrt{\Dy}},
  \quad
  \norm{\Xs}_{\max}\le C_0\log N,
  \quad
  \norm{b}_\infty\le C_0\log N,
\]
so that no single entry or score dominates the mask fluctuations. The decomposition of \autoref{sec:spiked_reduction} is algebraic, and the high-signal theorem of \autoref{sec:geometry} needs no assumption beyond incoherence. The finite-signal theorem of \autoref{sec:main_result} also uses a spectral condition on the masked design (\autoref{ass:closure}). We verify this condition for a class of flat orthogonal designs (\autoref{cor:flatdesigns}), where \concept{flat} (unimodular) means $|(\Xs)_{ki}|=1$ for every entry.

\subsection{Decomposing the cross-covariance}
\label{sec:spiked_reduction}
We first separate the masked signal from the response noise. Averaging the signal over the response mask then writes $C$ as a rank-one spike, product noise, and a remainder. The spike direction is the one PLS can recover at high signal, and the remainder is negligible in operator norm at fixed signal strength.

\begin{reduction}[Exact decomposition of the masked cross-covariance]
\label{red:deformed}
Define the normalized design $A$, standardized response noise $Q$, masked Gram matrix $G$ and signal vector $a_N$ by
\begin{equation}
  A:=\frac{X}{\sqrt{N\rhox}},
  \quad Q:=\frac{S_y\odot Z}{\sqrt{\rhoy}},
  \quad G:=A^\top A,
  \quad a_N:=\frac{A^\top b}{\sqrt N}.
\label{eq:pdefs}
\end{equation}
For every realization of the data and masks (\autoref{app:exact}),
\begin{equation}
  C = \underbrace{\theta\sqrt{\rhoy}\,a_Nv_0^\top\vphantom{\frac{1}{\sqrt N}}}_{\text{\normalfont Reduced signal}}
    + \underbrace{\frac{1}{\sqrt N}A^\top Q}_{\text{\normalfont Product noise}}
    + \,\Rres_N,
  \label{eq:barCN}
\end{equation}
where the reduced signal is the average of the masked signal over the response mask $S_y$, and the remainder is the fluctuation around this average,
\begin{equation}
  \Rres_N:=\frac{\theta}{N\sqrt{\rhoxy}}\,X^\top\!\diag(b)\big(S_y-\rhoy\one\one^\top\big)\diag(v_0).
  \label{eq:Rres}
\end{equation}
Conditionally on the design mask $S_x$, the quantities $(A,G,a_N)$ are fixed, and only $Q$ and $\Rres_N$ remain random. The entries of $Q$ are i.i.d., centered and of unit variance, so the columns of the product noise are independent with conditional covariance $G/N$ given $S_x$. At $\rhoy=1$ we have $\Rres_N=0$ and $Q=Z$, so $C$ is exactly a rank-one signal plus Gaussian noise with this column covariance.
\end{reduction}

Write $\Cbar_N:=C-\Rres_N$ for the reduced matrix. Its signal term is $\theta\sqrt{\rhoy}\,a_Nv_0^\top$, which we write as $s_N\han_Nv_0^\top$. Here $\han_N:=a_N/\norm{a_N}$ is the \concept{corrupted direction} and $s_N:=\theta\sqrt{\rhoy}\,\norm{a_N}$ is the \concept{effective spike strength}. Response retention scales this whole term, while the design mask changes the vector $a_N$ itself. The response mask still affects the remainder and the law of $Q$, whose entries are not Gaussian when $\rhoy<1$.

Under incoherence, $\E\opnorm{\Rres_N}\to0$ at each fixed $\theta$ (\autoref{lem:residual}). The finite-signal analysis therefore works with $\Cbar_N$, and \autoref{app:finitesignal} transfers the conclusions for $\Cbar_N$ back to $C$. The bound on $\Rres_N$ grows in proportion to $\theta$, so the remainder need not vanish when $\theta$ grows with $N$, and the high-signal proof works with $C/\theta$ instead.

\paragraph{Two effects of the design mask.} Averaging over both masks and the response noise gives $\E[C]=\theta\sqrt{\rhoxy}\,u_0v_0^\top$, but this mean hides two effects of the realized design mask. The direction $\han_N$ generally differs from $u_0$, and the conditional column covariance of the product noise is $G/N$ rather than $I_{\Dx}/N$. Although $\E[G]=I_{\Dx}$ exactly, $G$ need not approach $I$ in operator norm. For a flat design, the expected mean squared distance of the eigenvalues of $G$ from $1$ converges to $(1-\rhox^2)/\alpha_x$ (\autoref{lem:EGI} and \autoref{app:closurestatus}), so the eigenvalues stay spread around $1$. This spectral spread changes the singular-value bulk of the product noise. We call this the \concept{deformed bulk} and denote its limiting upper edge by $z_{\mathrm{edge}}$. Because $\han_N$ and $G$ share $S_x$, recovery depends on how the signal aligns with the noise eigendirections, as well as on the noise spectrum.

\subsection{Geometry and the ceiling}
\label{sec:geometry}
Increasing $\theta$ amplifies the part of the reduced signal along $u_0$ and the part orthogonal to it by the same factor. To quantify their relative energies, write $a_N=\sqrt{\rhox}\,u_0+g_N$, where $\sqrt{\rhox}\,u_0$ is the mean of $a_N$ over the design mask. The energy of the fluctuation $g_N$ is governed by the sample leverages and by how the score energy $b_k^2$ is spread over samples. The \concept{leverage} of sample $k$ is $\norm{(\Xs)_{k\cdot}}^2/N$. For a column-orthogonal design, it is the corresponding diagonal entry of the orthogonal projection onto the design's column space, so it measures the potential influence of that sample in a least-squares fit. The following coupling combines the two.

\begin{definition}[Score--leverage coupling]
\label{def:tau}
Define
\begin{equation}\label{eq:taudef}
  \tau_N:=\sum_{k=1}^N
  \underbrace{\frac{b_k^2}{N}\vphantom{\frac{\norm{(\Xs)_{k\cdot}}^2}{N}}}_{\text{\normalfont Score weight}}\;\cdot\;
  \underbrace{\frac{\norm{(\Xs)_{k\cdot}}^2}{N}}_{\text{\normalfont Sample leverage}}
  =\frac1{N^2}\sum_{k=1}^N b_k^2\norm{(\Xs)_{k\cdot}}^2.
\end{equation}
The score-energy weights sum to $1$, so $\tau_N\in[0,1]$ is the average sample leverage weighted by the signal's energy. Under asymptotically equal leverage, meaning $\norm{(\Xs)_{k\cdot}}^2=\Dx(1+o(1))$ uniformly in $k$, we have $\tau_N\to1/\alpha_x$.
\end{definition}

The geometry calculation (\autoref{prop:geometry}) gives
\[
  \ip{g_N}{u_0}\toP0,
  \qquad \norm{g_N}^2-(1-\rhox)\tau_N\toP0,
  \qquad \norm{a_N}^2-\big[\rhox+(1-\rhox)\tau_N\big]\toP0.
\]
Thus $\rhox$ is the asymptotic energy of $a_N$ along $u_0$, and $(1-\rhox)\tau_N$ is the corruption energy. The fraction of the energy that lies along $u_0$ is therefore $\kappa_{\tau,N}$, with
\begin{equation}\label{eq:kappa}
  \kappa_{\tau,N}:=\frac{\rhox}{\rhox+(1-\rhox)\tau_N},
  \qquad
  \kappa:=\frac{\alpha_x\rhox}{\alpha_x\rhox+1-\rhox}.
\end{equation}
\looseness=-1 Here $\kappa$ is the value of $\kappa_{\tau,N}$ at $\tau_N=1/\alpha_x$, which is its limit under asymptotically equal leverage. If $\tau_N\to\tau$, write $\kappa_\tau:=\rhox/[\rhox+(1-\rhox)\tau]$ for the limit of $\kappa_{\tau,N}$. Up to $o_P(1)$, $\kappa_{\tau,N}$ is the squared cosine of the angle between the corrupted direction $\han_N$ and $u_0$. For every design, it lies between $\rhox$ and $1$. Concentrating signal energy on high-leverage samples increases $\tau_N$ and lowers $\kappa_{\tau,N}$. Orthogonal columns fix only the average leverage, so they do not determine $\tau_N$. \autoref{ex:blocks} gives a column-orthogonal design with unequal block weights in which $\tau_N$ converges to a limit above $1/\alpha_x$. Under equal leverage, $\kappa$ increases from $\rhox$ at $\alpha_x=1$ to $1$ as $\alpha_x\to\infty$, so more samples per feature reduce the design masking cost.

\begin{theorem}[High-signal limit]
\label{thm:ceiling}
Suppose \autoref{ass:incoherence} holds. Then for every deterministic sequence $\theta_N\to\infty$,
\[
  R_x^2-\kappa_{\tau,N}\toP0,
  \qquad R_y^2\toP1,
\]
and $R_x^2\toP\kappa_\tau$ whenever $\tau_N\to\tau$. The high-signal limit on the $X$ side does not depend on $\rhoy$.
\end{theorem}

\paragraph{Origin of the ceiling.} The two masks act differently on the signal. The response mask acts independently on each response column. Conditionally on $S_x$, its fluctuation is the remainder $\Rres_N$, which is centered with independent columns. Under incoherence, the operator norm of $\Rres_N$ is $o_P(\theta)$, so at high signal the remainder does not move the leading singular vectors of $C$. The design mask enters through the factor $X^\top$ common to all response columns, so its fluctuation enters the signal as the single rank-one term $\theta\sqrt{\rhoy}\,g_Nv_0^\top$, which scales with $\theta$ exactly as the signal does. This term is nonzero because masking destroys the column orthogonality that gives $\Xs^\top b=Nu_0$ in the complete design. The factor $\sqrt{\rhoy}$ scales the useful and corrupted parts of $a_N$ equally, so $\rhoy$ drops out of the high-signal limit of $R_x^2$, which we call the ceiling.

For example, with equal leverage, $\alpha_x=2$ and $\rhox=0.3$ give $\tau=0.5$ and $\kappa=6/13$, so the ceiling is about $0.462$ even as $\theta\to\infty$. The theorem is a high-signal limit. It does not by itself assert a bound for every finite $\theta$.

\subsection{The finite-signal transition}
\label{sec:main_result}
At finite signal, recovery also depends on whether the spike separates from the deformed bulk. Whether it does depends on how the product noise $E:=A^\top Q/\sqrt N$ responds along the signal directions at a trial singular value $z>z_{\mathrm{edge}}$, because the corrupted direction and the noise share the design mask. We measure this response using three forms of the noise resolvents $R_1:=(z^2I-EE^\top)^{-1}$ and $R_2:=(z^2I-E^\top E)^{-1}$. The first two are built from $R_1$, one along the corrupted direction $\han_N$ and one between $u_0$ and $\han_N$, and the third is built from $R_2$ along the response direction $v_0$. These \concept{spectral responses} have deterministic limits,
\begin{equation}\label{eq:responses}
  z\,\ip{\han_N}{R_1\han_N}\toP m_1(z),
  \qquad
  z\,\ip{u_0}{R_1\han_N}\toP m_\times(z),
  \qquad
  z\,\ip{v_0}{R_2v_0}\toP m_2(z).
\end{equation}
Each response is a weighted sum over the singular modes of $E$. A mode with singular value $t$ enters with weight $z/(z^2-t^2)$ times its overlaps with the two vectors of the form, so modes with $t$ close to $z$ dominate. The limits depend only on $(\alpha_x,\alpha_y,\rhox)$, and not on $\rhoy$ or $\theta$. The convergence holds for each fixed $z$ under the condition on the masked Gram matrix $G$ stated next (\autoref{prop:rectangle}).

\paragraph{The spectral condition.}\looseness=-1 \autoref{ass:closure}, the spectral condition of the introduction, states that $G$ has a {deterministic equivalent}, that is, nonrandom limits for its edge and for three forms built from the resolvent $(wI-G)^{-1}$, where the spectral parameter $w$ lies above the spectrum of $G$. The assumption requires incoherence and, when $\rhox<1$, asymptotically equal leverage. The limits solve the \concept{scalar system} of \autoref{def:maskavg}. These explicit scalar equations also give the upper edge $w_G$ of the limiting spectrum of $G$ and the noise edge $z_{\mathrm{edge}}$. The assumption has two parts. \textbf{(i)} No eigenvalue of $G$ lies above the edge, in the sense that $\lambda_{\max}(G)\le w_G+o_P(1)$. \textbf{(ii)} The trace of $G(wI-G)^{-1}$ divided by $N$, the quadratic form of $(wI-G)^{-1}$ at $\han_N$ and its bilinear form between $u_0$ and $\han_N$ each converge to the value given by the scalar system, uniformly on a complex neighborhood of every compact subset of $(w_G,\infty)$. No limit is assumed for the quadratic form at $u_0$, because below threshold the proof needs only a finite-$N$ bound on that form, which part (i) provides.

\paragraph{The threshold.} The product $\Dbar_0(z):=m_1(z)m_2(z)$ is positive and decreases to zero above the edge. It is the directional analogue of the $D$-transform of \citet{benaychgeorges2012singular}. Under equal leverage, $s_N^2$ converges to the limiting spike energy $s_0^2:=\rhoxy\theta^2/\kappa$, and at $\rhox=1$ this holds for any column-orthogonal design. The separation threshold in the model parameter $\theta$ is
\begin{equation}\label{eq:boundary}
  \theta_c=\theta^*\sqrt{\kappa/\rhoxy},
  \qquad \theta^*:=\Dbar_0(z_{\mathrm{edge}})^{-1/2}.
\end{equation}
Here $\theta^*$ is the critical value of the effective spike strength $s_0$. For $\rhox<1$ the edge $z_{\mathrm{edge}}$ follows from the root of a single scalar equation. Once that root is known, \autoref{prop:curves} evaluates the threshold, the outlier location and the recovery curves explicitly.

\begin{theorem}[Finite-signal transition under the deterministic equivalent]
\label{thm:finitesignal}
In the model of \autoref{sec:setup} and under \autoref{ass:closure}, let $\theta\ne\theta_c$ be fixed, with $\theta_c$ as in \eqref{eq:boundary}. Let $\sigma_1(C)$ be the top singular value and $(\hu,\hv)$ any leading unit singular pair. If $\theta<\theta_c$, then
\[
  \sigma_1(C)\toP z_{\mathrm{edge}},
  \qquad R_x^2\toP0,
  \qquad R_y^2\toP0,
  \qquad (\ip{\hu}{\han_N})^2\toP0.
\]
If $\theta>\theta_c$, there is a unique root $\hz>z_{\mathrm{edge}}$ of $s_0^2\Dbar_0(\hz)=1$, and $\sigma_1(C)\toP\hz$. The recovery limits are
\begin{equation}\label{eq:overlaps}
  R_x^2\toP\frac{-2m_1}{s_0^2\Dbar_0'}
       \Big(\frac{m_\times}{m_1}\Big)^{\!2},
  \qquad
  R_y^2\toP\frac{-2m_2}{s_0^2\Dbar_0'}.
\end{equation}
All functions are evaluated at $\hz$, and the prime denotes differentiation with respect to $z$. The overlap $(\ip{\hu}{\han_N})^2$ converges to $-2m_1/(s_0^2\Dbar_0')$.
\end{theorem}

The $X$-side formula separates recovery of the corrupted direction from the factor $(m_\times/m_1)^2$ that carries it to $u_0$. This factor accounts for the dependence between $\han_N$ and $G$, and it generally differs from $\kappa$ at finite signal. As $\theta\to\infty$ the limits in \eqref{eq:overlaps} tend to $\kappa$ and to $1$, which are the equal-leverage values in \autoref{thm:ceiling} (\autoref{prop:curves}).

At the noise edge, $s_0^2\Dbar_0(z_{\mathrm{edge}})=(\theta/\theta_c)^2$, so the decreasing function $s_0^2\Dbar_0$ reaches $1$ outside the bulk exactly when $\theta>\theta_c$, and that crossing locates the outlier. Here $\theta$ is fixed as $N\to\infty$, so the critical window, in which $\theta$ approaches $\theta_c$ as $N$ grows, is not covered.

\begin{corollary}[Unmasked design and response-retention scaling]
\label{cor:closedforms}
At $\rhox=1$, \autoref{ass:closure} holds under incoherence, and the limits of \autoref{thm:finitesignal} have closed forms. With $s_0^2=\rhoy\theta^2$ and $\alpha_x\alpha_y s_0^4>1$,
\begin{equation}\label{eq:closedforms}
  R_x^2\toP\frac{\alpha_x\alpha_y s_0^4-1}{\alpha_y s_0^2(\alpha_x s_0^2+1)},
  \qquad
  R_y^2\toP\frac{\alpha_x\alpha_y s_0^4-1}{\alpha_x s_0^2(\alpha_y s_0^2+1)}.
\end{equation}
Both limits are zero below threshold, and $\theta_c=(\alpha_x\alpha_y)^{-1/4}/\sqrt{\rhoy}$. More generally, under \autoref{ass:closure}, fixing $(\alpha_x,\alpha_y,\rhox)$ gives the exact scaling $\theta_c(\rhoy)=\theta_c(1)\rhoy^{-1/2}$.
\end{corollary}

At fixed design retention, response retention enters the limiting curves only through $\theta\sqrt{\rhoy}$, so reducing $\rhoy$ rescales the signal axis. Design retention changes both the noise spectrum and the corrupted direction, so changing $\rhox$ has no analogous single rescaling.

Equation~\eqref{eq:boundary} separates three effects on the threshold. The factor $\rhoxy^{-1/2}$ compensates for signal attenuation by $\sqrt{\rhoxy}$. Bulk deformation raises $\theta^*$ above its isotropic value $(\alpha_x\alpha_y)^{-1/4}$ when $\rhox<1$ (\autoref{lem:edgeinflation}). The factor $\sqrt\kappa$ lowers the threshold in the original $\theta$ coordinate, since corruption also increases the spike energy $s_0^2=\rhoxy\theta^2/\kappa$. These effects compete, and the threshold can lie above or below the isotropic reference $(\alpha_x\alpha_y)^{-1/4}\rhoxy^{-1/2}$ of \autoref{cor:closedforms} with $\rhoy$ replaced by $\rhoxy$.

\paragraph{Verified design class.} \autoref{ass:closure} is proved for the following class. A sequence of flat column-orthogonal designs is \concept{nested} if, after signed permutations of rows and columns, each $\Xs$ is the $N\times\Dx$ corner of one fixed infinite unimodular array. The Sylvester/Walsh--Hadamard array at $N=2^m$ is the standard example.

\begin{corollary}[Nested flat orthogonal designs]
\label{cor:flatdesigns}
Let $(\Xs)$ be a nested sequence of flat orthogonal designs satisfying \autoref{ass:incoherence}. Then the deterministic equivalent of \autoref{ass:closure} holds, and every conclusion of \autoref{thm:finitesignal} holds unconditionally (\autoref{prop:flatclosure}).
\end{corollary}

In a flat design every row has squared norm $\Dx$, so every sample has exactly equal leverage. The verification in \autoref{app:rmt} uses flatness for the Gaussian comparison that gives part (ii), and nesting enters through the eigenvalue bound in part (i).

\begin{remark}[Result status by design class]
\label{rem:prediction}
\looseness=-1 The finite-signal theorem applies to every design satisfying \autoref{ass:closure}. \autoref{cor:closedforms} and \autoref{cor:flatdesigns} give the cases in which this assumption is verified. For QR/Haar designs with randomly drawn directions, incoherence holds with probability tending to $1$, and \autoref{ass:closure} is supported for them at $\rhox<1$ by numerical checks (\autoref{app:closurestatus}) and by the experiments of \autoref{sec:experiments}. Their finite-signal curves are therefore universality predictions. The curves for biological designs are equal-leverage approximations. We conjecture that the deterministic equivalent also holds for every column-orthogonal design with $\norm{\Xs}_{\max}=O(\log N)$ and asymptotically equal leverage under the standing assumptions. Without entrywise incoherence, orthogonal columns and equal leverage are insufficient (\autoref{app:closurestatus}). \autoref{app:roadmap} records the full result status.
\end{remark}

\subsection{Proof outline and replica derivation}
\label{sec:proof_sketch}
\looseness=-1 The finite-signal proof has four steps. Conditionally on $S_x$, exact rank-one identities express the outlying singular value and the overlaps of $\Cbar_N$ through quadratic forms of the noise resolvent (\autoref{app:maskavg}). A Gaussian comparison away from the spectrum then replaces these forms by finite-$N$ transforms of the realized $G$, and the resulting conditional equations for the outlier and the overlaps do not use \autoref{ass:closure}. Next, the edge theorems for sample covariance matrices bound the noise spectrum, and a sign change of the determinant in the rank-one identity locates the outlier. Finally, the deterministic equivalent turns the conditional equations into the scalar prediction (\autoref{app:finitesignal}). Below threshold, zero overlaps need a bound on the edge spectral subspace, since the absence of a separated root alone does not suffice. \autoref{app:rmt} verifies the deterministic equivalent for the nested flat class by the same comparison on the design side, combined with a result that excludes eigenvalues outside the support \citep{baisilverstein2012}. The high-signal theorem has a separate proof in \autoref{app:ceiling}.

A heuristic replica calculation at a fixed design and a fixed design mask reaches the same conditional equations by an independent route (\autoref{app:replica}), so \autoref{ass:closure} enters only when the limit in $N$ is taken. The calculation is carried out for the matrix obtained from $\Cbar_N$ by replacing the product noise with Gaussian noise of the same conditional column covariance $G/N$. At $\rhoy=1$ this matrix has the law of $\Cbar_N$ conditionally on $S_x$ (\autoref{cor:specialise}). For $\rhoy<1$ it matches $\Cbar_N$ only in conditional mean and covariance, so it serves as a comparison model. Replica symmetry, the continuation in the number of replicas, the exchange of limits, and the choice of saddle are not justified rigorously.

\section{Experiments}
\label{sec:experiments}
\label{sec:protocol}
\label{sec:results}

\looseness=-1 The experiments test the recovery curves and the threshold of \autoref{thm:finitesignal}, on nested Walsh--Hadamard designs in the verified class and on random orthogonal designs outside the verified class. They then test the ceiling of \autoref{thm:ceiling} on engineered and biological designs with unequal leverage, and compare zero-filling with imputation.

\paragraph{Setup.} In the synthetic experiments we construct a complete column-orthogonal design, generate the response from \eqref{eq:model} and apply independent MCAR masks to both views. We take the top singular pair of the matrix $C$ in \eqref{eq:crosscov}, without re-centering or re-whitening the masked views, and report the squared overlaps of \eqref{eq:overlapsdef}. A random orthogonal design, the QR/Haar design of \autoref{rem:prediction}, is the orthonormal factor in the QR decomposition of a Gaussian matrix \citep{mezzadri2007}, rescaled so that $\Xs^\top\Xs=NI_{\Dx}$. \autoref{app:robustness} lists all sizes, retentions, signal grids and trial counts.

\begin{figure}[!t]
\centering
\includegraphics[width=\textwidth]{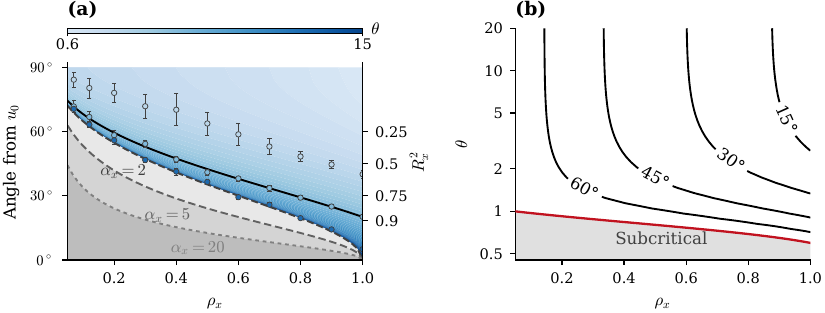}
\caption{%
\textbf{Angular recovery under design masking.} \textbf{(a)} Angle between $\hu$ and $u_0$ against $\rhox$ for one Walsh--Hadamard design with $N=1024$, $\Dx=512$, $\Dy=256$ and $\rhoy=1$, so that $\alpha_x=2$. The right axis gives the corresponding $R_x^2$. The colored field is filled between predicted angle curves for $\theta$ from $0.6$ to $15$ and is colored by $\theta$, and the black curve is $\theta=2$. Points are means over $5$ masks with $12$ noise draws each at $\theta\in\{1,2,8\}$, and bars are one standard deviation. The upper edge of each gray band is the limiting angle $\arccos\sqrt{\kappa}$, with darker bands for larger $\alpha_x$. \textbf{(b)} Predicted angle in the $(\rhox,\theta)$ plane. Contours show equal angles, and red marks $\theta_c$. }
\label{fig:geometry}
\end{figure}

\paragraph{Transition and recovery curves.}\looseness=-1 In \autoref{fig:theory_validation}(a), both overlaps are small below $\theta_c$ and follow the predicted curves above it. The Walsh--Hadamard and random orthogonal designs track the same curves. Over signal strengths above $1.1\,\theta_c$ the mean absolute error of the prediction is below $0.005$ for both overlaps on both designs. As $\theta$ grows, $R_y^2$ continues towards $1$ while $R_x^2$ bends towards the ceiling $\kappa=0.46$ and reaches $0.40$ at the largest signal strength shown. In panel~(b), on random orthogonal designs, the curves for five response retentions fall on one curve when plotted against $\theta\sqrt{\rhoy}$. This is the scaling of \autoref{cor:closedforms}, which is a prediction for these designs (\autoref{rem:prediction}).

\paragraph{Location of the threshold.} At finite $N$ the overlaps rise smoothly through $\theta_c$, so the threshold cannot be read off a measured curve. We instead record $\theta_{\mathrm{cross}}$, the signal strength at which the mean $R_x^2$ curve last crosses a small positive level, averaged over the levels $0.05$, $0.10$ and $0.15$, and we apply the same rule to each predicted curve (\autoref{app:configs}). \autoref{fig:theory_validation}(c) shows this statistic on random orthogonal designs, together with the prediction of \autoref{thm:finitesignal} and the isotropic reference of \autoref{sec:main_result}. The reference treats masking as a loss of signal-to-noise ratio alone, as for PCA. For $\rhox\le0.45$ the measured crossings lie up to a quarter below the reference, and every $95\%$ bootstrap interval contains our prediction and excludes the reference. 
At $\rhox=0.5$, where the two predictions nearly meet, the measurements are consistent with both, and \autoref{app:belowiso} covers the higher retentions.

\paragraph{Angular recovery.} \autoref{fig:geometry} expresses recovery as the angle between $\hu$ and $u_0$. Just above $\theta_c$, $R_x^2$ is positive while this angle can remain large. As $\theta$ grows, the angle tends not to zero but to the limiting angle $\arccos\sqrt{\kappa}$, which is the tilt of the recovered direction. For a flat design with $\alpha_x=2$, the limiting angle is below $30^\circ$ only if $\rhox>0.60$ and below $15^\circ$ only if $\rhox>0.87$, and both retentions decrease as $\alpha_x$ grows.

\paragraph{Direction-specific ceilings.} \autoref{thm:ceiling} needs incoherence but not equal leverage, and its ceiling $\kappa_{\tau,N}$ then depends on the direction. Whitening equalizes the column covariance but not the sample leverage, so two directions in the same design can have different couplings and different ceilings. In the engineered designs, rows are scaled by one of two factors before the columns are whitened. The biological designs come from TCGA BRCA, with $N=873$ samples of RNA-seq and methylation \citep{tcga2012brca}, and from PBMC Multiome, with $N=5000$ cells of scRNA-seq and scATAC-seq \citep{tenxgenomics2021pbmcmultiome}. Both views are reduced to $200$ principal components and whitened, and the RNA view is the design. Each of the eight leading singular pairs of the complete-data cross-covariance is planted in turn as $(u_0,v_0)$. The response is then generated from \eqref{eq:model}, and the whitened design and the response are masked (\autoref{app:preprocessing}). The design is real and the response is synthetic, so the truth is known.

In \autoref{fig:ceilings}(a) in \autoref{app:ceilings}, $R_x^2$ at $30\,\theta_c$ follows $\kappa_{\tau,N}$ on the engineered designs, on the biological directions and on random directions in the biological designs. The equal-leverage value $\kappa$ is the same for every direction of a design, so it cannot describe this spread. For the biological directions, the mean absolute deviation from $\kappa_{\tau,N}$ is $0.017$ on TCGA and $0.018$ on PBMC, against $0.078$ and $0.046$ from $\kappa$. Panel~(b) tests the finite-signal curves on the engineered designs. These curves assume equal leverage, and their error grows as the leverage becomes more unequal. Agreement with the ceiling at high signal therefore does not justify extending the equal-leverage recovery curve to designs with unequal leverage (\autoref{app:ceilings}).

\paragraph{Comparison with imputation.} The ceiling is proved for the zero-filled estimator, and other estimators may behave differently. In a synthetic comparison at $\rhox=\rhoy=0.3$, nuclear-norm imputation tuned on held-out observed entries raises the mean $R_x^2$ over signal strengths above $1.1\,\theta_c$ from $0.331$ to $0.361$ (\autoref{app:baselines}). At $\rhox=\rhoy=0.7$ the two agree within $0.01$.

\section{Conclusion}
\label{sec:conclusion}
We characterized recovery of zero-filled PLS-SVD under MCAR masking in a regression model with a complete column-orthogonal design. The design-side ceiling depends on design retention and score--leverage coupling, while finite-signal recovery also depends on the deformed bulk and its spectral responses. Design-side error can therefore persist as the response signal grows. The recovery curves and ceilings together distinguish informative recovery from accurate direction estimation. Under the deterministic equivalent, verified for nested flat orthogonal designs, response retention enters the limiting curves through $\theta\sqrt{\rho_y}$. The reference $\kappa_{\tau,N}$ can be evaluated for different candidate directions within a complete design (\autoref{tab:kappasurvey}). Limitations and open problems are discussed in \autoref{app:limitations}.
\label{page:main-text-end}

\clearpage
\linespread{1}\selectfont
\subsubsection*{Acknowledgments}
This work was supported by the Novo Nordisk Foundation grant NNF22OC0076907 ``Cognitive spaces -- Next generation explainability'' and the Pioneer Centre for AI, DNRF grant number P1. It was also supported by the Danish Data Science Academy, which is funded by the Novo Nordisk Foundation (NNF21SA0069429) and VILLUM FONDEN (40516).

\phantomsection
\label{page:references-start}
\bibliographystyle{iclr2027_conference}
\bibliography{references}

\newpage
\appendix
\numberwithin{equation}{section}
\section{Notation and proof guide}
\label{app:guide}

\subsection{Notation and conventions}
\label{app:notation}

Stars mark complete-data objects ($\Xs$, $\Ys$). Hats identify the estimated leading singular vectors ($\hu$, $\hv$), the corrupted direction ($\han_N$), and the supercritical outlier location ($\hz$). Bars mark conditional or reduced objects ($\Cbar_N$, $\mone$, $\Dbar$), and a subscript $N$ marks a finite-$N$ quantity. Its deterministic limit, when it exists, drops the subscript ($\tau_N\to\tau$) or replaces it by $0$ ($s_N\to s_0$, $\psi_N\to\psi_0$). The bare $\psi$, $\Phi$, $\Phix$ and $T$ are $\psi_0$, $\Phi_0$, $\Phix_0$ and $T_0$ written in the branch coordinate $\omega$ (\autoref{def:maskavg}, \autoref{lem:branch}). The exact aspect ratios are $\alpha_{x,N}=N/\Dx$ and $\alpha_{y,N}=N/\Dy$; the unsubscripted $\alpha_x,\alpha_y$ are their limits. Design-side objects are indexed by $\Dx$ and response-side objects by $\Dy$. An upright subscript $\mathrm G$ marks Gaussian objects ($\WG$, $C_{\mathrm G}$, $g_{\mathrm G,N}$), and an italic subscript $G$ marks the fold of the branch and the top of the Gram support ($\omega_G$, $w_G$).

Matrices are generally uppercase italic, while vectors and scalars are lowercase italic and distinguished by context. Calligraphic letters denote structured operators, events, contours, or residuals. Bold uppercase letters are reserved for Hermitized block objects, apart from $\mathbf 1$ for the all-ones vector.

\begin{center}
\footnotesize
\begin{tabular}{@{}l>{\raggedright\arraybackslash}p{0.70\linewidth}@{}}
\toprule
\multicolumn{2}{@{}l}{\emph{Model and finite-$N$ random quantities}}\\
\midrule
$\Xs,\Ys,u_0,v_0,b,\theta$ & complete views, planted directions, latent scores
$b=\Xs u_0$, signal strength; \eqref{eq:whitening}, \eqref{eq:model} \\
$S_x,S_y,\rhox,\rhoy,\rhoxy$ & masks, retentions, joint retention
$\rhoxy=\rhox\rhoy$ \\
$\alpha_{x,N},\alpha_{y,N}$ & exact finite-$N$ aspect ratios
$N/\Dx,N/\Dy$ \\
$X,Y,C,\hu,\hv,R_x^2,R_y^2$ & masked views, cross-covariance, leading pair,
squared overlaps; \eqref{eq:crosscov}, \eqref{eq:overlapsdef} \\
$A,Q,G,a_N,\han_N,s_N$ & normalized design and noise, masked Gram matrix,
signal vector, corrupted direction, effective spike strength;
\eqref{eq:pdefs}, \autoref{sec:spiked_reduction} \\
$\WG$ & standard Gaussian matrix of the Gaussian ensemble
$C_{\mathrm G}=s_N\han_Nv_0^\top+G^{1/2}\WG/\sqrt N$; \eqref{eq:singleview}, \eqref{eq:gaussens},
\autoref{app:replica} \\
$p_N,\Delta_N,\Rres_N$ & projected signal vector, signal residual, remainder;
\eqref{eq:pNdef}, \eqref{eq:signalresidual-B}, \eqref{eq:Rres} \\
$E,g_N,\hp,D_N$ & product noise $A^\top Q/\sqrt N$, fluctuation vector $a_N-\sqrt{\rhox}\,u_0$, unit vector $p_N/\norm{p_N}$, normalized residual $\Delta_N/\theta$;
\autoref{sec:main_result}, \autoref{prop:geometry}, \autoref{prop:deltabound} \\
$\psi_N,\Phi_N,\Phix_N$ & conditional resolvent functionals;
\eqref{eq:resolvents}, \autoref{app:conditionaltransforms} \\
$\tau_N,\kappa_{\tau,N}$ & score--leverage coupling and the level it implies;
\eqref{eq:taudef}, \eqref{eq:kappa} \\
\midrule
\multicolumn{2}{@{}l}{\emph{Deterministic limits, design side}}\\
\midrule
$\alpha_x,\alpha_y,\tau,\kappa_\tau,\kappa,n_a$ & aspect ratios, limiting
coupling, ceilings, squared spike norm $n_a=\rhox/\kappa$;
\autoref{def:tau}, \eqref{eq:kappa}, \autoref{def:maskavg} \\
$g,h,e_1,e_2,\xi$ & scalar system, $\xi=\sqrt w$;
\eqref{eq:ghsystem} \\
$\omega,W,\Lambda$ & branch coordinate $\omega=e_1/d_1$, Gram coordinate map
$w=W(\omega)$, Schur determinant $\Lambda=d_1d_2-\rhox$ with
$d_i=\xi-e_i$; \autoref{app:onevar} \\
$\omega_G,w_G$ & fold of the branch and top of the Gram support;
\eqref{eq:foldcubic} \\
$\omega_R,w_R,z_{\mathrm{edge}}$ & response stationary point and the top edge
of the deformed bulk; \eqref{eq:responseedge} \\
$\theta^*,\theta_c$ & detection level and threshold; \eqref{eq:boundary} \\
$\omega_\theta,\hz$ & supercritical root coordinate and outlier location
$\hz=\sqrt{T(\omega_\theta)}$; \autoref{prop:curves} \\
\midrule
\multicolumn{2}{@{}l}{\emph{Response quantities and their limits}}\\
\midrule
$z,\delta,K$ & response coordinate, fixed point, and ratio $K=w/z$;
\eqref{eq:fixedpoint} \\
$\mone,\mtwo,\mcross,\Dbar$ & conditional secular transforms;
\eqref{eq:transforms} \\
$m_1,m_2,m_\times,\Dbar_0$ & limiting transforms and their product;
\eqref{eq:limtransforms}, \autoref{app:curves} \\
$T_N,w_{R,N},z_{\mathrm{edge},N},\mathbf R^{\det}_N$ & finite-$N$ response map, its minimizer, the response edge, and the deterministic-equivalent resolvent; \autoref{app:baseline} \\
$\hz_N,\mathfrak t_N$ & finite-$N$ conditional secular root and design tolerance sequence; \autoref{thm:outlier}, \autoref{lem:reduction} \\
$\varsigma$ & edge-window half-width; \autoref{app:subcritical} \\
\midrule
\multicolumn{2}{@{}l}{\emph{Replica quantities (\autoref{app:replica})}}\\
\midrule
$\beta,n,u^a,v^a$ & inverse temperature, replica number, replicated pair \\
$q_v,r_y,\chi,\zeta$ & overlap order parameters, response susceptibility, and
spherical multiplier \\
\bottomrule
\end{tabular}
\end{center}

\newpage
\subsection{Proof roadmap and result status}
\label{app:roadmap}

\autoref{app:exact} gives the exact finite-$N$ reduction, and \autoref{app:geometry} describes the corrupted direction and effective spike strength. \autoref{app:ceiling} uses this geometry to prove the high-signal limit under the standing incoherence assumptions, without the deterministic equivalent. For finite signal, \autoref{app:conditionaltransforms} and \autoref{app:secular} provide conditional transforms and exact fixed-noise identities under their stated spectral conditions. The scalar comparison in Appendices~\ref{app:mde}--\ref{app:baseline} gives the candidate curves of \autoref{app:curves}.

\autoref{app:finitesignal} uses Gaussian comparison at fixed height, exterior continuation, spike geometry, and residual control to transfer these curves to the estimator under \autoref{ass:closure}, with separate arguments above and below threshold. See Appendices~\ref{app:noiseedge}--\ref{app:subcritical} for the response estimates and the inputs they take from \autoref{ass:closure}. \autoref{app:rmt} verifies \autoref{ass:closure} for the stated nested flat class. The replica derivation in \autoref{app:replica} reproduces the conditional equations for a covariance-matched Gaussian ensemble; it does not establish the deterministic equivalent.

\begin{table}[!ht]
\centering
\caption{Result status by setting.}\label{tab:status}
\small
\renewcommand{\arraystretch}{1.06}
\begin{tabular}{@{}>{\raggedright\arraybackslash}p{0.30\linewidth}>{\raggedright\arraybackslash}p{0.64\linewidth}@{}}
\toprule
\textbf{Setting} & \textbf{Status} \\
\midrule
General column-orthogonal designs under \autoref{ass:incoherence} &
\textbf{High-signal theorem.} \autoref{thm:ceiling}; $\tau_N\to\tau$ is needed
only for the limiting form. \\
\addlinespace[3pt]
Every design satisfying \autoref{ass:closure} &
\textbf{Finite-signal theorem.} \autoref{thm:finitesignal} gives the threshold
and curves under the standing hypotheses. \\
\addlinespace[3pt]
Nested flat orthogonal class of \autoref{cor:flatdesigns} &
\textbf{Unconditional theorem.} \autoref{app:rmt} verifies
\autoref{ass:closure} for the stated
nested class, including signed row and column permutations of nested
Sylvester/Walsh--Hadamard corners under \autoref{ass:incoherence}. \\
\addlinespace[3pt]
$\rhox=1$, any column-orthogonal design under \autoref{ass:incoherence} &
\textbf{Exact deterministic equivalent and closed forms.} \autoref{ass:closure} holds exactly, and
\autoref{prop:limits}(i) gives the closed forms. \\
\midrule
QR/Haar and other designs without a verified deterministic equivalent &
\textbf{Universality prediction.} Scope and evidence are recorded in
\autoref{rem:prediction} and \autoref{app:closurestatus}. \\
\addlinespace[3pt]
Biological designs in the semi-synthetic experiments &
\textbf{Equal-leverage approximation.} The theorem is not asserted for these designs. \\
\addlinespace[3pt]
Covariance-matched Gaussian replica ensemble, conditional on $(\Xs,S_x)$ under
(R1)--(R5) &
\textbf{Replica derivation.} It reproduces the conditional equations, not
\autoref{ass:closure};
its conditional law equals the reduced one only at $\rhoy=1$. \\
\bottomrule
\end{tabular}
\end{table}

\clearpage
\section{Exact conditional reduction}
\label{app:exact}

\autoref{red:deformed} is proved here by direct computation from \eqref{eq:model} and \eqref{eq:crosscov}. The reduction and its single-view specialization are exact at every finite $N$ and use \eqref{eq:whitening} only where stated. We then bound the residual at fixed signal strength under incoherence.

\begin{proposition}[Dual masking]
\label{prop:dualmask}
Let $q:=v_0\odot v_0$ and define the projected signal vector
\begin{equation}\label{eq:pNdef}
  p_N:=\frac{X^\top\!\big(b\odot S_yq\big)}{N\sqrt{\rhoxy}}.
\end{equation}
For $\rhoy\le1$ and every realization of the data and masks,
\begin{equation}
  C = \underbrace{\theta p_Nv_0^\top\vphantom{\frac{1}{\sqrt N}}}_{\text{\normalfont Projected signal}}
    + \underbrace{\Delta_N\vphantom{\frac{1}{\sqrt N}}}_{\text{\normalfont Signal residual}}
    + \underbrace{\frac{1}{\sqrt N}A^\top Q}_{\text{\normalfont Product noise}},
  \label{eq:dualmask-B}
\end{equation}
where the signal residual is
\begin{equation}
  \Delta_N := \frac{\theta}{N\sqrt{\rhoxy}}\,
    X^\top\!\diag(b)\big[S_y-(S_yq)\one^\top\big]\diag(v_0),
  \qquad \Delta_Nv_0=0.
  \label{eq:signalresidual-B}
\end{equation}
The noise term obeys the exact pathwise identity
\[
  E \;=\; \frac{1}{\sqrt N}\,A^\top Q,
  \qquad
  A=\frac{X}{\sqrt{N\rhox}},
  \qquad
  Q=\frac{S_y\odot Z}{\sqrt{\rhoy}},
\]
with $A$ and $Q$ as in \eqref{eq:pdefs}. The entries of $Q$ are i.i.d., centered, of unit variance, with $\E\abs{Q_{kj}}^m=\rhoy^{1-m/2}\,\E\abs{Z_{11}}^m$ for every $m\ge1$. Conditionally on $(S_x,S_y)$ the columns of $E$ are independent centered Gaussians of covariance $\widetilde G_j/N$, where $\widetilde G_j:=X^\top\diag\big((S_y)_{\cdot j}\big)X/(N\rhoxy)$, and $\E[\widetilde G_j\mid S_x]=G$ for every $j$. The splitting is unique: $\Delta_Nv_0=0$ pathwise, so $\theta p_Nv_0^\top$ is the part of the conditional mean acting along $v_0$ and $\Delta_N$ the part acting on its orthogonal complement.
\end{proposition}

\begin{proof}
Substituting $\Ys=\theta bv_0^\top+Z$ into $Y=S_y\odot\Ys$ and using the entrywise identity $S_y\odot(bv_0^\top)=\diag(b)\,S_y\diag(v_0)$, the cross-covariance \eqref{eq:crosscov} splits pathwise as $C=C_{\mathrm{sig}}+E$ with
\[
  C_{\mathrm{sig}}:=\frac{\theta}{N\sqrt{\rhoxy}}\,X^\top\diag(b)\,S_y\diag(v_0),
  \qquad
  E:=\frac{X^\top(S_y\odot Z)}{N\sqrt{\rhoxy}} .
\]
With $A$ and $Q$ as in \eqref{eq:pdefs}, $A^\top Q/\sqrt N=X^\top(S_y\odot Z)/(N\sqrt{\rhox\rhoy})=E$, the claimed identity. The moment identity for $Q$ follows from $\abs{Q_{kj}}^m=(S_y)_{kj}\abs{Z_{kj}}^m\rhoy^{-m/2}$ for a $0/1$ mask. Conditionally on $(S_x,S_y)$ only $Z$ is random. Writing $D_j:=\diag\big((S_y)_{\cdot j}\big)$, the $j$-th column of $E$ is $X^\top D_jZ_{\cdot j}/(N\sqrt{\rhoxy})$, a fixed linear image of $Z_{\cdot j}\sim\mathcal N(0,I_N)$: centered Gaussian with covariance $X^\top D_j^{\,2}X/(N^2\rhoxy)=\widetilde G_j/N$, since $D_j^2=D_j$ for a $0/1$ mask, and independent across $j$ because $E_{\cdot j}$ depends on $Z_{\cdot j}$ alone. Averaging over the response mask at fixed $S_x$, $\E D_j=\rhoy I_N$ gives $\E[\widetilde G_j\mid S_x]=\rhoy\,X^\top X/(N\rhoxy)=G$ for every $j$.

The signal is the conditional mean $C_{\mathrm{sig}}=\E[C\mid S_x,S_y]$. Its action along the planted response direction is $C_{\mathrm{sig}}v_0=\theta p_N$, since $\diag(v_0)v_0=q$ and $\diag(b)S_yq=b\odot S_yq$. Projecting onto this direction and its orthogonal complement therefore gives
\[
  C_{\mathrm{sig}}
  =C_{\mathrm{sig}}v_0v_0^\top+C_{\mathrm{sig}}(I-v_0v_0^\top)
  =\theta p_Nv_0^\top+\Delta_N.
\]
Expanding the second projection recovers \eqref{eq:signalresidual-B}, because $S_y\diag(v_0)(I-v_0v_0^\top) =[S_y-(S_yq)\one^\top]\diag(v_0)$. It also gives $\Delta_Nv_0=0$ directly. These two projections uniquely determine the signal matrices, including when $\theta=0$.
\end{proof}

Averaging the signal over the response mask gives $\E[S_yq]=\rhoy\one$, hence the exact identities
\[
  \E[p_N\mid S_x]=\sqrt{\rhoy}\,a_N,
  \qquad
  \E[C\mid S_x]=\theta\sqrt{\rhoy}\,a_Nv_0^\top.
\]
Since $\E[A]=\sqrt{\rhox/N}\,\Xs$, whitening gives $\E[a_N]=\sqrt{\rhox}\,u_0$ and therefore $\E[C]=\theta\sqrt{\rhoxy}\,u_0v_0^\top$. These expectations average the masks and noise with the complete design fixed; the decompositions \eqref{eq:dualmask-B} and \eqref{eq:barCN} remain pathwise identities.

\begin{corollary}[Single-view specialization]
\label{cor:specialise}
At $\rhoy=1$, $S_y\equiv\one$, hence $\Delta_N=0$, $p_N=a_N$, and $\widetilde G_j=G$ for every $j$. Conditionally on $S_x$,
\begin{equation}
  C \;\eqd\; \theta\,a_Nv_0^\top + \frac{1}{\sqrt N}\,G^{1/2}\WG,
  \qquad (\WG)_{ij}\iid\mathcal N(0,1).
  \label{eq:singleview}
\end{equation}
\end{corollary}

\begin{proof}
At $S_y\equiv\one$, $(S_yq)_k=\norm{v_0}^2=1$, so $(S_yq)\one^\top=S_y$ and the bracket in \eqref{eq:signalresidual-B} vanishes, giving $\Delta_N=0$. With $\rhoxy=\rhox$, each remaining pathwise claim follows by substitution in \eqref{eq:pdefs} and \eqref{eq:pNdef}. For \eqref{eq:singleview}, $S_y\equiv\one$ is deterministic, so conditioning on $S_x$ coincides with conditioning on $(S_x,S_y)$, and \autoref{prop:dualmask} makes the columns of $E=A^\top Q/\sqrt N$ independent centered Gaussians of covariance $\widetilde G_j/N=G/N$. The columns of $G^{1/2}\WG/\sqrt N$ have the same joint law, a centered Gaussian vector being determined by its covariance, invertible or not. Adding the $S_x$-measurable signal term $\theta\,a_Nv_0^\top$ gives \eqref{eq:singleview}.
\end{proof}

The residual $\Rres_N$ of \eqref{eq:barCN} collects the two mask-fluctuation terms of \eqref{eq:dualmask-B},
\begin{equation}\label{eq:Rressplit}
  \Rres_N=\theta\big(p_N-\sqrt{\rhoy}\,a_N\big)v_0^\top+\Delta_N,
\end{equation}
since $\theta p_Nv_0^\top+\Delta_N=C_{\mathrm{sig}}$ and $\Rres_N=C_{\mathrm{sig}}-\theta\sqrt{\rhoy}\,a_Nv_0^\top$. The next lemma bounds it in expectation.

\begin{lemma}[$L^1$ residual bound]
\label{lem:residual}
Under \autoref{ass:incoherence}, at each fixed $\theta$, $\E\opnorm{\Rres_N}\to0$.
\end{lemma}

\begin{proof}[Proof of \autoref{lem:residual}]
The residual $\Rres_N$ is the fluctuation of the signal about its mean given $S_x$. Subtracting $\theta\sqrt{\rhoy}\,a_Nv_0^\top$ from $C_{\mathrm{sig}}$ gives the exact identity
\[
  \Rres_N=\frac{\theta}{N\sqrt{\rhoxy}}\,
  X^\top\diag(b)(S_y-\rhoy\one\one^\top)\diag(v_0).
\]
With the complete design and directions fixed, $X$ depends only on $S_x$ and is independent of $S_y$. Moreover, $\E\opnorm X=O(\sqrt N)$ by \autoref{lem:opnorm}. The centered response mask has independent entries bounded by $1$, so the independent-entry norm bound \citep{bandeira2016} gives $\E\opnorm{S_y-\rhoy\one\one^\top}=O(\sqrt N+\sqrt{\Dy})=O(\sqrt N)$. Submultiplicativity and independence now give
\[
\begin{aligned}
  \E\opnorm{\Rres_N}
  &\le\frac{\theta\,\norm b_\infty\norm{v_0}_\infty}{N\sqrt{\rhoxy}}\,
     \E\opnorm X\,\E\opnorm{S_y-\rhoy\one\one^\top}\\
  &\le C\theta\,\norm b_\infty\norm{v_0}_\infty.
\end{aligned}
\]
Here $C$ is independent of $N$ and $\theta$, with the retentions and limiting aspect ratios fixed. Thus \autoref{ass:incoherence}(ii) and (iv) imply
\[
  \E\opnorm{\Rres_N}
  =O\Big(\theta\,\frac{\log N\log\Dy}{\sqrt{\Dy}}\Big)\longrightarrow0
\]
at each fixed $\theta$. At $\theta=0$, the residual is identically zero.
\end{proof}

The stated convergence concerns fixed $\theta$. The bound is linear in $\theta$, which also controls the normalized residual used in \autoref{app:ceiling}. Weyl's inequality bounds the change in every singular value by $\opnorm{\Rres_N}$, and Wedin's theorem \citep{wedin1972} controls the corresponding singular pair when the outlier is simple and separated from the remaining spectrum. Dependence between $\Rres_N$ and $A^\top Q$ is irrelevant to this deterministic perturbation step. The free-energy form of the same comparison, used by the replica derivation, is stated at the start of \autoref{app:replica}.

\clearpage
\section{Geometry of the corrupted direction}
\label{app:geometry}

We set out the finite-$N$ geometry of the corrupted direction: the incoherence assumption (\autoref{ass:incoherence}), the exact moments of $a_N$ (\autoref{prop:geometry}), and the operator-norm bounds (\autoref{lem:opnorm}, \autoref{prop:deltabound}) used throughout.

\begin{assumption}[Polylogarithmic incoherence]
\label{ass:incoherence}
There is a constant $C_0$, independent of $N$, such that
\begin{equation}\label{eq:incoherence}
\begin{aligned}
  \text{(i)}&\ \ \norm{u_0}_\infty\le C_0\frac{\log\Dx}{\sqrt{\Dx}},
  &\qquad
  \text{(iii)}&\ \ \norm{\Xs}_{\max}\le C_0\log N,\\[2pt]
  \text{(ii)}&\ \ \norm{v_0}_\infty\le C_0\frac{\log\Dy}{\sqrt{\Dy}},
  &\qquad
  \text{(iv)}&\ \ \norm{b}_\infty\le C_0\log N.
\end{aligned}
\end{equation}
\end{assumption}

The four conditions prevent any one feature coordinate, response coordinate, design entry, or sample score from dominating the mask fluctuations. For $(u_0,v_0)$ drawn independently of $\Xs$ and of each other, uniformly on their unit spheres or as normalized i.i.d.\ sign vectors, and $\Xs$ Haar or Walsh--Hadamard, a deterministic $C_0$ exists for which all four conditions hold jointly with probability tending to $1$. Every later bound depends on the data only through $C_0$, so the conclusions apply whenever these conditions hold. The sample-score condition \textbf{(iv)} is separate because \textbf{(i)} and \textbf{(iii)} do not control $b=\Xs u_0$.

\begin{lemma}[Centering of the masked Gram matrix]
\label{lem:EGI}
$\E[G]=I_{\Dx}$ exactly, for every $N$, and for a flat design $\E G_{ii'}^2=(1-\rhox^2)/N$ exactly for every $i\ne i'$.
\end{lemma}

\begin{proof}[Proof of \autoref{lem:EGI}]
Entrywise $X_{ki}=(S_x)_{ki}(\Xs)_{ki}$, so
\[
  G_{ii'}=\frac{1}{N\rhox}\sum_{k=1}^N(S_x)_{ki}(S_x)_{ki'}(\Xs)_{ki}(\Xs)_{ki'},
\]
and only the mask is random. On the diagonal the mask entries are $0$ or $1$, so $(S_x)_{ki}^2=(S_x)_{ki}$ and $\E[G_{ii}]=N^{-1}\sum_k(\Xs)_{ki}^2=1$, the $i$-th diagonal entry of \eqref{eq:whitening}. Off the diagonal the two mask entries lie in distinct columns and are independent, so $\E[G_{ii'}]=\rhox N^{-1}(\Xs^\top\Xs)_{ii'}=0$, again by \eqref{eq:whitening}. For the second moment write $c_k:=(\Xs)_{ki}(\Xs)_{ki'}$. Distinct rows are independent and $(S_x)_{ki}^2(S_x)_{ki'}^2=(S_x)_{ki}(S_x)_{ki'}$, so
\[
  \E G_{ii'}^2
  =\frac{1}{N^2\rhox^2}\Big(\rhox^4\sum_{k\ne l}c_kc_l+\rhox^2\sum_kc_k^2\Big).
\]
A flat design has $c_k^2=1$, hence $\sum_kc_k^2=N$, while $\sum_kc_k=(\Xs^\top\Xs)_{ii'}=0$ gives $\sum_{k\ne l}c_kc_l=-N$, and the two terms combine to the stated value.
\end{proof}

\begin{proposition}[Geometry of the corrupted direction]
\label{prop:geometry}
At every finite $N$, $\E[a_N]=\sqrt{\rhox}\,u_0$, the coordinates of $a_N$ are independent, and $\E\norm{g_N}^2=(1-\rhox)\tau_N$. Under \autoref{ass:incoherence},
\[
  \ip{g_N}{u_0}\toP0,
  \qquad
  \norm{g_N}^2-(1-\rhox)\tau_N\toP0,
  \qquad
  \norm{p_N-\sqrt{\rhoy}\,a_N}\toP0,
\]
so $\ip{p_N}{u_0}\toP\sqrt{\rhoxy}$ and $\norm{p_N}^2-(\rhoxy+\rhoy(1-\rhox)\tau_N)\toP0$. Let $\hp:=p_N/\norm{p_N}$ when $p_N\ne0$, with any fixed unit vector used on the event $p_N=0$, whose probability tends to zero. Then $(\ip{\hp}{u_0})^2-\kappa_{\tau,N}\toP0$. If moreover $\tau_N\to\tau$, the limits are $(1-\rhox)\tau$, $\rhoxy+\rhoy(1-\rhox)\tau$ and $\kappa_\tau$.
\end{proposition}

\begin{proof}
The exact moments follow from the Bernoulli mask distribution. The coordinate
\[
  (a_N)_i=\frac{1}{N\sqrt{\rhox}}\sum_{k=1}^N(S_x)_{ki}(\Xs)_{ki}\,b_k
\]
depends on the $i$-th mask column alone, and distinct columns of $S_x$ are disjoint collections of independent entries, so the coordinates are independent. Taking expectations, $\E[(S_x)_{ki}]=\rhox$ and $\Xs^\top b=\Xs^\top\Xs u_0=Nu_0$ give $\E[a_N]=\sqrt{\rhox}\,u_0$, so $g_N=a_N-\sqrt{\rhox}\,u_0$ is centered with independent coordinates. Summing the mask variances $\rhox(1-\rhox)$ down column $i$ yields
\[
  \sigma_i^2:=\Var\big((a_N)_i\big)=\frac{1-\rhox}{N^2}\sum_k(\Xs)_{ki}^2b_k^2,
\]
and summing over $i$, with $\sum_i(\Xs)_{ki}^2=\norm{(\Xs)_{k\cdot}}^2$, gives $\E\norm{g_N}^2=\sum_i\sigma_i^2=(1-\rhox)\tau_N$ by \eqref{eq:taudef}.

For the overlap, coordinate independence, the bound $b_k^2\le\norm b_\infty^2$, and the exact column norm $\sum_k(\Xs)_{ki}^2=N$ from \eqref{eq:whitening} give, uniformly in $i$,
\[
  \sigma_i^2\le\frac{(1-\rhox)\norm b_\infty^2}{N},
  \qquad
  \Var\big(\ip{g_N}{u_0}\big)=\sum_iu_{0,i}^2\sigma_i^2
  \le\frac{(1-\rhox)\,C_0^2\log^2\!N}{N}\longrightarrow0,
\]
using $\norm{u_0}=1$ and \autoref{ass:incoherence}(iv). Since $g_N$ is centered, Chebyshev's inequality gives, for each fixed $\varepsilon>0$,
\[
  \Prob\big(\abs{\ip{g_N}{u_0}}>\varepsilon\big)
  \le\frac{\Var\big(\ip{g_N}{u_0}\big)}{\varepsilon^2}
  \le\frac{(1-\rhox)\,C_0^2\log^2\!N}{\varepsilon^2N}
  \longrightarrow0,
\]
which is $\ip{g_N}{u_0}\toP0$.

For the norm, the squares $(g_N)_i^2$ are independent, so
\[
  \Var\big(\norm{g_N}^2\big)
  =\sum_i\Var\big((g_N)_i^2\big)\le\sum_i\E[(g_N)_i^4] .
\]
Each coordinate is a sum over $k$ of independent centered summands
\[
  \eta_{ki}=\frac{\big((S_x)_{ki}-\rhox\big)(\Xs)_{ki}b_k}{N\sqrt{\rhox}}.
\]
By \autoref{ass:incoherence}(iii)--(iv), these satisfy $\abs{\eta_{ki}}\le\eta_{\max}:=C_0^2\log^2\!N/(N\sqrt{\rhox})$. In the fourth-moment expansion, independence and centering eliminate every term with an index that appears only once. The surviving terms have either four equal indices or two pairs of equal indices. There are three pairings when the distinct indices are summed as ordered pairs, giving
\[
  \E[(g_N)_i^4]
  =\sum_k\E\,\eta_{ki}^4+3\sum_{k\neq l}\E\,\eta_{ki}^2\,\E\,\eta_{li}^2
  \le\eta_{\max}^2\sigma_i^2+3\sigma_i^4 .
\]
For the inequality, $\eta_{ki}^4\le\eta_{\max}^2\eta_{ki}^2$ bounds the first sum, and the pair sum is at most $(\sum_k\E\eta_{ki}^2)^2=\sigma_i^4$. As shown above, $\max_i\sigma_i^2\le(1-\rhox)C_0^2\log^2\!N/N$. The projector $\Xs\Xs^\top/N$ has diagonal entries in $[0,1]$, so $\norm{(\Xs)_{k\cdot}}^2\le N$. Together with $\norm b^2=N$, this gives $\sum_i\sigma_i^2=(1-\rhox)\tau_N\le1-\rhox$. Therefore
\[
  \Var\big(\norm{g_N}^2\big)
  \le\Big(3\max_i\sigma_i^2+\eta_{\max}^2\Big)\sum_i\sigma_i^2
  =O\Big(\frac{\log^2N}{N}\Big),
\]
so another application of Chebyshev's inequality gives $\norm{g_N}^2-(1-\rhox)\tau_N\toP0$. The finite-$N$ norm targets are
\[
  d_{a,N}:=\rhox+(1-\rhox)\tau_N,
  \qquad
  d_{p,N}:=\rhoy d_{a,N}=\rhoxy+\rhoy(1-\rhox)\tau_N.
\]
The projector identity used above gives $0\le\tau_N\le1$. Consequently, $\rhox\le d_{a,N}\le1$ and $d_{p,N}\ge\rhoxy>0$, uniformly in $N$. Moreover,
\[
  \ip{a_N}{u_0}-\sqrt{\rhox}=\ip{g_N}{u_0}\toP0,
\]
and
\[
  \norm{a_N}^2-d_{a,N}
  =2\sqrt{\rhox}\,\ip{g_N}{u_0}
   +\big\{\norm{g_N}^2-(1-\rhox)\tau_N\big\}\toP0.
\]
Since $d_{a,N}\le1$, this also gives $\norm{a_N}=O_P(1)$ without requiring convergence of $\tau_N$.

For the dual reduction, $\sqrt{\rhoy}\,a_N=X^\top(\rhoy b)/(N\sqrt{\rhoxy})$, so \eqref{eq:pdefs} and \eqref{eq:pNdef} give the pathwise identity
\[
  p_N-\sqrt{\rhoy}\,a_N
  =\frac{X^\top\big(b\odot(S_yq-\rhoy\one)\big)}{N\sqrt{\rhoxy}} .
\]
Only the weighted norm of the response-mask fluctuation is needed here. Each entry $(S_yq)_k-\rhoy$ has mean zero and variance $\rhoy(1-\rhoy)\sum_jq_j^2$. Using $\norm b^2=N$ and $\sum_jq_j^2=\sum_jv_{0,j}^4\le\norm{v_0}_\infty^2$ gives
\[
  \E\norm{b\odot(S_yq-\rhoy\one)}^2
  =N\rhoy(1-\rhoy)\sum_jv_{0,j}^4
  \le N\norm{v_0}_\infty^2.
\]
Markov's inequality therefore gives $\norm{b\odot(S_yq-\rhoy\one)}=O_P(\sqrt N\,\norm{v_0}_\infty)$. Combining this with $\opnorm X=O_P(\sqrt N)$ from \autoref{lem:opnorm},
\[
  \norm{p_N-\sqrt{\rhoy}\,a_N}
  \le\frac{\opnorm X}{N\sqrt{\rhoxy}}\,\norm{b\odot(S_yq-\rhoy\one)}
  =O_P(\norm{v_0}_\infty)\toP0
\]
by \autoref{ass:incoherence}(ii).

The response-mask remainder $r_N:=p_N-\sqrt{\rhoy}\,a_N$ satisfies $\norm{r_N}\toP0$ by the preceding calculation. Hence
\[
  \abs{\ip{p_N}{u_0}-\sqrt{\rhoy}\,\ip{a_N}{u_0}}
  \le\norm{r_N},
\]
so $\ip{p_N}{u_0}\toP\sqrt{\rhoxy}$. Moreover,
\[
\begin{aligned}
  \norm{p_N}^2-d_{p,N}
  ={}&\rhoy\big\{\norm{a_N}^2-d_{a,N}\big\}
    +2\sqrt{\rhoy}\,\ip{a_N}{r_N}+\norm{r_N}^2
    \toP0,
\end{aligned}
\]
using $\norm{a_N}=O_P(1)$. On an event whose probability tends to $1$,
\[
  \norm{p_N}^2\ge d_{p,N}/2\ge\rhoxy/2.
\]
On this event, $p_N\ne0$ and
\[
  \big(\ip{\hp}{u_0}\big)^2
  =\frac{\ip{p_N}{u_0}^2}{\norm{p_N}^2}.
\]
Therefore
\[
\begin{aligned}
  \left|
  \frac{\ip{p_N}{u_0}^2}{\norm{p_N}^2}
  -\frac{\rhoxy}{d_{p,N}}
  \right|
  &\le
  \frac{\abs{\ip{p_N}{u_0}^2-\rhoxy}}{\norm{p_N}^2}
  +\frac{\rhoxy\,\abs{\norm{p_N}^2-d_{p,N}}}
  {\norm{p_N}^2d_{p,N}}
  \toP0.
\end{aligned}
\]
The comparison term is exactly
\[
  \frac{\rhoxy}{d_{p,N}}
  =\frac{\rhox}{\rhox+(1-\rhox)\tau_N}
  =\kappa_{\tau,N}.
\]
If $\tau_N\to\tau$, continuous mapping gives the three limiting conclusions in the proposition.
\end{proof}

\begin{lemma}[Operator norm of the masked design]
\label{lem:opnorm}
Under \autoref{ass:incoherence}(iii), $\E\opnorm{X}=O(\sqrt N)$ and $\opnorm{X}=O_P(\sqrt N)$, hence $\opnorm{G}=O_P(1)$.
\end{lemma}

Whitening fixes the column variances of $X$ exactly, which removes the logarithmic loss a generic sub-Gaussian bound would incur.

\begin{proof}
We write $X=\rhox\Xs+\Xi$ with $\Xi:=(S_x-\rhox)\odot\Xs$, the fluctuation $\sqrt{N\rhox}\,\Theta$ of \autoref{app:maskavg}. Whitening makes every singular value of $\Xs$ equal to $\sqrt N$, so $\opnorm{\rhox\Xs}=\rhox\sqrt N$. The entries of $\Xi$ are independent, centered, bounded by $\norm{\Xs}_{\max}\le C_0\log N$, with $\Var(\Xi_{ki})=\rhox(1-\rhox)(\Xs)_{ki}^2$, so the sharp bound for matrices with independent centered entries \citep{bandeira2016} applies:
\[
  \E\opnorm{\Xi}\;\lesssim\;\sigma_{\mathrm{col}}+\sigma_{\mathrm{row}}
  +\sigma_*\sqrt{\log(N\wedge\Dx)} .
\]
Here $\sigma_{\mathrm{col}}^2=\max_i\sum_k\Var(\Xi_{ki})=\rhox(1-\rhox)N$ exactly, by the diagonal of \eqref{eq:whitening}. Next, $\sigma_{\mathrm{row}}^2=\rhox(1-\rhox)\max_k\norm{(\Xs)_{k\cdot}}^2\le\rhox(1-\rhox)N$, since $\Xs\Xs^\top/N$ is an orthogonal projector with diagonal in $[0,1]$. Finally $\sigma_*\le C_0\log N$. Hence $\E\opnorm{\Xi}=O(\sqrt N)$, since the logarithmic term is subdominant, and $\E\opnorm X\le\rhox\sqrt N+\E\opnorm\Xi\le C\sqrt N$. Markov's inequality now gives, for every $M>0$,
\[
  \Prob(\opnorm X>M\sqrt N)\le\frac{\E\opnorm X}{M\sqrt N}\le\frac{C}{M}.
\]
This proves $\opnorm X=O_P(\sqrt N)$ and consequently $\opnorm G=\opnorm X^2/(N\rhox)=O_P(1)$.
\end{proof}

\begin{proposition}[Residual bound]
\label{prop:deltabound}
Under \autoref{ass:incoherence}, at fixed $\theta$,
\[
  \opnorm{\Delta_N} \;=\; O_P\big(\theta\,\norm b_\infty\norm{v_0}_\infty\big)
  \;=\; O_P\Big(\theta\,\frac{\log N\log\Dy}{\sqrt{\Dy}}\Big)\longrightarrow0 .
\]
\end{proposition}

\begin{proof}
The residual $\Delta_N$ is the part of $\Rres_N$ acting on the orthogonal complement of $v_0$. Indeed, \eqref{eq:Rressplit} and $\Delta_Nv_0=0$ give
\[
  \Rres_N(I-v_0v_0^\top)=\Delta_N(I-v_0v_0^\top)=\Delta_N.
\]
Since $\opnorm{I-v_0v_0^\top}\le1$, the estimate in the proof of \autoref{lem:residual} yields
\[
  \E\opnorm{\Delta_N}\le\E\opnorm{\Rres_N}
  \le C\theta\,\norm b_\infty\norm{v_0}_\infty.
\]
Markov's inequality gives the stated $O_P$ bound. Inserting \autoref{ass:incoherence}(ii) and (iv) gives the rate, which vanishes at fixed $\theta$ since $\Dy\asymp N$.
\end{proof}

\paragraph{The normalized residual.} For $\theta>0$, let $D_N:=\Delta_N(\theta)/\theta$. The exact expression \eqref{eq:signalresidual-B} shows that $D_N$ does not depend on $\theta$. Since the constant in the preceding estimate is also independent of $\theta$,
\[
  \E\opnorm{D_N}\le C\norm b_\infty\norm{v_0}_\infty\longrightarrow0.
\]
Thus $\Delta_N(\theta_N)/\theta_N\toP0$ for every deterministic $\theta_N\to\infty$, with no upper restriction on its growth. This controls the residual relative to signal strength; the unnormalized residual need not vanish when $\theta_N$ grows.

\clearpage
\section{The high-signal limit}
\label{app:ceiling}

\begin{proof}[Proof of \autoref{thm:ceiling}]
The maximizer in \eqref{eq:crosscov} is invariant under positive rescaling. For every deterministic $\theta_N\to\infty$, the pair $(\hu,\hv)$ is therefore the leading singular pair of
\[
  \frac{C(\theta_N)}{\theta_N}=p_Nv_0^\top+D_N+\frac{E}{\theta_N}
\]
for all sufficiently large $N$. Here $D_N$ is the normalized residual defined after \autoref{prop:deltabound}, which gives $\opnorm{D_N}\toP0$. The response noise $E=A^\top Q/\sqrt N$ from \autoref{prop:dualmask} does not depend on $\theta$ and satisfies
\[
\begin{aligned}
  \opnorm E&\le\frac{\opnorm A\,\opnorm Q}{\sqrt N}=O_P(1),\\[2pt]
  \opnorm A&=\opnorm G^{1/2}=O_P(1),
  \qquad
  \opnorm Q=O_P\big(\sqrt N+\sqrt{\Dy}\big)=O_P(\sqrt N).
\end{aligned}
\]
Here \autoref{lem:opnorm} bounds $\opnorm A$, and the independent-entry norm bound \citep[Thm.~4.4.5]{vershynin2018} controls $\opnorm Q$. The entries of $Q$ are i.i.d.\ and uniformly sub-Gaussian by the moment identity of \autoref{prop:dualmask} with $\rhoy>0$ fixed. Consequently,
\[
  \opnorm{D_N+\frac{E}{\theta_N}}
  \le\opnorm{D_N}+\frac{\opnorm E}{\theta_N}\toP0.
\]
The leading term is rank one with singular pair $(\hp,v_0)$ and spectral gap $\norm{p_N}$, and $\norm{p_N}^2-(\rhoxy+\rhoy(1-\rhox)\tau_N)\toP0$ with $\rhoxy+\rhoy(1-\rhox)\tau_N\ge\rhoxy>0$ by \autoref{prop:geometry}, so $\norm{p_N}\ge\sqrt{\rhoxy}/2$ with probability tending to $1$. On the event where this lower bound holds and $\opnorm{D_N+E/\theta_N}\le\norm{p_N}/2$, whose probability tends to $1$, Wedin's theorem \citep{wedin1972} applied to the rank-one matrix $p_Nv_0^\top$, whose second singular value is zero, gives
\[
  \max\big\{\sin\angle(\hu,\hp),\ \sin\angle(\hv,v_0)\big\}
  \le\frac{2\,\opnorm{D_N+E/\theta_N}}{\norm{p_N}}\toP0.
\]
These are acute angles between the spanned lines, so they are invariant to sign changes, as are $R_x^2$ and $R_y^2$. Since $v_0$ is the leading right singular vector of $p_Nv_0^\top$,
\[
  \abs{R_y^2-1}=\sin^2\angle(\hv,v_0)\toP0 .
\]
For the left overlap, compare the orthogonal projectors onto the two directions. Since $u_0$ has unit norm,
\[
  \abs{R_x^2-(\ip{\hp}{u_0})^2}
  \le\opnorm{\hu\hu^\top-\hp\hp^\top}
  =\sin\angle(\hu,\hp)\toP0.
\]
The projector difference vanishes outside the span of the two directions. If that span is two-dimensional, its eigenvalues within the span are $\pm\sin\angle(\hu,\hp)$; if it is one-dimensional, the projectors coincide. Finally, $(\ip{\hp}{u_0})^2-\kappa_{\tau,N}\toP0$ by \autoref{prop:geometry}, which gives $R_x^2-\kappa_{\tau,N}\toP0$ without requiring convergence of $\tau_N$.
\end{proof}

The high-signal limit is strict, $\kappa_\tau<1$, exactly when $\rhox<1$ and $\tau>0$:

\begin{example}[Two-block design]
\label{ex:blocks}
Whitening does not force $\tau=1/\alpha_x$. Let $N$ be even, let $\alpha_{x,N}\ge2$ with $\alpha_{x,N}\to\alpha_x$, and let $U_1,U_2\in\R^{(N/2)\times\Dx}$ satisfy $U_m^\top U_m=\frac N2I_{\Dx}$ (feasible precisely when $\Dx\le N/2$, i.e.\ $\alpha_{x,N}\ge2$, since $U_m$ must have $\Dx$ orthogonal columns in $\R^{N/2}$). Assume that the row norms within each block converge uniformly to their common target:
\[
  \max_{k\le N/2}\left|\frac{\norm{(U_m)_{k\cdot}}^2}{\Dx}-1\right|
  \longrightarrow0,\qquad m=1,2.
\]
For $\gamma_1,\gamma_2>0$ with $\gamma_1^2+\gamma_2^2=2$, stack $\Xs=(\gamma_1U_1;\,\gamma_2U_2)$, so that $\Xs^\top\Xs=NI$ exactly. With Hadamard-derived blocks and random-sign $(u_0,v_0)$, \autoref{ass:incoherence} holds with probability tending to $1$. In block $m$, whitening determines the total score energy exactly, while the row-norm assumption determines the leverage uniformly over its rows:
\[
  \frac1N\sum_{k\text{ in block }m}b_k^2
  =\frac{\gamma_m^2}{N}\norm{U_mu_0}^2=\frac{\gamma_m^2}{2},
  \qquad
  \frac{\norm{(\Xs)_{k\cdot}}^2}{N}
  =\frac{\gamma_m^2}{\alpha_{x,N}}(1+o(1)).
\]
Their product is the contribution of block $m$ to \eqref{eq:taudef}. Summing the two contributions gives
\[
  \tau_N \;=\; \frac{\gamma_1^4+\gamma_2^4}{2}\cdot
  \frac{1}{\alpha_{x,N}}\,(1+o(1)),
\]
so $\tau\ge1/\alpha_x$ by the power-mean inequality, with equality only at $\gamma_1=\gamma_2$. At $\alpha_x=4$, $\rhox=\tfrac12$ and $\gamma_1^2=\tfrac32$, $\gamma_2^2=\tfrac12$ this gives
\[
  \tau=\tfrac5{16}>\tfrac14=\frac1{\alpha_x},
  \qquad
  \kappa_\tau=\tfrac{16}{21}<\tfrac45=\kappa .
\]
The block Gram identity fixes only the average row norm, so equal leverage is a separate hypothesis. Hadamard-derived blocks satisfy it exactly; for those blocks the row leverages have the stated block levels exactly and the formula for $\tau_N$ has no error term.
\end{example}

\clearpage
\section{The deterministic analysis}
\label{app:maskavg}

\subsection{Conditional spectral quantities}
\label{app:conditionaltransforms}

This subsection defines the conditional spectral quantities behind the finite-signal formulas in \autoref{thm:finitesignal}. The reduced matrix \eqref{eq:barCN} separates a rank-one spike from product noise. Its noise spectrum and its response along the true and corrupted directions are encoded by three Gram resolvent functionals: one weighted trace and two directional forms.

Given $S_x$, the quantities $G$, $\han_N$, and $s_N$ are fixed. The exact Gram resolvent $R(w):=(wI-G)^{-1}$ retains spectral and signal-direction information at the Gram coordinate $w$. We use three scalar functionals:
\begin{equation}\label{eq:resolvents}
\begin{aligned}
  \psi_N(w) &:= \tfrac1N\Tr\big[G\,R(w)\big]
  =\sum_i\frac{\lambda_i/N}{w-\lambda_i},\\[1pt]
  \Phi_N(w) &:= \ip{\han_N}{R(w)\han_N}
  =\sum_i\frac{(\ip{\varphi_i}{\han_N})^2}{w-\lambda_i},\\[1pt]
  \Phix_N(w) &:= \ip{u_0}{R(w)\han_N}
  =\sum_i\frac{\ip{\varphi_i}{u_0}\,\ip{\varphi_i}{\han_N}}{w-\lambda_i},
\end{aligned}
\end{equation}
where $\psi_N$ is the weighted average Gram spectral response, $\Phi_N$ the corrupted-direction response, and $\Phix_N$ the true/corrupted cross-response. In the orthonormal eigenbasis $(\varphi_i)_{i\le\Dx}$ of $G$, with eigenvalues $(\lambda_i)$, the form $x^\top R(w)y$ for $x,y\in\R^{\Dx}$ weights the $i$-th coordinate product by $(w-\lambda_i)^{-1}$. Thus the spectrum alone does not determine the directional responses. All three enter the supercritical calculation. Below threshold only the boundedness of the finite-$N$ true-direction response $\ip{u_0}{R(w)u_0}$ near the response stationary point $w_R$ of \autoref{lem:branch} is used (\autoref{app:subcritical}); confinement supplies it, and its limit is not part of the deterministic-equivalent assumption. Each functional is the resolvent transform $\sum_ic_i/(w-\lambda_i)$ of a finite (for $\Phix_N$ signed) measure on $\{\lambda_i\}$. At $\rhox=1$, $G=I_{\Dx}$ and $\han_N=u_0$, so the directional responses coincide while $\psi_N$ retains its $N^{-1}$ trace normalization. Explicitly, $\psi_N(w)=\alpha_{x,N}^{-1}(w-1)^{-1}$, whereas each directional response equals $(w-1)^{-1}$. The response fixed point then relates $w$ to a singular-value coordinate $z$. Its self-consistency correction $\delta$ selects the branch $\delta\in(-1,0)$ continued from zero at infinity, $w$ is the associated Gram coordinate, and $K=w/z$ is the transport factor. On this branch the fixed point and transforms are
\begin{gather}
  \delta = -\alpha_{y,N}(1+\delta)\,\psi_N(w),
  \qquad
  w := \alpha_{y,N}(1+\delta)\,z^2,
  \qquad
  K := w/z,
  \label{eq:fixedpoint}\\
  \mone(z) := K\,\Phi_N(w),
  \qquad
  \mcross(z) := K\,\Phix_N(w),
  \qquad
  \mtwo(z) := \frac{1}{z(1+\delta)},
  \label{eq:transforms}
\end{gather}
and these are conditional deterministic-equivalent transforms, with the exact identity $\Dbar:=\mone\mtwo=\alpha_{y,N}\Phi_N(w)$. An above-edge root of $s_N^2\Dbar(z)=1$ remains a candidate until \autoref{app:finitesignal} identifies the corresponding outlier of the random reduced matrix. The scalar construction in Appendices~\ref{app:mde}--\ref{app:baseline} provides deterministic comparison values for these functionals. \autoref{app:curves} evaluates the corresponding limiting transforms $m_1$, $m_2$, $m_\times$, and $\Dbar_0:=m_1m_2$. Their application to the realized design requires \autoref{ass:closure}; \autoref{app:finitesignal} then transfers the predictions to the estimator.

\subsection{The masked design and the scalar system}
\label{app:mde}

The three functionals in \eqref{eq:resolvents} depend on the realization of the design mask. Whitening gives the exact information-plus-noise split below, with deterministic $P$ and independent centered mask fluctuations $\Theta$. Flat and general profiles correspond to the settings of \citet{dozier2007analysis} and \citet{hachem2007}, respectively. For $A$, $a_N$, and $G$ as in \eqref{eq:pdefs}, with $\hat b:=b/\sqrt N$, the split is
\[
  A = P+\Theta,
  \qquad
  P = \sqrt{\rhox/N}\,\Xs,
  \qquad
  \Theta = \frac{(S_x-\rhox)\odot\Xs}{\sqrt{N\rhox}},
\]
where the deterministic identity $a_N=A^\top\hat b$ keeps the signal vector and $G$ coupled through the same mask.

\begin{proposition}[Exact information-plus-noise structure]
\label{prop:infoplusnoise}
At every $N$, $P^\top P=\rhox I_{\Dx}$ and $P^\top\hat b=\sqrt{\rhox}\,u_0$. The entries of $\Theta$ are independent and centered, with $\Var(\Theta_{ki})=(1-\rhox)(\Xs)_{ki}^2/N$. Each column sum of these variances equals $1-\rhox$ exactly, while the $k$-th row sum equals $(1-\rhox)\norm{(\Xs)_{k\cdot}}^2/N$.
\end{proposition}

\begin{proof}
Rescaling \eqref{eq:whitening} gives $P^\top P=(\rhox/N)\,\Xs^\top\Xs=\rhox I_{\Dx}$, and $\Xs^\top\Xs u_0=Nu_0$ gives $P^\top\hat b=\sqrt{\rhox}\,u_0$. Each $\Theta_{ki}$ depends only on the mask entry $(S_x)_{ki}$, so the entries remain independent and centered. Multiplying the Bernoulli variance $\rhox(1-\rhox)$ by $(\Xs)_{ki}^2/(N\rhox)$ gives the stated variance profile. Summing down column $i$ gives $(1-\rhox)(\Xs^\top\Xs)_{ii}/N=1-\rhox$, again by \eqref{eq:whitening}. Summing along row $k$ gives $(1-\rhox)\norm{(\Xs)_{k\cdot}}^2/N$.
\end{proof}

The scalar comparison sends the mask self-energies $e_1,e_2$ to block responses $g,h$ through Schur complements, then returns $e_1,e_2$ by normalized block traces. Here $g$ records the response on $\R^{\Dx}$ and $h$ on $\range\Xs$. Exact column regularity determines one trace weight, while the equal leverage of \autoref{def:tau} determines the other asymptotically. At $\xi:=\sqrt w$, the system is
\begin{equation}\label{eq:ghsystem}
\begin{aligned}
  g &= \cfrac{1}{(\xi-e_2)-\cfrac{\rhox}{\xi-e_1}},
  &\qquad
  h &= \cfrac{1}{(\xi-e_1)-\cfrac{\rhox}{\xi-e_2}},\\[4pt]
  e_1 &= \frac{(1-\rhox)\,g}{\alpha_x},
  &\qquad
  e_2 &= (1-\rhox)\Big(\frac{h}{\alpha_x}+\frac{1-1/\alpha_x}{\xi-e_1}\Big).
\end{aligned}
\end{equation}

For any real rectangular matrix $M$, its symmetric dilation is
\begin{equation}\label{eq:hermitization}
  H_M:=\begin{pmatrix}0&M\\M^\top&0\end{pmatrix}.
\end{equation}
The diagonal zero blocks have the row and column dimensions of $M$, respectively, so $H_M$ is square. For $A$, this construction acts on $\R^N\oplus\R^{\Dx}$ and places the three spectral quantities in one block resolvent as a normalized trace and two directional forms. For $w=\xi^2$ outside the spectrum of $AA^\top$, its resolvent is
\begin{equation}\label{eq:linblocks}
  \mathbf R(\xi) := \big(\xi I-H_{A}\big)^{-1}
  =
  \begin{pmatrix}
    \xi\,(wI-AA^\top)^{-1} & (wI-AA^\top)^{-1}A\\[2pt]
    A^\top(wI-AA^\top)^{-1} & \xi\,(wI-G)^{-1}
  \end{pmatrix}.
\end{equation}
Schur complementation and the commutation $A^\top(wI-AA^\top)^{-1}=(wI-G)^{-1}A^\top$ give \eqref{eq:linblocks}. Let $\Pi:=\Xs\Xs^\top/N$ be the rank-$\Dx$ orthogonal projector onto $\range(\Xs)$, so that $P^\top P=\rhox I_{\Dx}$ (\autoref{prop:infoplusnoise}) and $PP^\top=\rhox\Pi$. For a scalar $g$ and a complex $\xi$ at which $d_1$ below is nonzero, the block template at aspect ratio $\alpha_{x,N}$ is
\begin{equation}\label{eq:MblocksE}
\begin{aligned}
  d_1(g;\xi)&:=\xi-(1-\rhox)\alpha_{x,N}^{-1}g,
  &\qquad
  M_{22}&=g\,I_{\Dx},\\
  M_{11}&=d_1^{-1}(I_N-\Pi)+\big(d_1^{-1}+\rhox\,g\,d_1^{-2}\big)\Pi,
  &\qquad
  M_{21}&=d_1^{-1}g\,P^\top,
\end{aligned}
\end{equation}
and $\mathcal M_N(g;\xi)$ denotes the matrix on $\R^N\oplus\R^{\Dx}$ with these blocks and $M_{12}=M_{21}^\top$, indexed as in \eqref{eq:linblocks}. On a solution $(g,h,e_1,e_2)$ of \eqref{eq:ghsystem} at ratio $\alpha_{x,N}$, $d_1(g;\xi)=\xi-e_1$ and $d_1^{-1}+\rhox g\,d_1^{-2}=h$, so the template is the two-block inverse $\big(\xi I-H_P-\diag(e_1I_N,e_2I_{\Dx})\big)^{-1}$. The template takes one of three scalars.
\begin{compactitemize}
\item $g_{\mathrm G,N}(\xi):=\Dx^{-1}\,\E\Tr\big[(\xi I-H_{A_{\mathrm G}})^{-1}\big]_{22}$, the Gaussian expected trace of the benchmark model $A_{\mathrm G}:=P+\Theta_{\mathrm G}$ with iid $\mathcal N(0,(1-\rhox)/N)$ entries of $\Theta_{\mathrm G}$; \autoref{lem:gaussdesign} shows $\opnorm{\E(\xi I-H_{A_{\mathrm G}})^{-1}-\mathcal M_N(g_{\mathrm G,N};\xi)}=O(N^{-1})$ at fixed positive height.
\item $g_{0,N}(\xi)$, the exact finite-aspect scalar: the physical solution of the scalar equations \eqref{eq:ghsystem} at ratio $\alpha_{x,N}$; the exact finite-aspect baseline $M_{0,N}:=\mathcal M_N(g_{0,N};\xi)$ is displayed in \eqref{eq:finiteNblocks}.
\item $g(\xi)$, its limit at ratio $\alpha_x$, with $|g_{0,N}-g|\to0$ locally uniformly on the exterior by the aspect-ratio comparison \eqref{eq:baselinecompare}; $M_0:=\mathcal M(g;\xi)$, the template with $\alpha_x$ in place of $\alpha_{x,N}$ in $d_1$, is the limit of $M_{0,N}$.
\end{compactitemize}
\autoref{lem:gaussdesign} shows $g_{\mathrm G,N}-g_{0,N}\to0$; replacing $g_{\mathrm G,N}$ by $g_{0,N}$ or $g$ in the template therefore costs $o(1)$, not $O(N^{-1})$. The scalar equations \eqref{eq:ghsystem}, the physical branch and the support calculation (\autoref{prop:ipnsupport}, \autoref{lem:branch}) concern $g_{0,N}$ and $g$; this is the meaning of the scalar comparison. For nested flat orthogonal designs the mask fluctuations have exactly constant variance, so the Gaussian benchmark applies after a comparison step (\autoref{lem:comparison}); designs with a non-constant variance profile are not covered (\autoref{app:closurestatus}). For the masked design, convergence of the three Gram functionals is \autoref{ass:closure}; \autoref{app:rmt} proves it for the nested flat class.

\subsection{One variable}
\label{app:onevar}

The scalar system \eqref{eq:ghsystem} has four unknowns, but in that form it hides which solution is physical and where the spectral edge occurs. Reducing it to one coordinate makes both explicit. Let
\[
  d_1=\xi-e_1,\qquad d_2=\xi-e_2,\qquad
  \Lambda=d_1d_2-\rhox,
\]
and define the coordinate
\[
  \omega=\frac{e_1}{d_1},
\]
which parametrizes the physical branch throughout this appendix and \autoref{app:finitesignal}. The continued fractions of \eqref{eq:ghsystem} are the Schur-complement formulas $g=d_1/\Lambda$ and $h=d_2/\Lambda$. Since $e_1=\omega d_1=(1-\rhox)g/\alpha_x$ and $d_1d_2=\Lambda+\rhox$, the two trace equations of \eqref{eq:ghsystem} imply
\[
  \Lambda=\frac{1-\rhox}{\alpha_x\omega},\qquad
  d_1e_2=\frac{1-\rhox}{\alpha_x}\Big(\frac{d_1d_2}{\Lambda}+\alpha_x-1\Big)
  =1-\rhox+\rhox \omega.
\]
Now $\xi=d_1(1+\omega)$ and
\[
  \frac{1-\rhox}{\alpha_x\omega}=\Lambda=d_1(\xi-e_2)-\rhox
  =d_1^2(1+\omega)-1-\rhox \omega,
\]
so
\[
  d_1^2=\frac{1+\rhox \omega+(1-\rhox)/(\alpha_x\omega)}{1+\omega},
  \qquad
  \xi^2=(1+\omega)\Big(1+\rhox \omega+\frac{1-\rhox}{\alpha_x\omega}\Big).
\]
The remaining scalars are explicit as well:
\[
  e_1=\omega d_1,\qquad e_2=\frac{1-\rhox+\rhox \omega}{d_1},\qquad
  g=\frac{\alpha_x\omega d_1}{1-\rhox},\qquad
  h=\frac{1-\rhox+\alpha_x\rhox \omega}{(1-\rhox)\,d_1}.
\]
The physical solution is the branch continued from infinity, where $g,h\sim\xi^{-1}$, $e_1,e_2\to0$, and $\omega\downarrow0$.

The block resolvent expresses the three functionals of \eqref{eq:resolvents} through normalized traces and deterministic directions. Block inversion \eqref{eq:linblocks} gives $\Dx^{-1}\Tr(wI-G)^{-1}=(\xi\Dx)^{-1}\Tr\mathbf R_{22}$, hence $\psi_N(w)=N^{-1}\Tr[w(wI-G)^{-1}-I]$ reads off the lower-right block. The commutation $A(wI-G)^{-1}A^\top=w(wI-AA^\top)^{-1}-I$ moves the form at $a_N=A^\top\hat b$ to the $N$-dimensional side, where $\hat b$ is deterministic with $\Pi\hat b=\hat b$ and $\norm{\hat b}=1$, so the kernel part of $M_{11}$ is annihilated. The identity $(wI-G)^{-1}A^\top=\mathbf R_{21}$ with $P^\top\hat b=\sqrt{\rhox}\,u_0$ handles the cross form. Exactly, at every $N$,
\begin{equation}\label{eq:pushids}
  \ip{a_N}{(wI-G)^{-1}a_N}
  = \xi\,\ip{\hat b}{\mathbf R_{11}\hat b}-1,
  \qquad
  \ip{u_0}{(wI-G)^{-1}a_N}
  = \ip{u_0}{\mathbf R_{21}\hat b}.
\end{equation}
These exact identities express the random direction $a_N=A^\top\hat b$ through the deterministic vector $\hat b$, preserving its dependence on $G$. On $\{a_N\ne0\}$, replacing $a_N$ by the unit direction $\han_N$ divides the first form by $\norm{a_N}^2$ and the second by $\norm{a_N}$. No independence between $\han_N$ and $G$ is assumed.

The scalar comparison replaces $\mathbf R$ by $M_0$ in \eqref{eq:pushids} and the trace identity above. At $\tau=1/\alpha_x$, \autoref{prop:geometry} gives $\norm{a_N}^2\toP\rhox+(1-\rhox)/\alpha_x=n_a>0$, so $a_N\ne0$ with probability tending to $1$. Applying the two normalization factors gives the candidate limiting transforms
\begin{equation}\label{eq:closing}
\begin{aligned}
  \psi_0(w) &= \frac{\xi g-1}{\alpha_x},
  &\qquad
  \Phi_0(w) &= \frac{\xi h-1}{n_a},\\
  \Phix_0(w) &= \frac{\sqrt{\rhox}\,g}{d_1\sqrt{n_a}}.
\end{aligned}
\end{equation}
Separately, the effective spike strength $s_N=\theta\sqrt{\rhoy}\,\norm{a_N}$ satisfies $s_N^2\toP s_0^2:=\theta^2\rhoy n_a$. The scalars above turn \eqref{eq:closing} into rational functions of $\omega$ alone:
\[
  \xi g-1=\frac{\alpha_x\omega(1+\rhox \omega)}{1-\rhox},\qquad
  \xi h-1=\frac{\alpha_x\omega(n_a+\rhox \omega)}{1-\rhox},\qquad
  \frac g{d_1}=\frac{\alpha_x\omega}{1-\rhox}.
\]

\begin{definition}[The scalar system]
\label{def:maskavg}
With $n_a:=\rhox+(1-\rhox)/\alpha_x=\rhox/\kappa$, the limit of $\norm{a_N}^2$, the scalar system is \eqref{eq:ghsystem} with the limiting transforms \eqref{eq:closing}. On the physical branch it is equivalent to the one-variable form: writing $\psi(\omega):=\psi_0(W(\omega))$, $\Phi(\omega):=\Phi_0(W(\omega))$, and $\Phix(\omega):=\Phix_0(W(\omega))$ for $0<\omega<\omega_G$,
\begin{equation}\label{eq:uform}
\begin{aligned}
  W(\omega)&=(1+\omega)\Big(1+\rhox \omega+\frac{1-\rhox}{\alpha_x\omega}\Big),\\
  \psi(\omega)&=\frac{\omega(1+\rhox \omega)}{1-\rhox},
  \qquad
  \Phi(\omega)=\frac{\alpha_x\omega(n_a+\rhox \omega)}{(1-\rhox)\,n_a},
  \qquad
  \Phix(\omega)=\frac{\alpha_x\sqrt{\rhox}\,\omega}{(1-\rhox)\sqrt{n_a}}.
\end{aligned}
\end{equation}
\end{definition}

Here $\rhox$ is the design retention, $1-\rhox$ the design-missingness fraction, and $\alpha_x^{-1}$ the limiting feature-to-sample ratio, so the two contributions to $n_a$ are the mean design and the centered mask fluctuation (\autoref{prop:geometry}).

The finite-$N$ and scalar Gram measures can be identified separately. The measure $\nu_N:=\frac1N\sum_{i=1}^{\Dx}\delta_{\lambda_i(G)}$ is positive, supported on $[0,\infty)$, has mass $\alpha_{x,N}^{-1}$, and satisfies $\psi_N(w)=\int\lambda(w-\lambda)^{-1}\,\dd\nu_N(\lambda)$ by the spectral theorem.

For $0<\rhox<1$, eliminating the self-energies in \eqref{eq:ghsystem} identifies $-g(\sqrt w)/\sqrt w$ with the information-plus-noise Stieltjes solution of \citet[(1.1)]{dozier2007analysis}, at population law $\delta_{\rhox}$, noise variance $1-\rhox$, and ratio $\alpha_x^{-1}$. Its positive probability law and unique physical Stieltjes solution select the analytic continuation from infinity. Let $\nu_0$ be $\alpha_x^{-1}$ times that law. It is supported on $[0,\infty)$, has mass $\alpha_x^{-1}$, and gives
\[
  g(\xi)=\alpha_x\int\frac{\xi}{\xi^2-\lambda}\,\dd\nu_0(\lambda),
  \qquad
  \psi_0(w)=\int\frac{\lambda}{w-\lambda}\,\dd\nu_0(\lambda).
\]
The first formula is the resolvent of the symmetric lift to $\pm\sqrt\lambda$ with equal weights; squaring recovers $\alpha_x\nu_0$. The second follows from $\psi_0=(\xi g-1)/\alpha_x$. At $\rhox=1$, $g(\xi)=\xi/(\xi^2-1)$ directly, and $\nu_0=\alpha_x^{-1}\delta_1$. This identifies the scalar comparison measure without assuming convergence of the realized $\nu_N$; that link is supplied by \autoref{ass:closure}.

\subsection{The branch, the fold, and the response edge}
\label{app:branch}

The scalar masked-Gram comparison has top edge $w_G$. The response map $T_0$ selects a distinct exterior Gram coordinate $w_R>w_G$, and $z_{\mathrm{edge}}=\sqrt{T_0(w_R)}$ is the resulting candidate product-noise singular-value edge. \autoref{ass:closure} links the realized design functionals to the scalar branch, while \autoref{app:finitesignal} identifies the realized conditional response edge. Throughout, $\nu_0$ is the measure of \autoref{app:onevar}.

Every support statement below concerns the scalar comparison rather than the realized spectrum of $G$. With aspect ratio $\alpha_{x,N}$, the scalar system \eqref{eq:ghsystem} is the information-plus-noise system with population law $\delta_{\rhox}$, noise level $1-\rhox$, and ratio $\alpha_{x,N}^{-1}$, so its boundary is available from the published inverse support map. The notation $W_N$ denotes the map \eqref{eq:uform} with $\alpha_x$ replaced by $\alpha_{x,N}$, and $\omega_{G,N},w_{G,N}$ denote its turning point and minimum value.

\begin{proposition}[Finite-$N$ support map]
\label{prop:ipnsupport}
For fixed $0<\rhox<1$, let $\mathfrak b$ denote the variable written $b$ in the information-plus-noise support analysis of \citet[(1.6)]{baisilverstein2012}, following \citet{dozier2007analysis}, reserving $b$ for the latent scores. At population law $\delta_{\rhox}$, noise level $1-\rhox$, and ratio $\alpha_{x,N}^{-1}$, that analysis inverts the support map as
\[
  x_N(\mathfrak b)
  =\frac{1}{\mathfrak b^2}\,
  m_{\delta_{\rhox}}^{-1}\Big(
  \frac{\alpha_{x,N}}{1-\rhox}\Big(1-\frac1{\mathfrak b}\Big)\Big)
  +\frac{(1-\rhox)(1-\alpha_{x,N}^{-1})}{\mathfrak b}.
\]
Under $\mathfrak b=(1+\omega)^{-1}$, for every $\omega>0$,
\begin{equation}\label{eq:supportmap}
  x_N(\mathfrak b)=W_N(\omega)
  =(1+\omega)\Big(1+\rhox \omega+\frac{1-\rhox}{\alpha_{x,N}\omega}\Big),
\end{equation}
and the two coordinates run in opposite directions,
\begin{equation}\label{eq:supportorient}
  \frac{\dd\mathfrak b}{\dd\omega}=-\frac{1}{(1+\omega)^2}<0,
  \qquad\text{so}\qquad
  x_N'(\mathfrak b)>0
  \iff
  W_N'(\omega)<0 .
\end{equation}
The criterion $x_N'(\mathfrak b)>0$ that marks the exterior of the support therefore holds exactly on $0<\omega<\omega_{G,N}$, and
\[
  w_{G,N}=\min_{\omega>0}W_N(\omega)=W_N(\omega_{G,N})
\]
is the top boundary of the finite-$N$ support, the minimizer being unique by strict convexity. Since $\alpha_{x,N}\to\alpha_x$, $\omega_{G,N}\to\omega_G$ and $w_{G,N}\to w_G$.
\end{proposition}

\begin{proof}
For a point mass $m_{\delta_{\rhox}}(q)=(\rhox-q)^{-1}$, so $m_{\delta_{\rhox}}^{-1}(y)=\rhox-1/y$. At $\mathfrak b=(1+\omega)^{-1}$ the argument is $-\alpha_{x,N}\omega/(1-\rhox)$, giving $m_{\delta_{\rhox}}^{-1}=\rhox+(1-\rhox)/(\alpha_{x,N}\omega)$ and
\[
  x_N(\mathfrak b)
  =(1+\omega)^2\Big(\rhox+\frac{1-\rhox}{\alpha_{x,N}\omega}\Big)
  +(1-\rhox)\Big(1-\frac1{\alpha_{x,N}}\Big)(1+\omega)
  =W_N(\omega),
\]
the last step by collecting the common factor $1+\omega$ and canceling the two terms in $(1-\rhox)\alpha_{x,N}^{-1}$. Equation~\eqref{eq:supportorient} is the chain rule. Finally
\[
  W_N'(\omega)=2\rhox \omega+(1+\rhox)-\frac{1-\rhox}{\alpha_{x,N}\omega^2},
  \qquad
  W_N''(\omega)=2\rhox+\frac{2(1-\rhox)}{\alpha_{x,N}\omega^3}>0,
\]
so $W_N'$ has the single positive zero $\omega_{G,N}$, and $W_N\to\infty$ at both ends of $(0,\infty)$. Clearing $\omega^2$ makes $\omega_{G,N}$ the positive root of $2\rhox \omega^3+(1+\rhox)\omega^2-(1-\rhox)\alpha_{x,N}^{-1}=0$, which depends continuously on $\alpha_{x,N}^{-1}$, and $W_N$ depends continuously on both.
\end{proof}

\begin{lemma}[Branch, fold, and response edge]
\label{lem:branch}
Let $0<\rhox<1$ be fixed. \textbf{(a)} The physical solution of \eqref{eq:ghsystem} is the branch with $g,h\sim1/\xi$ at large $\xi$, continued downwards. It is parametrized by the coordinate $\omega$: the map $\omega\mapsto W(\omega)$ is strictly convex on $\omega>0$ with a unique turning point $\omega_G$, the positive root of
\begin{equation}\label{eq:foldcubic}
  2\rhox \omega_G^3+(1+\rhox)\omega_G^2-\frac{1-\rhox}{\alpha_x}=0,
\end{equation}
and $w=W(\omega)$ maps $(0,\omega_G)$ decreasingly onto $(w_G,\infty)$, where $w_G:=W(\omega_G)$ is the top boundary of the support of the scalar system. The branch has a square-root fold at $w_G$ and no real continuation below it. \textbf{(b)} On the branch, eliminating $\delta$ from the response fixed point \eqref{eq:fixedpoint}, evaluated on the limits \eqref{eq:closing}, through $1+\delta=(1+\alpha_y\psi_0(w))^{-1}$ sends the Gram coordinate to the squared response coordinate
\begin{equation}\label{eq:zbranch}
  z(w)^2 \;=\; T_0(w):=w\Big(\alpha_y^{-1}+\psi_0(w)\Big),
  \qquad
  T(\omega):=T_0(W(\omega)).
\end{equation}
The derivative $\partial_wT_0$ is strictly increasing on $(w_G,\infty)$, tends to $-\infty$ at $w_G$ and to $\alpha_y^{-1}>0$ at infinity, and therefore has a unique zero $w_R=W(\omega_R)$ with $\omega_R\in(0,\omega_G)$. Equivalently,
\begin{equation}\label{eq:responseedge}
  T'(\omega_R)=0,\qquad T'(\omega)<0\quad(0<\omega<\omega_R),\qquad
  z_{\mathrm{edge}}=\sqrt{T(\omega_R)},
\end{equation}
with primes denoting $\omega$-derivatives, and in explicit form
\[
  T'(\omega)=W'(\omega)\big(\alpha_y^{-1}+\psi(\omega)\big)
  +W(\omega)\,\frac{1+2\rhox \omega}{1-\rhox}.
\]
Thus $0<\omega<\omega_R$ parametrizes the exterior of the response bulk, with $z$ increasing as $\omega$ decreases, and $\omega=\omega_R$ is the bulk edge.
\end{lemma}

\begin{proof}
\textbf{(a)} Near infinity the system \eqref{eq:ghsystem} linearizes: $g,h\sim1/\xi$ and $e_1,e_2=O(1/\xi)$, which identifies the branch there uniquely, with $\omega\downarrow0$ as $\xi\to\infty$. The Stieltjes-law identification in \autoref{app:onevar} supplies its analytic continuation and the positive Gram measure $\nu_0$. Along the parametrization of \autoref{app:onevar}, every solution with $\xi^2=w$ real corresponds to a root of $W(\omega)=w$. The first two derivatives are
\[
  W'(\omega)=2\rhox \omega+(1+\rhox)-\frac{1-\rhox}{\alpha_x\omega^2},
  \qquad
  W''(\omega)=2\rhox+\frac{2(1-\rhox)}{\alpha_x\omega^3}>0.
\]
Thus $W'$ has exactly one positive zero $\omega_G$, which is \eqref{eq:foldcubic} after clearing $\omega^2$. As $\omega\downarrow0$, $W(\omega)\to\infty$ and $W'(\omega)<0$, so the branch continued from infinity is $0<\omega<\omega_G$, $W$ maps it decreasingly onto $(w_G,\infty)$, and by strict convexity $W(\omega)\geq w_G$ for every $\omega>0$. At fixed parameters, Taylor expansion at the turning point gives
\[
  W(\omega)-w_G
  =\frac12W''(\omega_G)(\omega-\omega_G)^2
   +O\big(\abs{\omega-\omega_G}^3\big).
\]
Since $W''(\omega_G)>0$, the two nearby real roots meet in a square-root fold. The physical branch has $\omega<\omega_G$, so its inverse takes the minus sign:
\[
  \omega(w)=\omega_G-\sqrt{\frac{2(w-w_G)}{W''(\omega_G)}}+O(w-w_G),
  \qquad w\downarrow w_G.
\]
There is no nearby real continuation below $w_G$. \autoref{prop:ipnsupport} and its proof use no property of the ratio beyond its positivity, so they hold verbatim with $\alpha_x$ in place of $\alpha_{x,N}$. At ratio $\alpha_x^{-1}$, population law $\delta_{\rhox}$ and noise level $1-\rhox$, the information-plus-noise law is the one defining $\nu_0$ in \autoref{app:onevar}, and $W$ is its inverse support map, so $w_G=\min_{\omega>0}W(\omega)$ is the top boundary of $\supp\nu_0$. Under \autoref{ass:closure} the functionals converge to this branch, so $\nu_0$ is the limiting Gram law on the exterior domain; for nested flat orthogonal designs \autoref{app:rmt} proves directly that $w_G$ is the spectral edge. \textbf{(b)} The fixed point \eqref{eq:fixedpoint}, evaluated on the limits \eqref{eq:closing}, is $\delta=-\alpha_y(1+\delta)\psi_0(w)$, so $1+\delta=(1+\alpha_y\psi_0(w))^{-1}$ and $z^2=w/(\alpha_y(1+\delta))=T_0(w)$, which is \eqref{eq:zbranch}. With $\nu_0$ as in part (a),
\[
  \partial_wT_0(w)=\frac1{\alpha_y}
  -\int\frac{\lambda^2}{(w-\lambda)^2}\,\dd\nu_0(\lambda),
  \qquad
  \partial_w^2T_0(w)=2\int\frac{\lambda^2}{(w-\lambda)^3}\,\dd\nu_0(\lambda)>0.
\]
The endpoint behavior comes from the one-variable formulas, with no regularity hypothesis on $\nu_0$. On the exterior branch $\psi_0(W(\omega))=\psi(\omega)$, so \eqref{eq:uform} gives
\[
  \psi_0'(w)=\frac{\psi'(\omega)}{W'(\omega)},
  \qquad
  \psi'(\omega)=\frac{1+2\rhox \omega}{1-\rhox}>0 .
\]
As $w\downarrow w_G$, equivalently $\omega\uparrow\omega_G$, one has $W'(\omega)\uparrow0^-$, so $\psi_0'(w)\to-\infty$ while $w$ and $\psi_0(w)$ stay bounded, and therefore
\[
  \partial_wT_0(w)=\alpha_y^{-1}+\psi_0(w)+w\,\psi_0'(w)\longrightarrow-\infty,
\]
while $\partial_wT_0(w)\to\alpha_y^{-1}>0$ at infinity. Strict positivity of $\partial_w^2T_0$ makes $\partial_wT_0$ increasing, so it vanishes exactly once, at $w_R>w_G$. Transporting through the decreasing parametrization $w=W(\omega)$ gives \eqref{eq:responseedge}, and the chain rule gives the explicit $T'(\omega)$ with $\psi'$ as above.
\end{proof}

\subsection{The finite-\texorpdfstring{$N$}{N} baseline and uniform margins}
\label{app:baseline}

The scalar system below is the exact finite-aspect baseline whose limits appear in \autoref{ass:closure}; \autoref{app:rmt} verifies those limits for nested flat orthogonal designs by a Gaussian comparison.

At finite $N$, the scalar baseline uses $\alpha_{x,N}^{-1}$ in place of $\alpha_x^{-1}$. Its scalar solution is the physical branch of \eqref{eq:ghsystem} at that ratio; with $e_{1,N}=(1-\rhox)g_{0,N}/\alpha_{x,N}$ and $h_{0,N}=(\xi-e_{1,N})^{-1}+\rhox g_{0,N}(\xi-e_{1,N})^{-2}$, the template \eqref{eq:MblocksE} evaluated at $g_{0,N}$ gives
\begin{equation}\label{eq:finiteNblocks}
  M_{0,N;22}=g_{0,N}I_{\Dx},
  \qquad
  M_{0,N;11}=h_{0,N}\Pi+\frac{I_N-\Pi}{\xi-e_{1,N}},
  \qquad
  M_{0,N;21}=\frac{g_{0,N}}{\xi-e_{1,N}}P^\top,
\end{equation}
and taking the two normalized traces returns $e_{1,N}$ and yields $e_{2,N}=\frac{1-\rhox}{\alpha_{x,N}}\{h_{0,N}+(\alpha_{x,N}-1)/(\xi-e_{1,N})\}$. This calculation uses only $P^\top P=\rhox I_{\Dx}$, not a probabilistic approximation. Here $g_{0,N},h_{0,N},e_{1,N},e_{2,N},w_{G,N}$ denote the finite-aspect baseline quantities ($g_{0,N}$ and $h_{0,N}$ are distinct from the response coupling $g_B$ in \eqref{eq:exactforms} and from the fluctuation vector $g_N$ of \autoref{prop:geometry}), and the unsubscripted notation is reserved for their $\alpha_{x,N}^{-1}\to\alpha_x^{-1}$ limits, as in \autoref{app:mde}. On every fixed exterior compact and for every fixed $k\ge0$,
\begin{equation}\label{eq:baselinecompare}
  \sum_{j=0}^{k}\norm{\partial_\xi^j(M_{0,N}-M_0)}_{\mathrm{op}}
  =O_k\big(\abs{\alpha_{x,N}^{-1}-\alpha_x^{-1}}\big).
\end{equation}
The physical branch $\omega(\xi)$ at ratio $\alpha_x$ is analytic on the exterior, and differentiating $W(\omega(\xi))=\xi^2$ gives $W'(\omega)\,\partial_\xi\omega=2\xi\ne0$, so $W'(\omega)\ne0$ along the branch. The implicit function theorem for $W(\omega)=\xi^2$, with the ratio as an additional variable, then makes $\omega$ jointly smooth in $\xi$ near the compact and in the ratio near $\alpha_x^{-1}$. For real $\xi$ the resulting root at ratio $\alpha_{x,N}^{-1}$ lies in $(0,\omega_{G,N})$ because $\omega_{G,N}\to\omega_G$, so it is the physical root, and analytic continuation in $\xi$ extends this to the whole compact. On the branch $d_1=\xi/(1+\omega)$, so each block of \eqref{eq:finiteNblocks} combines $I_N-\Pi$, $\Pi$, $I_{\Dx}$ and $P^\top$, all of norm at most $1$, with coefficients rational in $(\omega,\xi,\alpha_{x,N}^{-1})$ whose denominators stay away from zero. The mean value inequality in the ratio, applied to these coefficients and their $\xi$-derivatives, then gives \eqref{eq:baselinecompare}.
No convergence rate for $\alpha_{x,N}^{-1}-\alpha_x^{-1}$ is needed; \eqref{eq:baselinecompare} is a purely deterministic passage to the displayed limits.
The same convergence, without rate, holds locally uniformly on compact subsets of $\{\operatorname{Re}\xi>0,\operatorname{Im}\xi>0\}$: $-g_{0,N}(\xi)/\xi$ is the Stieltjes transform at $\xi^2$ of a law with support in $[0,w_{G,N}]$, these supports are uniformly bounded, every weakly convergent subsequence satisfies the limiting equation at ratio $\alpha_x^{-1}$, uniqueness of the solution $m$ with $\operatorname{Im}m>0$ and $\operatorname{Im}(\xi^2m)\ge0$ identifies the limit with the physical branch of \autoref{app:onevar}, and resolvent bounds with equicontinuity give local uniformity.

These calculations establish the branch structure in \autoref{lem:branch}. The fold formulas give a square-root branch point at $w_G$, strict convexity of $W$ shows that there is no higher fold, and the branch is real analytic for $w>w_G$ with no other obstruction connected to infinity.

At the response edge,
\[
  z_{\mathrm{edge}}^2=T_0(w_R),
  \qquad
  \frac{d^2}{dw^2}\sqrt{T_0(w)}\Big|_{w_R}
  =\frac{\partial_w^2T_0(w_R)}{2z_{\mathrm{edge}}}>0,
\]
and at the stationary point the integral formulas of \autoref{lem:branch} give the quantitative bound, $\nu_0$ being supported on $[0,\infty)$,
\begin{equation}\label{eq:curvaturebound}
  \partial_w^2T_0(w_R)
  \ge\frac2{w_R}\int\frac{\lambda^2}{(w_R-\lambda)^2}\,\dd\nu_0(\lambda)
  =\frac2{\alpha_yw_R}.
\end{equation}

All quantities above depend continuously on the parameters. For $(\rhox,\alpha_x^{-1},\alpha_y^{-1})$ in a compact subset of $(0,1)\times(0,1]\times(0,\infty)$, $\omega_G$ stays bounded above and away from zero, the fold derivative stays bounded below at the fold, and the endpoint signs are uniform. To see the latter without appealing to the limiting measure, use the explicit form of $T'(\omega)$ from \autoref{lem:branch}: uniformly on the compact parameter set, $T'(\omega)<0$ for all sufficiently small $\omega>0$, while $T'(\omega_G)=W(\omega_G)(1+2\rhox \omega_G)/(1-\rhox)>0$. Strict convexity of $T_0$ makes the root unique. Compactness then gives
\begin{equation}\label{eq:uniformmargins}
  w_R-w_G\ge\gamma,
  \qquad
  \partial_w^2T_0(w_R)\ge\gamma,
  \qquad
  -W'(\omega_R)\ge\gamma
\end{equation}
for some $\gamma>0$, where $\omega_R\in(0,\omega_G)$ corresponds to $w_R$. The derivative formulas of \autoref{lem:branch} together with \eqref{eq:curvaturebound} certify these margins analytically on every compact interior parameter set.

When $\rhox=1$, $\Theta=0$ and $G=I_{\Dx}$ exactly. The coordinate $\omega$ and the inner Gram fold disappear, and the response edge is the ordinary scalar sample-covariance edge, evaluated directly in \autoref{prop:limits}(i). When $\alpha_x\to\infty$, the cubic \eqref{eq:foldcubic} for $\omega_G$ likewise degenerates. Neither limit belongs to the generic proportional fold statement.

\paragraph{Finite-$N$ response transform and exterior branch.}
Write $\mathbb C^{+}:=\{\operatorname{Im}z>0\}$ and $\mathbb C^{++}:=\{\operatorname{Re}z>0,\ \operatorname{Im}z>0\}$. Assume $\opnorm G\le K_0$ and $G\ne0$, with $(\lambda_i,\varphi_i)$ the eigenpairs of $G$ as in \eqref{eq:resolvents}. For real $w>\lambda_{\max}(G)$ define
\begin{equation}\label{eq:randomfp}
  T_N(w):=\frac w{\alpha_{y,N}}+w\,\psi_N(w)
  =\frac w{\alpha_{y,N}}+\frac1N\sum_i\frac{w\lambda_i}{w-\lambda_i},
  \qquad
  z^2=T_N(w);
\end{equation}
on the exterior branch the response fixed point \eqref{eq:fixedpoint} is equivalent to the second relation, with $1+\delta=(1+\alpha_{y,N}\psi_N(w))^{-1}$. Since
\begin{equation}\label{eq:TNconvex}
  \partial_wT_N(w)=\frac1{\alpha_{y,N}}-\frac1N\sum_i\frac{\lambda_i^2}{(w-\lambda_i)^2},
  \qquad
  \partial_w^2T_N(w)=\frac2N\sum_i\frac{\lambda_i^2}{(w-\lambda_i)^3}>0,
\end{equation}
$\partial_wT_N$ increases strictly from $-\infty$ at $\lambda_{\max}(G)^+$ to $\alpha_{y,N}^{-1}$ at infinity; it has a unique zero $w_{R,N}>\lambda_{\max}(G)$, the minimizer of $T_N$ on $(\lambda_{\max}(G),\infty)$, and
\begin{equation}\label{eq:edgetransfer}
  z^2_{\mathrm{edge},N}:=T_N(w_{R,N})=\min_{w>\lambda_{\max}(G)}T_N(w),
  \qquad
  \partial_w^2T_N(w_{R,N})\ge\frac2{\alpha_{y,N}w_{R,N}},
\end{equation}
the last since $\frac1N\sum_i\lambda_i^2(w_{R,N}-\lambda_i)^{-2}=\alpha_{y,N}^{-1}$ and $(w_{R,N}-\lambda_i)^{-1}\ge w_{R,N}^{-1}$. Two bounds are uniform in the design: $T_N(w)\ge w/\alpha_{y,N}$ gives $w_{R,N}\le\alpha_{y,N}z^2_{\mathrm{edge},N}$, and $z^2_{\mathrm{edge},N}\le T_N(2K_0)\le2K_0(\alpha_{y,N}^{-1}+\alpha_{x,N}^{-1})$ because $2K_0-\lambda_i\ge K_0$.

For real $z>z_{\mathrm{edge},N}$ let $w_N(z)$ be the unique solution of $T_N(w)=z^2$ in $(w_{R,N},\infty)$, and put $c_N(z):=z/w_N(z)$, the reciprocal of the transport factor $K=w/z$ of \eqref{eq:fixedpoint}. For $z\ge z_{\mathrm{edge},N}+\varepsilon$, because $\partial_wT_N\le\alpha_{y,N}^{-1}$,
\[
  w_N(z)-w_{R,N}\ \ge\ \alpha_{y,N}\big(z^2-z^2_{\mathrm{edge},N}\big)\ \ge\ \alpha_{y,N}\varepsilon^2,
\]
and for $w\ge w_{R,N}$, from $\lambda(w-\lambda)^{-1}\le(w_{R,N}/w)\,\lambda(w_{R,N}-\lambda)^{-1}$ for $0\le\lambda<w_{R,N}$ and the stationarity identity,
\[
  \partial_wT_N(w)\ \ge\ \frac1{\alpha_{y,N}}\Big[1-\Big(\frac{w_{R,N}}w\Big)^2\Big].
\]
On any bounded real interval $[z_{\mathrm{edge},N}+\varepsilon,K_1]$ these give $\operatorname{dist}(w_N(z),\spec G)\ge\alpha_{y,N}\varepsilon^2$ and $\partial_wT_N(w_N(z))\ge c(\varepsilon,K_1)>0$, so by the inverse function theorem there is $\eta(\varepsilon,K_1)>0$, depending only on $(\varepsilon,K_1,K_0,\alpha_{y,N})$, such that $w_N$ extends analytically to the closed rectangle
\begin{equation}\label{eq:rectangle}
  \mathcal S_N(\varepsilon,K_1):=\big\{z:\ z_{\mathrm{edge},N}+\varepsilon\le\operatorname{Re}z\le K_1,\ \abs{\operatorname{Im}z}\le\eta(\varepsilon,K_1)\big\},
\end{equation}
with $\abs{w_N'}\le C(\varepsilon,K_1)$ and $\operatorname{dist}(w_N(z),\spec G)\ge\tfrac12\alpha_{y,N}\varepsilon^2$ there. On $\mathcal S_N(\varepsilon,K_1)$ the finite-$N$ conditional transforms of \eqref{eq:transforms},
\[
\begin{gathered}
  \mone=c_N^{-1}\Phi_N(w_N),\qquad
  \mtwo=\alpha_{y,N}c_N,\\
  \mcross=c_N^{-1}\Phix_N(w_N),\qquad
  \Dbar=\mone\mtwo=\alpha_{y,N}\Phi_N(w_N),
\end{gathered}
\]
and the secular function $f_N:=1-s_N^2\Dbar$ are analytic with all derivatives bounded by constants depending only on $(\varepsilon,K_1,K_0,\alpha_{y,N},s_N)$. On the real branch $f_N$ is real; $f_N\equiv1$ if $s_N=0$, and otherwise $f_N$ is strictly increasing, since $f_N'=-s_N^2\alpha_{y,N}\Phi_N'(w_N)\,w_N'$ with $\Phi_N'<0$ and $w_N'=2z/\partial_wT_N(w_N)>0$. At a root $\hz_N$, Cauchy--Schwarz in the spectral decomposition gives $-\Phi_N'\ge\Phi_N^2$, and $w_N'\ge2\alpha_{y,N}z$, so
\[
  f_N'(\hz_N)\ \ge\ \frac{2\hz_N}{s_N^2},\qquad
  w_N(\hz_N)\le\lambda_{\max}(G)+\alpha_{y,N}s_N^2,\qquad
  \hz_N^2-z^2_{\mathrm{edge},N}\le s_N^2,
\]
the last two from $1=s_N^2\alpha_{y,N}\Phi_N(w_N(\hz_N))\le s_N^2\alpha_{y,N}/(w_N(\hz_N)-\lambda_{\max}(G))$ and $\partial_wT_N\le\alpha_{y,N}^{-1}$.

\paragraph{The same branch at positive height.}
Let $H_N$ be the empirical spectral distribution of $AA^\top/\alpha_{y,N}$ and, for $\zeta\in\mathbb C^+$, let $m^{\det}_N(\zeta)$ be the unique solution in $\mathbb C^+$ of Silverstein's equation at ratio $\alpha_{y,N}$ and population law $H_N$,
\begin{equation}\label{eq:silversteinN}
  \zeta=-\frac1m+\frac1N\sum_i\frac{\lambda_i}{1+\lambda_im/\alpha_{y,N}}
  \qquad\Longleftrightarrow\qquad
  \zeta=T_N(w),\quad w=-\frac{\alpha_{y,N}}m,
\end{equation}
the equivalence because $\lambda_i/(1+\lambda_im/\alpha_{y,N})=w\lambda_i/(w-\lambda_i)$. By \citet[(6.1.5)--(6.1.6)]{baisilverstein2010}, $m^{\det}_N$ is the Stieltjes transform of the probability measure $F_N:=F^{\alpha_{y,N},H_N}$ on $[0,\infty)$, whose support has top $z^2_{\mathrm{edge},N}$ (\autoref{lem:support}). For $z\in\mathbb C^{++}$ put $w_N(z):=-\alpha_{y,N}/m^{\det}_N(z^2)$; then $\operatorname{Im}w_N(z)>0$ and $T_N(w_N(z))=z^2$. For real $z>z_{\mathrm{edge},N}$ this function is real with $w_N(z)\in(w_{R,N},\infty)$ and $T_N(w_N(z))=z^2$ (\autoref{lem:support}), so it is the real exterior branch; its analytic continuation coincides with the rectangle extension on $\mathcal S_N(\varepsilon,K_1)\cap\mathbb C^{++}$, and on $\operatorname{Im}z<0$ by conjugation. The reciprocal transport factor is
\[
  c_N(z)=\frac z{w_N(z)}=-\frac{z\,m^{\det}_N(z^2)}{\alpha_{y,N}}
  =\frac1{2\alpha_{y,N}}\int\Big[\frac1{z-\sqrt\lambda}+\frac1{z+\sqrt\lambda}\Big]\dd F_N(\lambda),
\]
which is analytic on $\mathbb C^+$ with $\operatorname{Im}c_N\le0$ there; hence $\operatorname{Im}(zI-c_NG)\succeq(\operatorname{Im}z)I$, and
\[
  \opnorm{(zI-c_NG)^{-1}}\le(\operatorname{Im}z)^{-1},
  \qquad
  \abs{\alpha_{y,N}c_N(z)}\le(\operatorname{Im}z)^{-1}.
\]
With $(zI-c_NG)^{-1}=c_N^{-1}(w_NI-G)^{-1}$, the matrix
\[
  \mathbf R^{\det}_N(z):=\diag\big((zI-c_NG)^{-1},\ \alpha_{y,N}c_N\,I_{\Dy}\big)
\]
on $\R^{\Dx}\oplus\R^{\Dy}$ is the deterministic equivalent, established in \autoref{app:finitesignal}, of $(zI-H_E)^{-1}$ with $E:=A^\top Q/\sqrt N$ as in \eqref{eq:barCN}; it is distinct from $\mathbf R(\xi)=(\xi I-H_A)^{-1}$ of \eqref{eq:linblocks}, and $\opnorm{\mathbf R^{\det}_N(z)}\le(\operatorname{Im}z)^{-1}$ on $\mathbb C^+$. Its upper block has entries $\mone$, $\mcross$ and $c_N^{-1}\ip{u_0}{(w_NI-G)^{-1}u_0}$ in the directions $(\han_N,\han_N)$, $(u_0,\han_N)$ and $(u_0,u_0)$, its lower block is $\mtwo=\alpha_{y,N}c_N$ times the identity, and its off-diagonal blocks vanish.

The limiting objects are built from the Gram law $\nu_0$ of \autoref{app:onevar}. For $\zeta\in\mathbb C^+$ let $m_0(\zeta)$ be the unique solution in $\mathbb C^+$ of
\begin{equation}\label{eq:silverstein0}
  -\frac1{m_0(\zeta)}=\zeta-\int\frac{\lambda\,\dd\nu_0(\lambda)}{1+\lambda m_0(\zeta)/\alpha_y},
\end{equation}
the Stieltjes transform of a probability measure $F_0$ whose support has top $T_0(w_R)=z_{\mathrm{edge}}^2$ (\autoref{lem:support}); put $w_0(z):=-\alpha_y/m_0(z^2)$, $c_0(z):=z/w_0(z)$, and $f_0:=1-s_0^2\Dbar_0$ with $s_0^2=\theta^2\rhoy n_a$ and $\Dbar_0$ from \eqref{eq:productcollapse}. On the real exterior $z>z_{\mathrm{edge}}$, $w_0$ coincides with the branch $w=W(\omega(z))$ of \autoref{lem:branch}: both are analytic there, invert $T_0$, and tend to infinity with $z$.

\subsection{Exact rank-one identities}
\label{app:secular}

The first result locates singular values produced by a rank-one perturbation. The second gives the singular-vector overlaps at a simple singular value above the noise spectrum. The finite-signal proof in \autoref{app:finitesignal} uses both results, and \autoref{app:curves} evaluates them on the limiting transforms.

Let $B\in\R^{\Dx\times\Dy}$ be a deterministic noise matrix and let $C_B:=s_N\han_Nv_0^\top+B$, with the strength $s_N$ and unit directions $\han_N,u_0,v_0$ held fixed. We use the symmetric dilation $H_M$ defined in \eqref{eq:hermitization}. For real $z>\opnorm B$, let
\[
  R_1(z):=(z^2I-BB^\top)^{-1},
  \qquad R_2(z):=(z^2I-B^\top B)^{-1}.
\]
The five exact scalar forms are
\begin{equation}\label{eq:exactforms}
\begin{aligned}
  m_{1,B}(z)&:=z\,\ip{\han_N}{R_1(z)\han_N},
  &\qquad
  m_{2,B}(z)&:=z\,\ip{v_0}{R_2(z)v_0},\\
  m_{\times,B}(z)&:=z\,\ip{u_0}{R_1(z)\han_N},
  &\qquad
  g_B(z)&:=\ip{\han_N}{R_1(z)Bv_0},\\
  g_{\times,B}(z)&:=\ip{u_0}{R_1(z)Bv_0}.
\end{aligned}
\end{equation}
Define the secular function
\begin{equation}\label{eq:secfun}
  F_B(z):=\big(1-s_Ng_B(z)\big)^2-s_N^2m_{1,B}(z)m_{2,B}(z).
\end{equation}

\begin{proposition}[Exact rank-one secular identity]
\label{prop:secular}
For real $z>\opnorm B$,
\begin{equation}\label{eq:detratio}
  \frac{\det(zI-H_{C_B})}{\det(zI-H_B)}=F_B(z).
\end{equation}
The determinant identity also holds for complex $z$ wherever $R_1(z)$ and $R_2(z)$ exist. For real $z>\opnorm B$, $z$ is a singular value of $C_B$ if and only if $F_B(z)=0$, or equivalently,
\begin{equation}\label{eq:secular}
  \big(1-s_N\,g_B(z)\big)^2 \;=\; s_N^2\,m_{1,B}(z)\,m_{2,B}(z).
\end{equation}
\end{proposition}

\begin{proof}[Proof of \autoref{prop:secular}]
The two columns of $U$ place the spike directions in the left and right blocks, and $J$ couples them:
\[
  U:=\begin{pmatrix}\han_N&0\\0&v_0\end{pmatrix},
  \qquad J:=\begin{pmatrix}0&1\\1&0\end{pmatrix}.
\]
Thus $H_{C_B}=H_B+s_NUJU^\top$. The relation $R_1(z)B=BR_2(z)$ follows from $(z^2I-BB^\top)B=B(z^2I-B^\top B)$ and gives the block inverse
\[
  (zI-H_B)^{-1}
  =\begin{pmatrix}zR_1(z)&R_1(z)B\\[2pt]B^\top R_1(z)&zR_2(z)\end{pmatrix}.
\]
Compressing this inverse to the two spike directions gives
\[
 U^\top(zI-H_B)^{-1}U
 =\begin{pmatrix}m_{1,B}&g_B\\g_B&m_{2,B}\end{pmatrix}.
\]
The matrix determinant lemma therefore yields
\[
 \frac{\det(zI-H_{C_B})}{\det(zI-H_B)}
 =\det\!\left(I_2-s_NJU^\top(zI-H_B)^{-1}U\right)
 =\det\!\begin{pmatrix}1-s_Ng_B&-s_Nm_{2,B}\\-s_Nm_{1,B}&1-s_Ng_B\end{pmatrix}.
\]
The $2\times2$ determinant is $(1-s_Ng_B)^2-s_N^2m_{1,B}m_{2,B}=F_B(z)$. For real $z>\opnorm B$ the denominator is nonzero, so $F_B(z)=0$ exactly when $z$ is an eigenvalue of $H_{C_B}$. Its positive eigenvalues are the singular values of $C_B$, which proves the equivalence. The same block algebra proves the identity for complex $z$ wherever the resolvents exist. It also covers $s_N=0$, for which $F_B\equiv1$.
\end{proof}

We next calculate the corresponding singular-vector overlaps.

\begin{proposition}[Exact overlap identities]
\label{prop:overlapidentities}
Suppose that $\hz>\opnorm B$ is a simple singular value of $C_B$, and let $(u,v)$ be a corresponding unit singular pair. Then $F_B'(\hz)>0$ and
\begin{equation}\label{eq:residues}
  (\ip{\han_N}{u})^2=\frac{2m_{1,B}}{F_B'},
  \qquad
  (\ip{v_0}{v})^2=\frac{2m_{2,B}}{F_B'},
\end{equation}
\begin{equation}\label{eq:residuecross}
  \ip{u_0}{u}
  =\Big(\big(1-s_Ng_B\big)\frac{m_{\times,B}}{m_{1,B}}
    +s_Ng_{\times,B}\Big)\ip{\han_N}{u}.
\end{equation}
All functions are evaluated at $\hz$, and primes denote differentiation with respect to $z$.
\end{proposition}

\begin{proof}[Proof of \autoref{prop:overlapidentities}]
Let $\alpha:=\ip{\han_N}{u}$ and $\beta:=\ip{v_0}{v}$, and evaluate all functions at $\hz$ unless an argument is shown. Since $\hz>\opnorm B$, we have $s_N\ne0$. If $\alpha=0$, the equation $C_B^\top u=\hz v$ would give $B^\top u=\hz v$, contradicting $\hz>\opnorm B$ on unit vectors. The equation $C_Bv=\hz u$ similarly rules out $\beta=0$. Also, $R_1(\hz)$ and $R_2(\hz)$ are positive definite, so $m_{1,B}(\hz),m_{2,B}(\hz)>0$.

The two singular-vector equations read
\[
  Bv+s_N\beta\,\han_N=\hz u,
  \qquad B^\top u+s_N\alpha\,v_0=\hz v.
\]
The block inverse in the proof of \autoref{prop:secular} therefore gives
\[
  u=s_NR_1(\hz)\big(\hz\beta\,\han_N+\alpha\,Bv_0\big),
  \qquad
  v=s_NR_2(\hz)\big(\hz\alpha\,v_0+\beta\,B^\top\han_N\big).
\]
Projecting onto $\han_N$ and $v_0$, respectively, gives
\[
  \alpha(1-s_Ng_B)=s_Nm_{1,B}\beta,
  \qquad \beta(1-s_Ng_B)=s_Nm_{2,B}\alpha.
\]
Multiplying the first by $\beta$ and the second by $\alpha$ gives $\beta^2m_{1,B}=\alpha^2m_{2,B}$.

To impose unit norms, use the same equations in stacked form, with $U$ defined above:
\[
  \binom{u}{v}=s_N(\hz I-H_B)^{-1}U\binom{\beta}{\alpha}.
\]
For real $z>\opnorm B$, the resolvent is symmetric and its derivative is $\frac{\dd}{\dd z}(zI-H_B)^{-1}=-(zI-H_B)^{-2}$. Taking the squared norm of the stacked vector and differentiating the compressed resolvent gives
\[
\begin{aligned}
  2=\norm{u}^2+\norm{v}^2
  &=s_N^2(\beta,\alpha)U^\top(\hz I-H_B)^{-2}U\binom{\beta}{\alpha}\\
  &=-s_N^2\big(\beta^2m_{1,B}'+2\alpha\beta\,g_B'+\alpha^2m_{2,B}'\big).
\end{aligned}
\]
Only the resolvent is differentiated; $\alpha$ and $\beta$ are the fixed overlaps at $\hz$. Substituting $\beta^2=\alpha^2m_{2,B}/m_{1,B}$ and $\alpha\beta=\alpha^2(1-s_Ng_B)/(s_N\,m_{1,B})$, with the divisions justified above, reduces this normalization identity to
\[
  \frac{\alpha^2}{m_{1,B}}
  \Big((m_{1,B}m_{2,B})'+\frac{2}{s_N}\big(1-s_Ng_B\big)g_B'\Big)=-\frac{2}{s_N^2},
\]
whose bracket equals $-F_B'/s_N^2$ by \eqref{eq:secfun}. Hence $\alpha^2F_B'=2m_{1,B}$, so $F_B'(\hz)=2m_{1,B}/\alpha^2>0$, and $\beta^2F_B'=2m_{2,B}$, which is \eqref{eq:residues}. Finally, projecting the representation of $u$ onto $u_0$ gives $\ip{u_0}{u}=s_N\big(m_{\times,B}\,\beta+g_{\times,B}\,\alpha\big)$, and eliminating $\beta$ through the first projection gives \eqref{eq:residuecross}.
\end{proof}

\subsection{The threshold and the recovery curves}
\label{app:curves}

This subsection evaluates the exact identities of \autoref{app:secular} on the deterministic scalar limits from Appendices~\ref{app:mde}--\ref{app:baseline}. If the secular product exceeds $1$ at the edge, the comparison has a candidate separated root, and the same transforms give candidate overlaps. \autoref{app:finitesignal} combines \autoref{ass:closure} with the spike geometry of \autoref{prop:geometry}, a Gaussian comparison for the response noise, classical spectral exclusion for the noise edge, and residual control to transfer the root and residues to the random estimator. At an exterior response coordinate $z$, with $w=W(\omega(z))$ on the branch of \autoref{lem:branch}, the notation $m_1,m_2,m_\times$ denotes the transforms $\mone,\mtwo,\mcross$ from \eqref{eq:transforms} evaluated on the limits \eqref{eq:closing}:
\begin{equation}\label{eq:limtransforms}
  m_1(z)=\frac{W(\omega)\,\Phi(\omega)}{z},
  \qquad
  m_2(z)=\frac{1+\alpha_y\psi(\omega)}{z},
  \qquad
  z^2=T(\omega),
\end{equation}
using $K=w/z$ and $1+\delta=(1+\alpha_y\psi(\omega))^{-1}$. Their product simplifies exactly: since $z^2=W(\omega)(1+\alpha_y\psi(\omega))/\alpha_y$ by \eqref{eq:zbranch},
\begin{equation}\label{eq:productcollapse}
  \Dbar_0(z):=m_1(z)\,m_2(z)
  =\frac{W(\omega)\,\Phi(\omega)\,(1+\alpha_y\psi(\omega))}{z^2}
  =\alpha_y\,\Phi(\omega).
\end{equation}

\begin{proposition}[Threshold, outlier, and recovery curves]
\label{prop:curves}
On the scalar system of \autoref{def:maskavg}, on the exterior branch of \autoref{lem:branch}, with $s_0^2=\theta^2\rhoy n_a$:

\textbf{(i)} $\Dbar_0$ is positive and strictly decreasing on $(z_{\mathrm{edge}},\infty)$, from its edge value $\alpha_y\Phi(\omega_R)$ to zero, and the outlier equation $s_0^2\Dbar_0(z)=1$ has a root above the edge if and only if $\theta>\theta_c$, where
\begin{equation}\label{eq:thetac}
  \theta_c^2
  =\frac{1-\rhox}{\alpha_x\rhoy\alpha_y\,\omega_R\,(n_a+\rhox \omega_R)} .
\end{equation}
Because $\omega_R$ does not depend on $\rhoy$, at fixed $(\alpha_x,\alpha_y,\rhox)$,
\begin{equation}\label{eq:rhoylaw}
  \theta_c(\rhoy)=\theta_c(1)\,\rhoy^{-1/2}.
\end{equation}

\textbf{(ii)} For fixed $\theta>\theta_c$ the root is unique and simple: the secular equation reads
\begin{equation}\label{eq:utheta}
\begin{aligned}
  &\frac{\alpha_x\theta^2\rhoy\alpha_y}{1-\rhox}\,
  \omega_\theta\,(n_a+\rhox \omega_\theta)=1,\\[2pt]
  &\omega_\theta=
  \frac{\sqrt{n_a^2+4\rhox(1-\rhox)/(\alpha_x\theta^2\rhoy\alpha_y)}-n_a}
  {2\rhox}\;\in\;(0,\omega_R),
\end{aligned}
\end{equation}
with outlier location $\hz=\sqrt{T(\omega_\theta)}$, and
\begin{equation}\label{eq:Dbarderiv}
\begin{aligned}
  &\Phi'(\omega)=\frac{\alpha_x(n_a+2\rhox \omega)}{(1-\rhox)\,n_a},
  \qquad
  \frac{d\omega}{dz}=\frac{2z}{T'(\omega)},\\[2pt]
  &\Dbar_0'(\hz)
  =\frac{2\alpha_x\alpha_y \hz\,(n_a+2\rhox \omega_\theta)}
  {(1-\rhox)\,n_a\,T'(\omega_\theta)}<0 .
\end{aligned}
\end{equation}

\textbf{(iii)} At the root, the residue identities of \autoref{prop:overlapidentities} evaluated on the limiting transforms assign the squared overlaps, with $R_a^2$ the limit assigned to $(\ip{\hu}{\han_N})^2$,
\begin{equation}\label{eq:overlaplims}
  R_a^2=\frac{-2m_1}{s_0^2\Dbar_0'},
  \qquad
  R_y^2=\frac{-2m_2}{s_0^2\Dbar_0'},
  \qquad
  R_x^2=\frac{-2m_1}{s_0^2\Dbar_0'}
  \Big(\frac{\Phix(\omega_\theta)}{\Phi(\omega_\theta)}\Big)^2,
\end{equation}
every unlabeled quantity evaluated at $z=\hz$. Substitution gives, with every term at $\omega=\omega_\theta$,
\begin{equation}\label{eq:curves}
\begin{aligned}
  R_a^2&=-\frac{\alpha_xW(\omega)\,\omega^2(n_a+\rhox \omega)^2\,T'(\omega)}
  {(1-\rhox)\,n_a\,T(\omega)\,(n_a+2\rhox \omega)},\\
  R_x^2&=-\frac{\alpha_x\rhox W(\omega)\,\omega^2\,T'(\omega)}
  {(1-\rhox)\,T(\omega)\,(n_a+2\rhox \omega)},
  \qquad
  R_y^2=-\frac{\alpha_y \omega\,(n_a+\rhox \omega)\,T'(\omega)}
  {W(\omega)\,(n_a+2\rhox \omega)}.
\end{aligned}
\end{equation}
They are positive, tend to zero as $\theta\downarrow\theta_c$, and satisfy $R_x^2\to\rhox/n_a=\kappa$ and $R_y^2\to1$ as $\theta\to\infty$.

\textbf{(iv)} For $\theta<\theta_c$ the outlier equation has no root on $[z_{\mathrm{edge}},\infty)$, and the secular function keeps the edge margin
\begin{equation}\label{eq:edgemargin}
  1-s_0^2\Dbar_0(z_{\mathrm{edge}})
  =1-\frac{\theta^2}{\theta_c^2}>0 .
\end{equation}
\end{proposition}

\begin{proof}
\textbf{(i)} By \eqref{eq:uform}, $\Phi'(\omega)>0$, and on the exterior branch $T'(\omega)<0$ (\autoref{lem:branch}), so $d\omega/dz=2z/T'(\omega)<0$ and $\Dbar_0(z)=\alpha_y\Phi(\omega(z))$ decreases strictly, to zero because $\omega\downarrow0$ as $z\to\infty$. Equality of $s_0^2\Dbar_0$ with $1$ at the edge defines the threshold: inserting $s_0^2=\theta^2\rhoy n_a$ and $\alpha_y\Phi(\omega_R)=\alpha_x\alpha_y\omega_R(n_a+\rhox \omega_R)/((1-\rhox)n_a)$ and solving for $\theta^2$ gives \eqref{eq:thetac}. The coordinate $\omega_R$ solves $T'(\omega_R)=0$, an equation in $(\alpha_x,\alpha_y,\rhox)$ alone, so $\theta_c^2\rhoy$ is free of $\rhoy$, which is \eqref{eq:rhoylaw}.

\textbf{(ii)} The left side of \eqref{eq:utheta} is strictly increasing in $\omega_\theta$, so the root is unique, and the quadratic formula gives its closed form. The root satisfies $\omega_\theta<\omega_R$ because the left side evaluated at $\omega_R$ exceeds $1$ exactly when $\theta>\theta_c$. Differentiating \eqref{eq:uform} gives $\Phi'$, the chain rule on $z^2=T(\omega)$ gives $d\omega/dz$, and $\Dbar_0'=\alpha_y\Phi'(\omega)\,d\omega/dz$ gives \eqref{eq:Dbarderiv}. The sign comes from $T'(\omega_\theta)<0$.

\textbf{(iii)} At a simple zero of the secular function, the residue identities \eqref{eq:residues}--\eqref{eq:residuecross} express the three squared overlaps through $m_{1,B}$, $m_{2,B}$, $m_{\times,B}/m_{1,B}$, and $F_B'$. On the limiting transforms, $g_B$ and $g_{\times,B}$ are replaced by zero, $m_{\times}/m_1$ by $\Phix/\Phi$, and $F_B'$ by $-s_0^2\Dbar_0'$, which gives \eqref{eq:overlaplims}. \autoref{app:finitesignal} proves, under the deterministic equivalent, that the finite-$N$ overlap forms converge to these limits. The cancellation is clearest before expanding $W$ and $\Phi$. At the root, $s_0^2\alpha_y\Phi=1$ and $\Dbar_0'=2\alpha_y z\Phi'/T'$. Using $m_1=W\Phi/z$, $m_2=\alpha_yz/W$, and $z^2=T$ in \eqref{eq:overlaplims} gives
\[
  R_a^2=-\frac{W\Phi^2T'}{T\Phi'},
  \qquad R_y^2=-\frac{\alpha_y\Phi T'}{W\Phi'},
  \qquad R_x^2=R_a^2\Big(\frac{\Phix}{\Phi}\Big)^2.
\]
Here all functions are evaluated at $\omega_\theta$; $T'$ and $\Phi'$ are $\omega$-derivatives, while $\Dbar_0'$ is a $z$-derivative. Substituting \eqref{eq:uform}, including $(\Phix/\Phi)^2=\rhox n_a/(n_a+\rhox\omega)^2$, gives the expanded formulas \eqref{eq:curves}. Their common factor $-T'$ makes the response-edge behavior explicit. Positivity follows from $T'(\omega_\theta)<0$. As $\theta\downarrow\theta_c$, $\omega_\theta\to \omega_R$ and $T'(\omega_\theta)\to0$, so all three curves vanish. As $\theta\to\infty$, $\omega_\theta\downarrow0$ with $W(\omega)\sim(1-\rhox)/(\alpha_x\omega)$, $T(\omega)\sim(1-\rhox)/(\alpha_x\alpha_y\omega)$, and $T'(\omega)\sim-(1-\rhox)/(\alpha_x\alpha_y\omega^2)$. Inserting these asymptotics into \eqref{eq:curves} gives $R_x^2\to\rhox/n_a$ and $R_y^2\to1$.

\textbf{(iv)} At $\theta<\theta_c$ the edge value satisfies $s_0^2\Dbar_0(z_{\mathrm{edge}})=\theta^2/\theta_c^2<1$ by \eqref{eq:thetac}, and $\Dbar_0$ decreases beyond the edge, so $1-s_0^2\Dbar_0>0$ on the whole exterior branch, which is \eqref{eq:edgemargin}.
\end{proof}

The subcritical margin \eqref{eq:edgemargin} rules out a separated root but does not by itself force zero overlaps: that requires control of the edge spectral subspace of the random matrix, carried out in \autoref{app:subcritical}.

Since $s_0^2=\theta^2\rhoy n_a$ and $n_a=\rhox/\kappa$, the root condition $s_0^2\,\Dbar_0(z_{\mathrm{edge}})>1$ reads $\theta^2>\theta^{*2}\kappa/\rhoxy$, the factorization \eqref{eq:boundary} of the transition boundary. Both signs of the competition there occur: solving \eqref{eq:ghsystem} by the continuation of \autoref{lem:branch} at $\alpha_x=4$, $\alpha_y=2$ places $\theta_c$ above the isotropic reference $(\alpha_x\alpha_y)^{-1/4}\rhoxy^{-1/2}$ at $\rhox=\tfrac12$ and below it at $\rhox=0.3$.

\begin{lemma}[The deformation raises the detection level]
\label{lem:edgeinflation}
For every $\alpha_x\ge1$, $\alpha_y>0$ and $0<\rhox<1$, $\theta^*>(\alpha_x\alpha_y)^{-1/4}$, with equality only in the limit $\rhox\to1$. The threshold is not ordered this way: $\sqrt\kappa<1$ multiplies $\theta^*$ in \eqref{eq:boundary}, and $\theta_c$ falls on either side of $(\alpha_x\alpha_y)^{-1/4}\rhoxy^{-1/2}$, as above.
\end{lemma}

\begin{proof}
By \eqref{eq:productcollapse}, $\theta^{*2}=1/(\alpha_y\Phi(\omega_R))$, so the claim is $\Phi(\omega)^2<\alpha_x/\alpha_y$ at $\omega=\omega_R$. Write $q:=(1-\rhox)/\alpha_x$, so that $n_a=\rhox+q$ and, by \eqref{eq:uform}, $\Phi(\omega)=\omega(n_a+\rhox \omega)/(q\,n_a)$ and $\psi'(\omega)=(1+2\rhox \omega)/(1-\rhox)$; put $p:=-W'(\omega)=q/\omega^2-(1+\rhox)-2\rhox \omega$. Solving $T'(\omega)=0$ in the explicit form of \autoref{lem:branch} gives $\alpha_y^{-1}=W(\omega)\psi'(\omega)/p-\psi(\omega)$, and substituting it,
\[
  \frac{\alpha_x}{\alpha_y}-\Phi(\omega)^2
  =\frac{Q(q,\rhox,\omega)}{p\,q^2n_a^2},
\]
where $Q$ has positive integer coefficients, grouped by powers of $\omega$ as
\[
\begin{aligned}
 Q={}&q^2(q+\rhox)^2(1+\rhox)
 +2q(q+\rhox)(q^2\rhox+q\rhox^2+2q\rhox+q+2\rhox^2)\omega\\
 &+(3q^3\rhox^2+6q^3\rhox+6q^2\rhox^3+12q^2\rhox^2+q^2\rhox+q^2\\
 &\qquad+3q\rhox^4+6q\rhox^3+q\rhox^2+2q\rhox+\rhox^3+\rhox^2)\omega^2\\
 &+2\rhox(q+\rhox)(2q^2\rhox+2q\rhox^2+q+2\rhox+1)\omega^3
 +\rhox^2(4q+5\rhox+1)\omega^4+2\rhox^3\omega^5 .
\end{aligned}
\]
Every factor of the denominator except $p$ is positive, and $p>0$ reads $2\rhox \omega^3+(1+\rhox)\omega^2-(1-\rhox)/\alpha_x<0$, which \eqref{eq:foldcubic} identifies with $\omega<\omega_G$. The certificate therefore applies to every exterior stationary pair satisfying the edge equation, and in particular at $\omega_R$. At $\rhox=1$ the two sides agree by \autoref{prop:limits}(i).
\end{proof}

\begin{proposition}[Limits of the scalar system]
\label{prop:limits}
\textbf{(i)} At $\rhox=1$ the deformation is absent: $G=I_{\Dx}$ exactly, $\han_N=u_0$, $n_a=1$, both self-energies vanish, and the transforms collapse to
\[
  \psi_0(w)=\frac1{\alpha_x(w-1)},
  \qquad
  \Phi_0(w)=\Phix_0(w)=\frac1{w-1},
  \qquad
  T_0(w)=w\Big(\alpha_y^{-1}+\frac1{\alpha_x(w-1)}\Big).
\]
Here $s_0^2=\rhoy\theta^2$. The edge, threshold, and outlier have closed forms,
\begin{equation}\label{eq:rhoxone}
\begin{aligned}
  w_R&=1+\sqrt{\alpha_y/\alpha_x},
  &\qquad
  z_{\mathrm{edge}}&=\alpha_y^{-1/2}+\alpha_x^{-1/2},\\
  \theta_c^2&=\frac1{\rhoy\sqrt{\alpha_x\alpha_y}},
  &\qquad
  w_\theta&=1+\alpha_y s_0^2,
\end{aligned}
\end{equation}
with $\hz^2=(s_0^2+\alpha_x^{-1})(s_0^2+\alpha_y^{-1})/s_0^2$. For $\theta>\theta_c$,
\begin{equation}\label{eq:rhoxonecurves}
  R_x^2=\frac{\alpha_x\alpha_y s_0^4-1}
  {\alpha_y s_0^2(\alpha_x s_0^2+1)},
  \qquad
  R_y^2=\frac{\alpha_x\alpha_y s_0^4-1}
  {\alpha_x s_0^2(\alpha_y s_0^2+1)},
\end{equation}
which are the closed forms in \eqref{eq:closedforms}. Since $\han_N=u_0$ here, $R_a^2=R_x^2$. In particular $\theta^*=(\alpha_x\alpha_y)^{-1/4}$ exactly in this case.

\textbf{(ii)} Let $0<\rhox<1$, $0<\rhoy\le1$, and $0<\alpha_y<\infty$ be fixed. As $\alpha_x\to\infty$, $\kappa\to1$. For every fixed $\theta>0$, $s_0^2\to\rhoxy\theta^2$, the root is eventually supercritical, and
\[
  \hz^2\longrightarrow\rhoxy\theta^2+\alpha_y^{-1},
  \qquad R_a^2,R_x^2\longrightarrow1,
  \qquad R_y^2\longrightarrow
  \frac{\alpha_y\rhoxy\theta^2}{1+\alpha_y\rhoxy\theta^2}.
\]
At $\theta=0$ there is no exterior root and the subcritical recovery curves are zero. The effective threshold retains a relative inflation: $(\alpha_x\alpha_y)^{1/4}\theta^*\to\big(1+(1-\rhox^2)/\alpha_y\big)^{1/4}>1$, with equality only as $\rhox\to1$.
\end{proposition}

\begin{proof}
\textbf{(i)} At $\rhox=1$, $\Theta=0$, so $G=I_{\Dx}$ and $a_N=A^\top\hat b=u_0$ exactly; the coordinate $\omega$ and the inner Gram fold disappear, and \eqref{eq:ghsystem} collapses to $e_1=e_2=0$ and $g=h=\xi/(w-1)$, giving the displayed transforms through \eqref{eq:closing}. Solving $\partial_wT_0(w_R)=0$ with
\[
  \partial_wT_0(w)=\frac1{\alpha_y}-\frac1{\alpha_x(w-1)^2}
\]
gives $w_R=1+\sqrt{\alpha_y/\alpha_x}$ and $z_{\mathrm{edge}}^2=T_0(w_R)=(\alpha_x^{-1/2}+\alpha_y^{-1/2})^2$. The outlier equation $s_0^2\alpha_y\Phi_0(w_\theta)=1$ gives $w_\theta=1+\alpha_y s_0^2$, hence $\hz^2=T_0(w_\theta)$ as displayed, and the edge value $\Phi_0(w_R)=\sqrt{\alpha_x/\alpha_y}$ gives $\theta_c^2$. For the curves, $\Dbar_0(z)=\alpha_y/(w-1)$ and $dw/dz=2z/\partial_wT_0(w)$, so \eqref{eq:overlaplims} evaluates through
\[
  \partial_wT_0(w_\theta)
  =\frac{\alpha_x\alpha_y s_0^4-1}{\alpha_x\alpha_y^2 s_0^4},
  \qquad
  R_y^2=\frac{(w_\theta-1)^2\,\partial_wT_0(w_\theta)}{s_0^2\,w_\theta},
  \qquad
  R_x^2=\frac{\alpha_x(w_\theta-1)^2\,\partial_wT_0(w_\theta)}
  {s_0^2\,(\alpha_x(w_\theta-1)+\alpha_y)},
\]
where the last two displays substitute $m_2=\alpha_yz/w$, $m_1=w\Phi_0(w)/z$, and $\Phix_0/\Phi_0=1$ into \eqref{eq:overlaplims}. Inserting $w_\theta-1=\alpha_y s_0^2$ gives \eqref{eq:rhoxonecurves}, which is \eqref{eq:closedforms}. \textbf{(ii)} The limit has two scales: fixed exterior Gram coordinates and coordinates approaching the edge. Evaluating $W$ at $\alpha_x^{-1/2}$ gives $1\le w_G\le(1+\alpha_x^{-1/2})^2\to1$, so every fixed $w>1$ is eventually exterior. Its physical solution of $W(\omega)=w$ has $\omega\to0$, since $\omega<\omega_G$ and \eqref{eq:foldcubic} gives $\omega_G=O(\alpha_x^{-1/2})$. Expanding $W$ gives
\[
  w-1=(1+\rhox)\omega+\rhox\omega^2
  +\frac{1-\rhox}{\alpha_x}+\frac{1-\rhox}{\alpha_x\omega},
\]
so $\alpha_x\omega\to(1-\rhox)/(w-1)$. Thus \eqref{eq:uform} and \eqref{eq:closing}, with $n_a\to\rhox$, yield
\[
  \psi_0(w)\to0,
  \qquad \Phi_0(w),\Phix_0(w)\to\frac1{w-1}.
\]
These pointwise limits describe the response away from $w=1$.

To locate the response edge, expand its defining derivative directly:
\[
  T'(\omega)=-\frac{1-\rhox}{\alpha_x\alpha_y\omega^2}
  +\frac{1+\rhox}{\alpha_y}+\frac1{1-\rhox}
  +O(\omega+\alpha_x^{-1}).
\]
The remainder is uniform for $0<\omega\le\omega_0$ at any fixed small $\omega_0>0$. Indeed, the omitted terms form a polynomial in $\omega$ whose constant term is $(1+\rhox)/\alpha_x$ and whose positive-power coefficients are bounded for $\alpha_x\ge1$, with $\rhox$ and $\alpha_y$ fixed. By \autoref{lem:branch}, $0<\omega_R<\omega_G=O(\alpha_x^{-1/2})$. The uniform expansion therefore applies at $\omega_R$, where $T'(\omega_R)=0$:
\[
  \frac{1-\rhox}{\alpha_x\alpha_y\omega_R^2}
  =\frac{1+\rhox}{\alpha_y}+\frac1{1-\rhox}+O(\alpha_x^{-1/2}).
\]
The right-hand side tends to a positive constant. Solving for $\omega_R$ gives
\[
  \omega_R=\frac{1-\rhox}{\sqrt{\alpha_x(\alpha_y+1-\rhox^2)}}\,(1+o(1)),
  \qquad
  \Phi(\omega_R)=\sqrt{\frac{\alpha_x}{\alpha_y+1-\rhox^2}}\,(1+o(1)),
\]
where the second relation follows from \eqref{eq:uform} and $n_a\to\rhox$. Therefore $\theta^{*2}=1/(\alpha_y\Phi(\omega_R))$ gives the stated relative inflation, while $\theta_c^2=\theta^{*2}/(\rhoy n_a)\to0$.

For fixed $\theta>0$, the secular equation \eqref{eq:utheta} and $n_a\ge\rhox>0$ first give $\omega_\theta=O(\alpha_x^{-1})$, and then
\[
  \alpha_x\omega_\theta\longrightarrow
  \frac{1-\rhox}{\alpha_y\rhoxy\theta^2},
  \qquad \frac{\omega_\theta}{\omega_R}\longrightarrow0.
\]
Thus this root is eventually on the exterior response branch. Substitution into $W$ and $T$ gives
\[
  W(\omega_\theta)\to1+\alpha_y\rhoxy\theta^2,
  \qquad T(\omega_\theta)\to\rhoxy\theta^2+\alpha_y^{-1}.
\]
The derivative expansion and \eqref{eq:uform} also give, at $\omega_\theta$,
\[
  \omega_\theta T'\to-\rhoxy\theta^2,
  \qquad \frac{\Phi}{\Phi'}\sim\omega_\theta,
  \qquad \Phi\to\frac1{\alpha_y\rhoxy\theta^2},
  \qquad \frac{\Phix}{\Phi}\to1.
\]
Inserting these relations into the compact recovery expressions in the proof of \autoref{prop:curves} yields the three stated overlap limits. At $\theta=0$, the secular function is identically $1$, so there is no exterior root and the subcritical curves remain zero.
\end{proof}

The coordinate formulas require $\rhox<1$. The large-aspect-ratio limit also leaves the fixed-parameter fold regime because $\omega_G\to0$. Part \textbf{(ii)} concerns limits of the deterministic curves, with positive signal and the retentions fixed. Where the finite-signal theorem applies at each fixed aspect ratio, the estimator interpretation takes $N\to\infty$ first and then $\alpha_x\to\infty$. The convergence is not uniform as $\theta\downarrow0$. The effective threshold $\theta^*$ retains a ratio strictly greater than $1$ to its isotropic counterpart. The variance scale $(1-\rhox^2)/\alpha_x$ from \eqref{eq:flatmoment} has the same order as $(w_R-1)^2$, so it remains visible at the edge while the fixed-$w$ transforms converge.

\subsection{The deterministic equivalent}
\label{app:closurestatus}

Appendices~\ref{app:mde}--\ref{app:baseline} construct the scalar comparison, \autoref{app:secular} gives exact fixed-noise identities, and \autoref{app:curves} evaluates them on scalar limits. \autoref{ass:closure} independently bridges to the realized mask through Gram confinement and convergence of the weighted spectral response and two directional functionals.

\begin{assumption}[Deterministic equivalent for the masked design]
\label{ass:closure}
Under \autoref{ass:incoherence} and, for $0<\rhox<1$, asymptotically equal leverage as in \autoref{def:tau}, let $w_G$ be the fold of \autoref{lem:branch}; at $\rhox=1$, where $G=I_{\Dx}$ exactly, read $w_G:=1$, the top of the Gram support. Then \textbf{(i)} $\lambda_{\max}(G)\le w_G+o_P(1)$, and \textbf{(ii)} for every compact $K\subset(w_G,\infty)$ there is a fixed open complex neighborhood $U_K\supset K$, with $\overline U_K$ separated from the Gram support, on which the random functionals $\psi_N$, $\Phi_N$, $\Phix_N$ of \eqref{eq:resolvents} converge in probability, uniformly, to the analytic continuations of the limits $\psi_0$, $\Phi_0$, $\Phix_0$ in \eqref{eq:closing}. At $\rhox=1$, condition \textbf{(ii)} is interpreted through $\psi_0(w)=1/[\alpha_x(w-1)]$ and $\Phi_0(w)=\Phix_0(w)=1/(w-1)$, as in \autoref{prop:limits}(i).
\end{assumption}

\paragraph{Consequences.}
On the confinement event in \textbf{(i)} the three functionals are analytic and bounded on $U_K$. Cauchy's formula on a smaller neighborhood of $K$ transfers the uniform convergence in \textbf{(ii)} to every fixed $w$-derivative; in particular $T_N\to T_0$ in $C^2$ on compacts of $(w_G,\infty)$, with $T_N$ from \eqref{eq:randomfp} (\autoref{app:baseline}). The Gram law is recovered from $\psi_N$: since $\nu_N$ has mass $\alpha_{x,N}^{-1}$, $\int\dd\nu_N(\lambda)/(w-\lambda)=(\psi_N(w)+\alpha_{x,N}^{-1})/w$. Fix a countable set $\mathcal D\subset U_K\cap\mathbb C^+$ with an accumulation point in $U_K\cap\mathbb C^+$. Given a subsequence, pass to a further subsequence along which confinement and the convergence $\psi_N(w)\to\psi_0(w)$ hold almost surely for every $w\in\mathcal D$. On the almost-sure event the measures $\nu_N$ are eventually supported in $[0,w_G+1]$ with masses $\alpha_{x,N}^{-1}\to\alpha_x^{-1}$, every subsequential weak limit has Stieltjes transform $(\psi_0(w)+\alpha_x^{-1})/w=\int\dd\nu_0(\lambda)/(w-\lambda)$ on $\mathcal D$, hence on $\mathbb C^+$ by analyticity, and a finite measure is determined by its Stieltjes transform. So $\nu_N\Rightarrow\nu_0$ in probability, with $\nu_0$ the Gram law of \autoref{app:onevar}, and the empirical spectral distribution of $AA^\top$ converges weakly in probability to $\nu_0+(1-\alpha_x^{-1})\delta_0$.

\autoref{prop:flatclosure} verifies the assumption for nested flat orthogonal designs. For unequal-leverage designs $\tau$ need not equal $1/\alpha_x$, since the coupling of \autoref{def:tau} depends on the scores as well as the row profile; persistent leverage variation generally makes the deterministic equivalents nonconstant on the row block, so the scalar reduction changes.

\paragraph{Necessity of incoherence.} Orthogonal columns and exactly equal leverage do not imply the deterministic equivalent. Let $\alpha_x$ be a fixed integer and consider a constant-ratio sequence $\alpha_{x,N}=N/\Dx\equiv\alpha_x$. Partition the $N$ rows into $\Dx$ blocks of size $\alpha_x$. Let column $i$ of $\Xs$ equal $\sqrt{\Dx}$ on block $i$ and zero elsewhere. Then
\[
  \Xs^\top\Xs=NI_{\Dx},
  \qquad
  \norm{(\Xs)_{k\cdot}}^2=\Dx
\]
exactly, but $\norm{\Xs}_{\max}=\sqrt{\Dx}$, so incoherence fails. The masked Gram matrix is diagonal, with
\[
  G_{ii}=\frac{\mathrm{Binomial}(\alpha_x,\rhox)}{\alpha_x\rhox},
\]
and therefore
\begin{equation}\label{eq:partitionmoment}
  \frac1{\Dx}\Tr G^2\;\longrightarrow\;
  1+\frac{1-\rhox}{\alpha_x\rhox}.
\end{equation}
For a flat incoherent profile, direct second-moment expansion (or the scalar system \eqref{eq:ghsystem} at infinity) gives instead
\begin{equation}\label{eq:flatmoment}
  \frac1{\Dx}\Tr G^2\;\longrightarrow\;
  1+\frac{1-\rhox^2}{\alpha_x}.
\end{equation}
The limits differ, for example at $\rhox=\tfrac12$; hence orthogonal columns and equal leverage do not imply the deterministic equivalent without entrywise incoherence.

In finite-size checks, the mask-averaged functionals for nonnested Paley and Gaussian-QR designs differ from their predicted limits on a scale consistent with $O(N^{-1/2})$ fluctuations of resolvent bilinear forms.

\clearpage
\section{The finite-signal proof}
\label{app:finitesignal}

We prove \autoref{thm:finitesignal} from \autoref{ass:closure}, \autoref{ass:incoherence}, \autoref{def:tau}, \autoref{prop:geometry} and \autoref{lem:residual}. Classical covariance results locate the noise edge \citep[Theorem 4.1, Lemma 6.1, Theorem 6.3(1)]{baisilverstein2010}. Under the stated moment bounds, matching the first two entry moments gives asymptotic resolvent comparison at fixed positive height (\autoref{lem:comparison}); a Gaussian calculation identifies the finite-$N$ comparison forms (\autoref{lem:gaussresp}). Exterior continuation (\autoref{prop:rectangle}) then gives the outlier by a sign change and interlacing (\autoref{thm:outlier}); the deterministic equivalent gives the curves (\autoref{cor:outlierclosure}, \autoref{app:curves}). Below threshold we control the edge subspace, then transfer both regimes through the residual.

\subsection{Setting and reduction to deterministic designs}
\label{app:setting}

Throughout, $(\Xs,u_0,v_0,b)$ are deterministic and we condition on the design mask $S_x$. We use the reduction \eqref{eq:barCN}, with the notation of \eqref{eq:pdefs}, the Gram measures of \autoref{app:onevar}, and the response objects of \autoref{app:baseline}. The remaining randomness is $(S_y,Z)$; the entries of $Q$ are iid, centered and of unit variance, with every fixed moment bounded at the fixed retention $\rhoy\in(0,1]$, since $\E\abs{Q_{kj}}^m=\rhoy^{1-m/2}\,\E\abs{Z_{11}}^m$ (\autoref{prop:dualmask}).

For rectangular $M$, let $H_M$ be the dilation \eqref{eq:hermitization} and $R_M(z):=(zI-H_M)^{-1}$. We embed first- and second-block vectors as $\mathbf u:=(u;0)$ and $\mathbf v:=(0;v)$; in particular, the three signal directions are $\hat{\mathbf a}_N,\mathbf u_0,\mathbf v_0$. The domains $\mathbb C^+$ and $\mathbb C^{++}$ are as in \autoref{app:baseline}. Write $K_0$ for a norm bound on $G$ and $C_\eta$ for constants depending on the height $\eta$ and the fixed parameters. Constants are uniform for $(\rhox,\alpha_{x,N}^{-1},\alpha_{y,N}^{-1})$ in a compact subset of $(0,1)\times(0,1]\times(0,\infty)$ and for $\theta$ in a compact set separated from $\theta_c$; limits are in probability unless stated.

For $M\in\R^{\Dx\times\Dy}$,
\begin{equation}\label{eq:responseblocks}
\begin{gathered}
  R_M(z)=\begin{pmatrix}
    z\,(z^2I-MM^\top)^{-1} & (z^2I-MM^\top)^{-1}M\\[2pt]
    M^\top(z^2I-MM^\top)^{-1} & z\,(z^2I-M^\top M)^{-1}
  \end{pmatrix},\\[2pt]
  \opnorm{R_M(z)}\le\abs{\operatorname{Im}z}^{-1},
  \qquad
  R_M(\bar z)=\overline{R_M(z)} .
\end{gathered}
\end{equation}
The five scalar forms and the secular function are those of \eqref{eq:exactforms} and \eqref{eq:secfun}; the determinant identity \eqref{eq:detratio} applies with $B=E$.

Standing bounds: on $\{\opnorm G\le K_0\}$, $\opnorm E\le\opnorm G^{1/2}\opnorm Q/\sqrt N\le K_E:=\sqrt{K_0}\,(1+\alpha_y^{-1/2})+1$ with probability tending to $1$ \citep[Theorem 5.8]{baisilverstein2010}; $\opnorm{\Rres_N}=o_P(1)$ (\autoref{lem:residual} and Markov's inequality).

\begin{lemma}[Reduction to deterministic designs]
\label{lem:reduction}
Choose a deterministic sequence $\mathfrak t_N\downarrow0$ so slowly that each of the following holds with probability tending to $1$: $\opnorm G\le K_0$; $\lambda_{\max}(G)\le w_G+\mathfrak t_N$; $\sup_{U_K}\big(\abs{\psi_N-\psi_0}+\abs{\Phi_N-\Phi_0}+\abs{\Phix_N-\Phix_0}\big)\le\mathfrak t_N$ for the finitely many neighborhoods $U_K$ of \autoref{ass:closure} used below; $\abs{\norm{a_N}^2-n_a}+\abs{\ip{a_N}{u_0}-\sqrt{\rhox}}\le\mathfrak t_N$; $\E[\opnorm{\Rres_N}\mid\text{design}]\le\mathfrak t_N$. Such a sequence exists because each quantity tends to zero in probability. It suffices to prove each conclusion of \autoref{thm:finitesignal}, conditionally on the design, along every deterministic sequence of designs satisfying these bounds; the remaining randomness is $(S_y,Z)$.
\end{lemma}

\begin{proof}
Let $\phi_N(S)\in[0,1]$ be the conditional probability, given design $S$, that a fixed conclusion fails by more than $\epsilon$. Let $B_N$ be the event that all displayed bounds hold. If $\E\phi_N\not\to0$, choose a subsequence and $\delta>0$ with $\E\phi_N>\delta$. Since $\Prob(B_N^c)\to0$, we have $\E[\phi_N\one_{B_N}]\ge\E\phi_N-\Prob(B_N^c)>\delta/2$ for all large $N$ on that subsequence. Some $S_N\in B_N$ therefore satisfies $\phi_N(S_N)>\delta/2$. The claimed result along every such deterministic design sequence gives $\phi_N(S_N)\to0$, a contradiction.
\end{proof}

For the assembly of \autoref{thm:finitesignal}, fix a design sequence from \autoref{lem:reduction}; probabilities now refer to the response randomness. The lemmas and conditional theorem below retain their own stated assumptions. We record the deterministic consequences of the deterministic equivalent used in the assembly.

By \autoref{ass:closure}, its recorded consequences and \autoref{prop:geometry},
\[
 \nu_N\Rightarrow\nu_0,\qquad s_N\to s_0,\qquad
 T_N\longrightarrow T_0\quad\text{in }C^2\text{ on compacts of }(w_G,\infty).
\]
On $[w_G+\gamma/2,K_2]$, for large fixed $K_2$, $\partial_wT_N$ converges uniformly to $\partial_wT_0$, which changes sign only at $w_R$. Since $\partial_wT_N$ is strictly increasing above $\lambda_{\max}(G)\le w_G+\mathfrak t_N$, its unique zero satisfies $w_{R,N}\to w_R$, and $z^2_{\mathrm{edge},N}=T_N(w_{R,N})\to T_0(w_R)$.

Each of $\psi_N,\Phi_N,\Phix_N$ is a transform $\int(w-\lambda)^{-1}\,\dd\mu(\lambda)$ with bounded total variation and support in $\spec G$. The measure-identification argument used for $\nu_N$ therefore gives convergence of all three, uniformly on every compact of $\mathbb C\setminus[0,w_G]$, where they are defined for large $N$. The identification argument in \autoref{lem:gaussresp} uses only $\nu_N\Rightarrow\nu_0$ and gives $m_N^{\det}\to m_0$ on $\mathbb C^+$ and $w_N\to w_0$, $c_N\to c_0$ on $\mathbb C^{++}$, locally uniformly. Since $\operatorname{Im}w_0>0$ there, composition gives
\[
 (\mone,\mtwo,\mcross,\Dbar,f_N)
 \longrightarrow(m_1,m_2,m_\times,\Dbar_0,f_0)
 \quad\text{locally uniformly on }\mathbb C^{++}.
\]
On real exterior compacts and the rectangles of \eqref{eq:rectangle}, the inverse function theorem applied to $T_N\to T_0$ on $U_K$ gives $w_N\to w_0$. Cauchy's formula then gives convergence of these transforms and all fixed derivatives on smaller neighborhoods.

\subsection{The response-noise edge}
\label{app:noiseedge}

\begin{lemma}[Support criterion]
\label{lem:support}
Assume $\opnorm G\le K_0$ and $G\ne0$, and let $F_N$ be the probability law with Stieltjes transform $m^{\det}_N$ of \eqref{eq:silversteinN}, the deterministic equivalent of the ESD of $E^\top E$. If $w>\lambda_{\max}(G)$ and $\partial_wT_N(w)>0$, then $T_N(w)\notin\supp F_N$; and $\sup\supp F_N=z^2_{\mathrm{edge},N}$. For real $u>z^2_{\mathrm{edge},N}$, $-\alpha_{y,N}/m^{\det}_N(u)$ is the unique solution of $T_N(w)=u$ in $(w_{R,N},\infty)$. The same statements hold for $\nu_0$, $T_0$ and the law $F_0$ with transform $m_0$ of \eqref{eq:silverstein0}; in particular $\sup\supp F_0=T_0(w_R)=z^2_{\mathrm{edge}}$.
\end{lemma}

\begin{proof}
For $F_N=F^{\alpha_{y,N},H_N}$, with $H_N$ the ESD of $T=AA^\top/\alpha_{y,N}$, write $f(m):=-1/m+\alpha_{y,N}\int t/(1+tm)\,\dd H_N(t)$. By \citet[Lemma 6.1]{baisilverstein2010}, the conditions $m\in\R\setminus\{0\}$, $-1/m\notin\supp H_N$ and $f'(m)>0$ imply $f(m)\notin\supp F_N$. Conversely, an exterior point $u$ has $m=m_N^{\det}(u)$ satisfying these conditions and $u=f(m)$ by their (6.1.6). Substituting \eqref{eq:silversteinN},
\[
 m=-\frac{\alpha_{y,N}}w,
 \qquad f(m)=T_N(w),
 \qquad \frac{\dd w}{\dd m}=\frac{\alpha_{y,N}}{m^2}>0.
\]
For $w>\lambda_{\max}(G)$ the population coordinate $-1/m=w/\alpha_{y,N}$ lies above $\spec T$, and $f'(m)>0$ exactly when $\partial_wT_N(w)>0$. Thus \eqref{eq:edgetransfer} gives
\[
 T_N((w_{R,N},\infty))=(z^2_{\mathrm{edge},N},\infty)
 \subset(\supp F_N)^c,
 \qquad \sup\supp F_N\le z^2_{\mathrm{edge},N}.
\]

For the reverse inequality, take $u>\sup\supp F_N$. Then $m_N^{\det}(u)$ is real, negative and increasing. Hence $w(u):=-\alpha_{y,N}/m_N^{\det}(u)$ is continuous and tends to infinity. By \citet[Lemma 6.1]{baisilverstein2010}, it avoids $\spec G$; continuity in \eqref{eq:silversteinN}, whose denominators are nonzero, gives $u=T_N(w(u))$. Continuity from infinity forces $w(u)>\lambda_{\max}(G)$, so
\[
 u=T_N(w(u))\ge\min_{w>\lambda_{\max}(G)}T_N(w)
   =z^2_{\mathrm{edge},N}\qquad(u>\sup\supp F_N).
\]
Letting $u\downarrow\sup\supp F_N$ proves the reverse inequality. For $u>z^2_{\mathrm{edge},N}$, $w(u)$ cannot equal $w_{R,N}$ because their images under $T_N$ differ. Again by continuity from infinity, $w(u)>w_{R,N}$, where $T_N$ is strictly increasing, proving uniqueness of the exterior inverse.

For the limiting law the same argument replaces $\lambda_{\max}(G)$ by $w_G=\sup\supp\nu_0$. The set $\supp\nu_0\cup\{0\}$ is closed; a continuous positive $w(u)$ avoiding it and tending to infinity stays in $(w_G,\infty)$. Thus $\sup\supp F_0=T_0(w_R)$ and the same branch identification holds.
\end{proof}

\begin{proposition}[Edge of the response noise]
\label{prop:noiseedge}
Along the design sequences of \autoref{lem:reduction}, for every fixed $\varepsilon>0$: (1) $\Prob(\sigma_1(E)>z_{\mathrm{edge}}+\varepsilon)\to0$; (2) $\sigma_2(E)\toP z_{\mathrm{edge}}$; (3) $\Prob(\opnorm E\le K_E)\to1$.
\end{proposition}

\begin{proof}
The standing norm bound gives (3). For (1)--(2), \citet[Theorem 6.3(1)]{baisilverstein2010}, the exclusion theorem of \citet{baisilverstein1998}, applies with
\[
 \begin{gathered}
 T=\frac{AA^\top}{\alpha_{y,N}}\in\R^{N\times N},\qquad
 Q\in\R^{N\times\Dy},\qquad \frac N\Dy\to\alpha_y,\\
 \frac1\Dy Q^\top TQ=\frac1NQ^\top AA^\top Q=E^\top E.
 \end{gathered}
\]
The population is deterministic given the design, $\opnorm T\le K_0/\alpha_{y,N}$, and its ESD converges to the image of $\nu_0+(1-\alpha_x^{-1})\delta_0$ under $\lambda\mapsto\lambda/\alpha_y$. The entries of $Q$ are iid with mean zero, unit variance and finite fourth moment. The theorem's hypothesis (f) concerns the finite-$N$ law $F_N$ of this companion.

For any fixed small $\varepsilon'>0$, choose $\delta'>0$ small and $K_2$ large. Closure, \eqref{eq:uniformmargins} and $T_N(w)\ge w/\alpha_{y,N}$ give, for large $N$,
\[
\begin{gathered}
 w_R+\delta'>w_G+\gamma/2\ge\lambda_{\max}(G),
 \qquad \inf_{[w_R+\delta',K_2]}\partial_wT_N\ge c/2>0,\\
 T_N(w_R+\delta')\le z^2_{\mathrm{edge}}+\varepsilon'/2,
 \qquad T_N(K_2)\ge K_E^2+1.
\end{gathered}
\]
For the final inequality, take $K_2\ge\sup_N\alpha_{y,N}(K_E^2+1)$. By \autoref{lem:support},
\[
 [z^2_{\mathrm{edge}}+\varepsilon',K_E^2]
 \subset\big(T_N(w_R+\delta'),T_N(K_2)\big)
 \subset(\supp F_N)^c.
\]
Under the nested iid coupling of the $Q$'s, the exclusion theorem places no eigenvalues of $E^\top E$ in this closed interval eventually almost surely. Combining with (3), and choosing $\varepsilon'$ smaller than $(z_{\mathrm{edge}}+\varepsilon)^2-z^2_{\mathrm{edge}}$, proves (1). If the target interval is empty, (3) already gives the bound.

By \citet[Theorem 4.1]{baisilverstein2010}, $\mathrm{ESD}(E^\top E)\Rightarrow F_0$ almost surely. By \autoref{lem:support}, $\sup\supp F_0=z^2_{\mathrm{edge}}$, so $F_0((z^2_{\mathrm{edge}}-\delta'',z^2_{\mathrm{edge}}])>0$ for every $\delta''>0$. Consequently at least two eigenvalues exceed $z^2_{\mathrm{edge}}-\delta''$ eventually. This lower bound and (1) give $\sigma_2(E)\toP z_{\mathrm{edge}}$, proving (2).
\end{proof}

\begin{remark}[No outlier without the deterministic equivalent]
Hypothesis (b) of \autoref{thm:outlier} below follows from its hypothesis (a) alone, for $Q$ and for the Gaussian $Q_{\mathrm G}$ alike. If $\Prob(\sigma_1(E)>z_{\mathrm{edge},N}+\varepsilon')\not\to0$ along a subsequence, extract a further subsequence along which $\nu_N\Rightarrow\nu_*$ (masses $\alpha_{x,N}^{-1}\to\alpha_x^{-1}$, supports in $[0,K_0]$), the aspect ratios converge, and $z_{\mathrm{edge},N}\to z_*$.

The laws $F_N$ have supports in $[0,z^2_{\mathrm{edge},N}]$ (\autoref{lem:support}), so the fixed interval $[z_*^2+\varepsilon'',K_E^2]$ lies in an open interval outside $\supp F_N$ for all large $N$ along the subsequence, and \citet[Theorem 6.3(1)]{baisilverstein2010} with the norm bound (3) gives $\sigma_1(E)\le z_*+\varepsilon'''\le z_{\mathrm{edge},N}+\varepsilon'$ with probability tending to $1$ along it, a contradiction.
\end{remark}

\subsection{Comparison at fixed positive height}
\label{app:comparison}

\begin{lemma}[Comparison of resolvent bilinear forms]
\label{lem:comparison}
Let $P_0\in\R^{d\times m}$ and $B\in\R^{d\times n}$ be deterministic with $d,m,n\le CN$ and $\opnorm B\le C$, and let $X,Y\in\R^{n\times m}$ have independent entries with $\E X_{kj}=\E Y_{kj}=0$, $\E X_{kj}^2=\E Y_{kj}^2$ and $\sup_{k,j}(\E\abs{X_{kj}}^p+\E\abs{Y_{kj}}^p)\le M_p$ for every fixed $p$. Put $Z_X:=P_0+BX/\sqrt N$, $Z_Y:=P_0+BY/\sqrt N$ ($d\times m$), $R_X(z):=(zI-H_{Z_X})^{-1}$, and $R_Y$ likewise. Then for every compact $\mathcal K\subset\mathbb C^+$,
\[
\begin{aligned}
  \sup_{z\in\mathcal K}\ \sup_{\norm r=\norm s=1}
  \big|\E\,\ip r{R_X(z)s}-\E\,\ip r{R_Y(z)s}\big|&\longrightarrow0,\\
  \sup_{z\in\mathcal K}\ \sup_{\norm r=\norm s=1}
  \big|\E\abs{\ip r{R_X(z)s}}^2-\E\abs{\ip r{R_Y(z)s}}^2\big|&\longrightarrow0 ;
\end{aligned}
\]
\end{lemma}

\begin{proof}
Let $\eta:=\min_{\mathcal K}\operatorname{Im}z$. We prove the bound $C_{\eta,\epsilon}N^{-1/2+\epsilon}$ for every $\epsilon>0$. Write
\[
\begin{gathered}
 H(x)=H_{P_0}+\sum_{k,j}x_{kj}F_{kj},\qquad
 F_{kj}=p_kq_j^\top+q_jp_k^\top,\\
 p_k=\frac{(Be_k;0)}{\sqrt N},\qquad q_j=(0;e_j),\qquad
 \norm{p_k}\le CN^{-1/2},\quad\norm{q_j}=1,\\
 \opnorm{\sum_kp_kp_k^\top}=\frac{\opnorm{BB^\top}}N\le\frac{C^2}N,
 \qquad \opnorm{\sum_jq_jq_j^\top}=1.
\end{gathered}
\]
Thus $H_{Z_X}=H(X)$. For fixed $z\in\mathcal K$ and unit $r,s$, put $f(x):=\ip r{(zI-H(x))^{-1}s}$; every real $x$ satisfies $\opnorm{(zI-H(x))^{-1}}\le\eta^{-1}$.

Let $X^\vartheta$, $0\le\vartheta\le1$, have independent entries with laws $\vartheta\mathcal L(X_{kj})+(1-\vartheta)\mathcal L(Y_{kj})$, independently of $X,Y$. Differentiating the product measure gives
\[
 \frac{\dd}{\dd\vartheta}\E f(X^\vartheta)
 =\sum_\ell\E[f(X^\vartheta_{(\ell)},X_\ell)-f(X^\vartheta_{(\ell)},Y_\ell)].
\]
Here $(X^\vartheta_{(\ell)},t)$ replaces entry $\ell=(k,j)$ by $t$. Set $R:=(zI-H(X^\vartheta))^{-1}$, the same matrix for all $\ell$. With $f_i:=\partial_\ell^if(X^\vartheta)$, Taylor expansion and independence give
\[
\begin{aligned}
 f(X^\vartheta_{(\ell)},X_\ell)-f(X^\vartheta_{(\ell)},Y_\ell)
 ={}&(X_\ell-Y_\ell)f_1\\
 &+\tfrac12[(X_\ell-X^\vartheta_\ell)^2-(Y_\ell-X^\vartheta_\ell)^2]f_2+\rho_\ell,\\
 \E[(X_\ell-Y_\ell)f_1]&=0,\\
 \E\{[(X_\ell-X^\vartheta_\ell)^2-(Y_\ell-X^\vartheta_\ell)^2]f_2\}
 &=\E[X_\ell^2-Y_\ell^2]\E f_2
   -2\E[X_\ell-Y_\ell]\E[X^\vartheta_\ell f_2]=0.
\end{aligned}
\]
The last two equalities use the matching moments. It remains to bound
\[
\begin{gathered}
 |\E f(X)-\E f(Y)|\le\sup_\vartheta\sum_\ell\E|\rho_\ell|,\\
 |\rho_\ell|\le\tfrac16(|X_\ell-X^\vartheta_\ell|^3+|Y_\ell-X^\vartheta_\ell|^3)
 \sup_t|\partial_\ell^3 f(X^\vartheta_{(\ell)},t)|,
\end{gathered}
\]
where $t$ lies between $X^\vartheta_\ell$ and $X_\ell$ or $Y_\ell$.

Differentiation gives $\partial_\ell R=RF_\ell R$ and $\partial_\ell^3f=6r^\top RF_\ell RF_\ell RF_\ell Rs$. Expanding the three insertions gives eight products
\[
 (r^\top Ra_1)(b_1^\top Ra_2)(b_2^\top Ra_3)(b_3^\top Rs),
 \qquad \{a_i,b_i\}=\{p_k,q_j\}.
\]
Each contains three $p_k$'s and three $q_j$'s. Bound the middle factors by $\norm a\norm b/\eta$, but keep the end factors $\alpha_k:=|r^\top Rp_k|$ and $\beta_j:=|r^\top Rq_j|$, with primes for $s$. The bounds $\sum_k\alpha_k^2\le C^2/(N\eta^2)$ and $\sum_j\beta_j^2\le\eta^{-2}$ also hold for the primed factors. The three endpoint types give
\[
\begin{aligned}
 (p,p):\quad &\sum_{k,j}\alpha_k\alpha'_k\frac{CN^{-1/2}}{\eta^2}
       \le\frac{C^3mN^{-3/2}}{\eta^4},\\
 (p,q):\quad &\sum_{k,j}\alpha_k\beta'_j\frac{C^2N^{-1}}{\eta^2}
       \le\frac{C^3\sqrt{mn}\,N^{-3/2}}{\eta^4},\\
 (q,q):\quad &\sum_{k,j}\beta_j\beta'_j\frac{C^3N^{-3/2}}{\eta^2}
       \le\frac{C^3nN^{-3/2}}{\eta^4}.
\end{aligned}
\]
These are Cauchy--Schwarz bounds; for the mixed case use $\sum_{k,j}\alpha_k\beta'_j\le\sqrt{n\sum_k\alpha_k^2}\sqrt{m\sum_j(\beta'_j)^2}$. The $(q,p)$ case is identical. Since $m,n\le CN$, their sum is $O_\eta(N^{-1/2})$.

Put $R^{(t)}:=(zI-H(X^\vartheta)-(t-X^\vartheta_\ell)F_\ell)^{-1}$. The identity $R^{(t)}=R+(t-X^\vartheta_\ell)RF_\ell R^{(t)}$ gives
\[
\begin{aligned}
 |r^\top R^{(t)}p_k|
 &\le\alpha_k(1+|t-X^\vartheta_\ell|CN^{-1/2}\eta^{-1})
       +|t-X^\vartheta_\ell|C^2N^{-1}\eta^{-1}\beta_j,\\
 |r^\top R^{(t)}q_j|
 &\le\beta_j(1+|t-X^\vartheta_\ell|CN^{-1/2}\eta^{-1})
       +|t-X^\vartheta_\ell|\eta^{-1}\alpha_k.
\end{aligned}
\]
The same bounds hold at the other end, and the middle factors retain their norm bounds. Expanding the two endpoint bounds gives a polynomial of degree at most two in $|t-X^\vartheta_\ell|$; exchanging an endpoint type gives one of the three sums above with the accompanying powers of $N^{-1/2}$. Thus
\[
 |\partial_\ell^3f(X^\vartheta_{(\ell)},t)|
 \le(1+|t-X^\vartheta_\ell|)^2\Xi_\ell,
 \qquad \sum_\ell\Xi_\ell\le C_\eta N^{-1/2}
\]
for every realization.

The cubic Taylor factor and the quadratic perturbation factor give
\[
\begin{aligned}
 \sum_\ell\E|\rho_\ell|
 &\le C_\eta N^{-1/2}\,
   \E\max_{\ell\le nm}(1+|X_\ell|+|Y_\ell|+|X^\vartheta_\ell|)^5\\
 &\le C_{\eta,\epsilon}N^{-1/2}(nm)^\epsilon
 \le C_{\eta,\epsilon}N^{-1/2+2\epsilon}.
\end{aligned}
\]
The middle inequality bounds the maximum by an $L^{1/\epsilon}$ norm and uses the uniform fixed-moment bounds. Replacing $\epsilon$ by $\epsilon/2$ gives the announced bound.

Apply the same interpolation to $f\bar f$. The product rule gives
\[
 \partial_\ell^3(f\bar f)
 =\bar f\partial_\ell^3f+3\partial_\ell\bar f\,\partial_\ell^2f
  +3\partial_\ell^2\bar f\,\partial_\ell f+f\partial_\ell^3\bar f.
\]
For the outer terms, $|f|\le\eta^{-1}$ and the third-derivative and remainder estimates above apply. For the middle terms,
\[
\begin{aligned}
 |\partial_\ell^2f|&\le8\norm{p_k}^2\norm{q_j}^2\eta^{-3}
                    \le8C^2N^{-1}\eta^{-3},\\
 \sum_{k,j}(\alpha_k\beta'_j+\beta_j\alpha'_k)
 &\le2\sqrt{mn}\,CN^{-1/2}\eta^{-2}.
\end{aligned}
\]
Their product is $O_\eta(N^{-1/2})$. The perturbed-point and weight estimates are unchanged, up to additional inverse-height factors. All bounds depend on $z$ only through $\eta$ and are uniform over the unit directions, proving both claims.
\end{proof}

\begin{corollary}[Directional $L^2$ convergence and traces]
\label{cor:comparisontrace}
In the setting of \autoref{lem:comparison}, let $M_N(z)$ be deterministic, complex symmetric, and uniformly bounded on $\mathcal K$, and let $\delta_N$ be the supremum over $z\in\mathcal K$ and unit $r,s$ of $\norm{\ip r{R_Ys}-\ip r{M_Ns}}_{L^2}$, with $\delta_N\to0$. Then $\sup_{\mathcal K}\sup_{\norm r=\norm s=1}\norm{\ip r{R_Xs}-\ip r{M_Ns}}_{L^2}\to0$, and for every deterministic real $\Omega$ with $\opnorm\Omega\le C$,
\[
  \sup_{\mathcal K}\big\|N^{-1}\Tr\,\Omega\,(R_X-M_N)\big\|_{L^2}
  \le\frac{d+m}N\,\opnorm\Omega\ \sup_{\mathcal K}\sup_{\norm u=1}\big\|\ip u{(R_X-M_N)u}\big\|_{L^2}
  \longrightarrow0 .
\]
\end{corollary}

\begin{proof}
For the directional conclusion, expand
\[
\begin{aligned}
 \E|\ip r{R_Xs}-\ip r{M_Ns}|^2
 ={}&\E|\ip r{R_Xs}|^2
 -2\operatorname{Re}\big(\overline{\ip r{M_Ns}}\,\E\ip r{R_Xs}\big)
 +|\ip r{M_Ns}|^2.
\end{aligned}
\]
By \autoref{lem:comparison}, the first two expectations differ from their $Y$-versions by $o(1)$ uniformly. Since $|\ip r{M_Ns}|\le\opnorm{M_N}$, the whole expression is at most $\delta_N^2+o(1)$.

For the trace, complex symmetry of $R_X-M_N$ gives $\Tr\Omega(R_X-M_N)=\Tr\Omega_{\mathrm{sym}}(R_X-M_N)$. Diagonalize $\Omega_{\mathrm{sym}}=\sum_i\tau_iu_iu_i^\top$. Minkowski's inequality yields
\[
\begin{aligned}
 \big\|N^{-1}\Tr\Omega(R_X-M_N)\big\|_{L^2}
 &\le\frac1N\sum_i|\tau_i|\,\big\|\ip{u_i}{(R_X-M_N)u_i}\big\|_{L^2}\\
 &\le\frac{d+m}{N}\opnorm\Omega\,
       \sup_{\norm u=1}\big\|\ip u{(R_X-M_N)u}\big\|_{L^2}.
\end{aligned}
\]
The supremum is over deterministic directions outside the expectation.
\end{proof}

\subsection{The Gaussian response calculation}
\label{app:gaussresp}

\begin{lemma}[Gaussian response model]
\label{lem:gaussresp}
Assume $\opnorm G\le K_0$ and $G\ne0$. Let $Q_{\mathrm G}$ have iid $\mathcal N(0,1)$ entries, $E_{\mathrm G}:=A^\top Q_{\mathrm G}/\sqrt N$, and $R:=R_{E_{\mathrm G}}(z)$. Then for every compact $\mathcal K\subset\mathbb C^{++}$, with $\eta:=\min_{\mathcal K}\operatorname{Im}z$,
\[
  \sup_{\mathcal K}\sup_{\norm r=\norm s=1}\Var\big(\ip r{Rs}\big)\le\frac{4K_0}{N\eta^4},
  \qquad
  \sup_{\mathcal K}\big\|\E R-\mathbf R^{\det}_N(z)\big\|_{\mathrm{op}}\longrightarrow0,
\]
with $\mathbf R^{\det}_N$ the deterministic equivalent of \autoref{app:baseline}. Consequently, uniformly on $\mathcal K$ and over unit $r,s$, $\E\abs{\ip r{Rs}-\ip r{\mathbf R^{\det}_Ns}}^2\to0$.
\end{lemma}

\begin{proof}
Write $y_k:=A^\top e_k/\sqrt N\in\R^{\Dx}$, $\tilde y_k:=(y_k;0)$ and $\tilde f_j:=(0;e_j)\in\R^{\Dx+\Dy}$. Then
\[
 H_{E_{\mathrm G}}=\sum_{k,j}Q_{kj}F_{kj},\qquad
 F_{kj}=\tilde y_k\tilde f_j^\top+\tilde f_j\tilde y_k^\top,
 \qquad \sum_k\tilde y_k\tilde y_k^\top=\diag(G/N,0).
\]

Differentiate, use $\opnorm R\le\eta^{-1}$, and sum with the preceding identity:
\[
\begin{gathered}
 \partial_{Q_{kj}}\ip r{Rs}
 =\ip r{R\tilde y_k}\ip{\tilde f_j}{Rs}
  +\ip r{R\tilde f_j}\ip{\tilde y_k}{Rs},\\
 \sum_k|\ip r{R\tilde y_k}|^2\le\frac{K_0}{N\eta^2},
 \qquad \sum_j|\ip{\tilde f_j}{Rs}|^2\le\eta^{-2},\\
 \sum_{k,j}|\partial_{Q_{kj}}\ip r{Rs}|^2
 \le\frac{4K_0}{N\eta^4},
 \qquad \Var(\ip r{Rs})\le\frac{4K_0}{N\eta^4}.
\end{gathered}
\]
The last inequality is Gaussian Poincar\'e; linearity of $H_{E_{\mathrm G}}$ makes the derivative bound global. Put $\hat t_N:=N^{-1}\Tr R_{22}$, $t_N:=\E\hat t_N$ and $m_N^{\mathrm G}(\zeta):=\Dy^{-1}\Tr(E_{\mathrm G}^\top E_{\mathrm G}-\zeta I)^{-1}$. Minkowski's inequality and \eqref{eq:responseblocks} give
\[
 \norm{\hat t_N-t_N}_{L^2}
 \le\frac\Dy N\sup_{\norm u=1}\Var(\ip u{Ru})^{1/2}=O(N^{-1/2}),
 \qquad t_N=-\alpha_{y,N}^{-1}z\,\E m_N^{\mathrm G}(z^2).
\]

For orthogonal $O\in\R^{\Dy\times\Dy}$, $Q_{\mathrm G}O\eqd Q_{\mathrm G}$ and $H_{E_{\mathrm G}O}=\diag(I,O^\top)H_{E_{\mathrm G}}\diag(I,O)$. Thus $\E R_{22}$ commutes with every $O$, while $\E R_{12}=\E R_{12}O$. Taking traces gives $\E R_{22}=\alpha_{y,N}t_NI_{\Dy}$, and $O=-I$ gives $\E R_{12}=0$.

The block equation is $zR_{11}-E_{\mathrm G}R_{21}=I$. Resolvent differentiation and Gaussian integration by parts give
\[
\begin{aligned}
 (\partial_{Q_{kj}}R)_{21}
 &=(R_{21}y_k)(e_j^\top R_{21})+(R_{22}e_j)(y_k^\top R_{11}),\\
 \E[E_{\mathrm G}R_{21}]
 &=\sum_{k,j}y_k\E[\partial_{Q_{kj}}(R_{21})_{j\cdot}]
 =G\E[\hat t_NR_{11}]+\frac GN\E[R_{12}R_{21}].
\end{aligned}
\]
Center both factors in the covariance. For real unit $u,v\in\R^{\Dx}$,
\[
\begin{aligned}
 |\ip u{\E[(\hat t_N-t_N)R_{11}]v}|
 &=\big|\E[(\hat t_N-t_N)(\ip u{R_{11}v}-\E\ip u{R_{11}v})]\big|\\
 &\le\norm{\hat t_N-t_N}_{L^2}\Var(\ip u{R_{11}v})^{1/2}
 =O(N^{-1}).
\end{aligned}
\]
Real unit tests control the complex operator norm up to a factor $2$. Also $\opnorm{(G/N)R_{12}R_{21}}\le K_0/(N\eta^2)$. Hence
\[
\begin{gathered}
 (zI-t_NG)\E R_{11}=I+O(N^{-1}),\qquad
 \opnorm{(zI-t_NG)^{-1}}\le\eta^{-1},\\
 \E R_{11}=(zI-t_NG)^{-1}+O(N^{-1}).
\end{gathered}
\]
The inverse bound follows from $\operatorname{Im}t_N\le0$ and $G\succeq0$.

Suppose $\sup_{\mathcal K}|t_N-c_N|\not\to0$. Pass to a subsequence with $z_N\to z_*\in\mathcal K$ and $|t_N(z_N)-c_N(z_N)|\ge\delta>0$. The measures $\nu_N$ have bounded support and masses $\alpha_{x,N}^{-1}\to\alpha_x^{-1}$, so a further subsequence satisfies $\nu_N\Rightarrow\nu_*$. Apply \citet[Theorem 4.1]{baisilverstein2010} to $E_{\mathrm G}^\top E_{\mathrm G}=\Dy^{-1}Q_{\mathrm G}^\top TQ_{\mathrm G}$ with $T=AA^\top/\alpha_{y,N}$. It gives $m_N^{\mathrm G}(\zeta)\to m_*(\zeta)$ almost surely, where $m_*$ is the unique solution in $\mathbb C^+$ of \eqref{eq:silverstein0} with $\nu_*$ replacing $\nu_0$ \citep[(6.1.5)]{baisilverstein2010}. Bounded convergence and equicontinuity give $\E m_N^{\mathrm G}\to m_*$ locally uniformly on $\mathbb C^+$.

The laws with transforms $m_N^{\det}$ have common compact support. Every subsequential limit is therefore a probability Stieltjes transform, with $\operatorname{Im}m_N^{\det}(\zeta)\ge\operatorname{Im}\zeta/(C+|\zeta|)^2>0$. This keeps $1+\lambda m_N^{\det}/\alpha_{y,N}$ away from zero for $\lambda\in[0,K_0]$. Passing to the limit in \eqref{eq:silversteinN} identifies the same unique solution $m_*$. Thus $m_N^{\det}\to m_*$ locally uniformly and
\[
 t_N-c_N=-\alpha_{y,N}^{-1}z\big(\E m_N^{\mathrm G}-m_N^{\det}\big)(z^2)
 \longrightarrow0
\]
uniformly near $z_*$, a contradiction.

By the resolvent identity and the bounds from \autoref{app:baseline},
\[
 \opnorm{(zI-t_NG)^{-1}-(zI-c_NG)^{-1}}
 \le K_0\eta^{-2}|t_N-c_N|\longrightarrow0.
\]
Together with $\alpha_{y,N}|t_N-c_N|\to0$ and the symmetry and mean-block identities above, this gives $\opnorm{\E R-\mathbf R_N^{\det}}\to0$. The variance bound then gives the stated directional $L^2$ convergence.
\end{proof}

\begin{proposition}[Masked model at fixed height]
\label{prop:maskedheight}
Assume $\opnorm G\le K_0$ and $G\ne0$. For every compact $\mathcal K\subset\mathbb C^{++}$,
\[
  \sup_{z\in\mathcal K}\ \sup_{\norm r=\norm s=1}\ \E\big|\ip r{R_E(z)s}-\ip r{\mathbf R^{\det}_N(z)s}\big|^2\longrightarrow0 ;
\]
in particular, uniformly on $\mathcal K$, $m_{1,E},m_{2,E},m_{\times,E},g_E,g_{\times,E}$ are within $o_P(1)$ of $\mone,\mtwo,\mcross,0,0$, and $\ip{\mathbf u_0}{R_E\mathbf u_0}$ is within $o_P(1)$ of $c_N^{-1}\ip{u_0}{(w_NI-G)^{-1}u_0}$.
\end{proposition}

\begin{proof}
\autoref{lem:comparison} with $P_0=0$, $B=A^\top$ ($\opnorm B\le\sqrt{K_0}$), $X=Q$, $Y=Q_{\mathrm G}$, together with \autoref{lem:gaussresp} and \autoref{cor:comparisontrace} ($\mathbf R^{\det}_N$ is complex symmetric and bounded on $\mathcal K$ by \autoref{app:baseline}), gives the display; the entries of $\mathbf R^{\det}_N$ in the three directions are those listed in \autoref{app:baseline}. A finite net of $\mathcal K$ with the Lipschitz bounds $\opnorm{R_E(z)-R_E(z')}\le\abs{z-z'}\eta^{-2}$ and $\opnorm{\mathbf R^{\det}_N(z)-\mathbf R^{\det}_N(z')}\le C_\eta\abs{z-z'}$ turns the pointwise $L^2$ convergence into uniform convergence in probability.
\end{proof}

\subsection{Continuation to the real exterior}
\label{app:rectangle}

\begin{proposition}[Uniform convergence on a rectangle]
\label{prop:rectangle}
Assume $\opnorm G\le K_0$, $G\ne0$ and $s_N\le C_s$. Fix $\varepsilon>0$ and $K_1<\infty$, put $\eta_1:=\eta(\varepsilon/2,K_1+1)$ from \eqref{eq:rectangle}, and let
\[
\begin{aligned}
  \mathcal S_N&:=\{z:\ z_{\mathrm{edge},N}+\varepsilon/2\le\operatorname{Re}z\le K_1+1,\ \abs{\operatorname{Im}z}\le\eta_1\},\\
  \mathcal S_N^{\circ}&:=\{z:\ z_{\mathrm{edge},N}+3\varepsilon/4\le\operatorname{Re}z\le K_1+\tfrac12,\ \abs{\operatorname{Im}z}\le\eta_1/2\}.
\end{aligned}
\]
Assume $\Prob(\sigma_1(E)\le z_{\mathrm{edge},N}+\varepsilon/4)\to1$. Then
\[
  \sup_{z\in\mathcal S_N}\ \sup_{r,s\in\{\hat{\mathbf a}_N,\mathbf u_0,\mathbf v_0\}}
  \big|\ip r{R_E(z)s}-\ip r{\mathbf R^{\det}_N(z)s}\big|\toP0,
  \qquad
  \sup_{\mathcal S_N}\abs{F_E-f_N}\toP0,
\]
and the same holds for all $z$-derivatives on $\mathcal S_N^\circ$. Under \autoref{ass:closure} the finite-$N$ objects other than the $(\mathbf u_0,\mathbf u_0)$ entry may be replaced by their limits, and the rectangles by the fixed ones with left endpoints $z_{\mathrm{edge}}+\varepsilon$ and $z_{\mathrm{edge}}+\tfrac54\varepsilon$; the $(\mathbf u_0,\mathbf u_0)$ entry stays bounded by $\abs{c_N^{-1}}/\operatorname{dist}(w_N,\spec G)$.
\end{proposition}

\begin{proof}
On the confinement event, $\spec(H_E)\subset[-z_{\mathrm{edge},N}-\varepsilon/4,z_{\mathrm{edge},N}+\varepsilon/4]$. By \autoref{app:baseline}, with parameters $\varepsilon/2,K_1+1$, $\mathbf R_N^{\det}$ is analytic on $\mathcal S_N$. Thus
\[
 \sup_{\mathcal S_N}\opnorm{R_E}\le4/\varepsilon,
 \qquad \sup_{\mathcal S_N}\big(\opnorm{\mathbf R_N^{\det}}
                +\opnorm{(\mathbf R_N^{\det})'}\big)\le C.
\]

Fix $0<\eta\le\eta_1$. For $z=x+iy\in\mathcal S_N$ with $0\le y\le\eta$, use $R_E(z)-R_E(x+i\eta)=i(\eta-y)R_E(z)R_E(x+i\eta)$. For the stated unit directions,
\[
\begin{aligned}
 |\ip r{(R_E(z)-R_E(x+i\eta))s}|&\le16\eta/\varepsilon^2,\\
 |\ip r{(\mathbf R_N^{\det}(z)-\mathbf R_N^{\det}(x+i\eta))s}|&\le C\eta.
\end{aligned}
\]
The part with $y\ge\eta$ is already at positive height; negative heights follow by conjugation. Therefore
\[
 \sup_{\mathcal S_N}|\ip r{(R_E-\mathbf R_N^{\det})s}|
 \le\sup_{\mathcal S_N\cap\{\operatorname{Im}z\ge\eta\}}
       |\ip r{(R_E-\mathbf R_N^{\det})s}|+C'\eta.
\]
The supremum on the right lies in the fixed compact
$\{\varepsilon/2\le\operatorname{Re}z\le K_1+1,\ \eta\le\operatorname{Im}z\le\eta_1\}\subset\mathbb C^{++}$.
Apply \autoref{prop:maskedheight}, first letting $N\to\infty$ and then $\eta\downarrow0$. All stated entries converge uniformly in probability, and their bounded polynomial combination gives $F_E-f_N\toP0$ uniformly as well.

Since $\operatorname{dist}(\mathcal S_N^\circ,\partial\mathcal S_N)\ge\min\{\varepsilon/4,1/2,\eta_1/2\}>0$, Cauchy's formula gives convergence of each fixed derivative there. Under the deterministic equivalent, the inputs in \autoref{app:setting} give convergence of $\mone,\mtwo,\mcross,f_N$ and their derivatives on the stated fixed rectangles, which lie inside $\mathcal S_N,\mathcal S_N^\circ$ for large $N$. The $(\mathbf u_0,\mathbf u_0)$ entry has the stated resolvent bound; no limit for it is assumed.
\end{proof}

\subsection{The outlier theorem}
\label{app:outliertheorem}

\begin{theorem}[Conditional outlier and overlaps]
\label{thm:outlier}
Condition on the design, so that $(A,G,\han_N,s_N)$ are deterministic, and assume with constants uniform in $N$:
\begin{enumerate}
\item[(a)] $\opnorm G\le K_0$, $G\ne0$, and $0<s_N\le C_s$;
\item[(b)] for every fixed $\varepsilon'>0$, $\sigma_1(E)\le z_{\mathrm{edge},N}+\varepsilon'$ with conditional probability tending to $1$;
\item[(c)] the finite-$N$ secular function $f_N$ has a real root $\hz_N\ge z_{\mathrm{edge},N}+\varepsilon$ for a fixed $\varepsilon>0$.
\end{enumerate}
Then, with conditional probability tending to $1$, $\Cbar_N$ has exactly one singular value above $z_{\mathrm{edge},N}+\varepsilon/2$; it is simple, equals $\hz_N+o_P(1)$, and the squared overlaps of its singular pair $(\bar u,\bar v)$ are, up to $o_P(1)$, the residue forms of \autoref{prop:overlapidentities} at $\hz_N$:
\[
\begin{gathered}
  (\ip{\bar u}{\han_N})^2=\frac{-2\mone}{s_N^2\Dbar'}+o_P(1),
  \qquad
  (\ip{\bar v}{v_0})^2=\frac{-2\mtwo}{s_N^2\Dbar'}+o_P(1),\\
  (\ip{\bar u}{u_0})^2=\frac{-2\mone}{s_N^2\Dbar'}\Big(\frac{\mcross}{\mone}\Big)^2+o_P(1).
\end{gathered}
\]
\end{theorem}

\begin{proof}[Proof of \autoref{thm:outlier}]
By \autoref{app:baseline}, $f_N'(\hz_N)\ge2\hz_N/s_N^2\ge c:=2\varepsilon/C_s^2$. Its root bound $\hz_N^2\le2K_0(\alpha_{y,N}^{-1}+\alpha_{x,N}^{-1})+C_s^2$ gives the fixed cutoff
\[
 K_1:=1+\sup_N\big\{2K_0(\alpha_{y,N}^{-1}+\alpha_{x,N}^{-1})+C_s^2\big\}^{1/2},
 \qquad I:=[z_{\mathrm{edge},N}+\varepsilon/2,K_1].
\]
It is finite on the compact parameter set, and $\hz_N\le K_1-1$. On $I$, $f_N$ is real and strictly increasing; the rectangle bounds give $|f_N''|\le M$ for a fixed $M>0$. With $r:=\min\{\varepsilon/4,1/2,c/(2M)\}$,
\[
\begin{gathered}
 [\hz_N-r,\hz_N+r]\subset[z_{\mathrm{edge},N}+3\varepsilon/4,K_1-1/2]
 \subset I,\\
 f_N'\ge c/2\quad\text{on this interval},\qquad
 f_N(\hz_N-r)\le-cr/2<0<cr/2\le f_N(\hz_N+r),\\
 |f_N(z)|\ge\tfrac c2|z-\hz_N|\quad(|z-\hz_N|\le r,\ z\in\R).
\end{gathered}
\]

By \autoref{prop:rectangle} and hypothesis (b), with probability tending to $1$, $\sup_I|F_E-f_N|<cr/2$ and $\sigma_1(E)\le z_{\mathrm{edge},N}+\varepsilon/4$. On this event $F_E$ is real and continuous on $I$ and has opposite endpoint signs at $\hz_N\pm r$. It therefore has a real zero $\check z$ between them. Since $\check z>\sigma_1(E)$, \eqref{eq:detratio} makes it a singular value of $\Cbar_N$. Rank-one interlacing gives $\sigma_2(\Cbar_N)\le\sigma_1(E)\le z_{\mathrm{edge},N}+\varepsilon/4<\check z$. Thus $\check z=\sigma_1(\Cbar_N)$ is simple and is the only singular value above $z_{\mathrm{edge},N}+\varepsilon/2$. Its displacement obeys
\[
 \tfrac c2|\check z-\hz_N|
 \le|f_N(\check z)|
 =|f_N(\check z)-F_E(\check z)|
 \le\sup_I|F_E-f_N|=o_P(1).
\]

By \autoref{prop:overlapidentities}, the first two squared overlaps at $\check z$ are $2m_{1,E}(\check z)/F_E'(\check z)$ and $2m_{2,E}(\check z)/F_E'(\check z)$; the true-direction overlap is \eqref{eq:residuecross}. On $[\hz_N-r,\hz_N+r]\subset\mathcal S_N^\circ$, \autoref{prop:rectangle} compares the five forms and $F_E'$ with $\mone,\mtwo,\mcross,0,0$ and $f_N'$. Uniform derivative bounds and the displacement estimate allow evaluation at $\hz_N$, at cost $O(|\check z-\hz_N|)=o_P(1)$.

Both denominators stay away from zero. In fact,
\[
\begin{gathered}
 f_N'(\hz_N)=-s_N^2\Dbar'(\hz_N)\ge c,\\
 w_N(\hz_N)\ge\alpha_{y,N}\varepsilon^2
 \quad\Longrightarrow\quad
 \mtwo(\hz_N)=\frac{\alpha_{y,N}\hz_N}{w_N(\hz_N)}
 \le\frac{K_1-1}{\varepsilon^2},\\
 \mone(\hz_N)=\frac1{s_N^2\mtwo(\hz_N)}
 \ge\frac{\varepsilon^2}{C_s^2(K_1-1)}.
\end{gathered}
\]
Substituting in the residue identities gives all three stated formulas.
\end{proof}

\begin{corollary}[Under the deterministic equivalent]
\label{cor:outlierclosure}
Let $\theta>\theta_c$ and let $\hz>z_{\mathrm{edge}}$ be the root of $f_0$ (\autoref{prop:curves}(ii)). Along the design sequences of \autoref{lem:reduction}, (a)--(c) hold with $\varepsilon:=\tfrac12(\hz-z_{\mathrm{edge}})$ and $\hz_N\to\hz$; consequently $\sigma_1(\Cbar_N)\toP\hz$, $\sigma_1(\Cbar_N)-\sigma_2(\Cbar_N)\ge\varepsilon/2$ with probability tending to $1$, and the overlaps converge to \eqref{eq:overlaplims}, that is, to the curves \eqref{eq:curves}.
\end{corollary}

\begin{proof}
Hypothesis (a) follows from \autoref{lem:reduction} and \autoref{prop:geometry}: $s_N\to s_0>0$ and $\norm{a_N}^2\to n_a>0$, hence $G\ne0$ eventually. Hypothesis (b) is \autoref{prop:noiseedge}(1) with $z_{\mathrm{edge},N}\to z_{\mathrm{edge}}$.

For (c), the inputs of \autoref{app:setting} give uniform convergence of $f_N$ and its derivatives near $\hz$. For every sufficiently small fixed $r>0$, the signs $f_0(\hz-r)<0<f_0(\hz+r)$ give $f_N(\hz-r)<0<f_N(\hz+r)$ for large $N$. Strict monotonicity gives the unique root $\hz_N$ between these endpoints. Since $r$ is arbitrary, $\hz_N\to\hz$ and $\hz_N\ge z_{\mathrm{edge},N}+\varepsilon$ eventually. Apply \autoref{thm:outlier}; its residue forms converge to those of \autoref{prop:curves}.
\end{proof}

\begin{remark}[Relation to the replica analysis]
The objects of \autoref{thm:outlier} are the exact algebraic finite-$N$ quantities of \eqref{eq:fixedpoint}--\eqref{eq:transforms}; the saddle equations of \autoref{app:replica} are equivalent to $f_N(\hz_N)=0$ and the residue forms above, subject to assumptions (R1)--(R5) stated there.
\end{remark}

\subsection{The subcritical regime}
\label{app:subcritical}

Below threshold there is no separated root, and operator-norm control of $\Rres_N$ cannot resolve individual eigenvectors at the edge; we control the whole edge spectral subspace of $\Cbar_N$ and transfer that control through the residual. Throughout, the design is a sequence of \autoref{lem:reduction}, $\theta<\theta_c$ is fixed, and $C_s\ge\sup_Ns_N$.

Fix $0<\varepsilon\le1$, and set $K_1':=K_1+C_s+1$ using the cutoff from \autoref{thm:outlier}. Since $z_{\mathrm{edge},N}\le K_1-1$, its limit has the same bound. On the event $\sigma_1(E)\le z_{\mathrm{edge}}+\varepsilon/4$ and $s_N\le C_s$, we have $\opnorm{\Cbar_N}\le\sigma_1(E)+s_N\le z_{\mathrm{edge}}+\varepsilon/4+C_s\le K_1'-1$. Every singular value above $z_{\mathrm{edge}}+2\varepsilon$ is therefore a real zero of $F_E$ in the fixed interval $[z_{\mathrm{edge}}+2\varepsilon,K_1'-1]$, by \eqref{eq:detratio}. On that interval, \autoref{prop:rectangle} gives $\sup|F_E-f_0|\toP0$, while monotonicity and \eqref{eq:edgemargin} give
\[
 f_0(z)\ge1-s_0^2\Dbar_0(z_{\mathrm{edge}})
          =1-\theta^2/\theta_c^2>0.
\]
Thus there is no such zero with probability tending to $1$. For the lower bound, the spike vanishes on $v_0^\perp$, so Courant--Fischer and \autoref{prop:noiseedge}(2) give
\[
 \sigma_1(\Cbar_N)
 \ge\sup_{\varphi\perp v_0,\,\norm\varphi=1}\norm{E\varphi}
 \ge\sigma_2(E)\toP z_{\mathrm{edge}}.
\]
Since $\varepsilon$ was arbitrary, $\sigma_1(\Cbar_N)\toP z_{\mathrm{edge}}$.

For $\mathbf r\in\{\hat{\mathbf a}_N,\mathbf u_0,\mathbf v_0\}$, let $\mu_{\mathbf r,N}(\mathcal I):=\ip{\mathbf r}{\one_{\mathcal I}(H_{\Cbar_N})\mathbf r}$. We will prove
\begin{equation}\label{eq:edgemass}
  \lim_{\varsigma\downarrow0}\limsup_{N\to\infty}
  \Prob\Big(\mu_{\mathbf r,N}\big([z_{\mathrm{edge}}-\varsigma,\infty)\big)>\epsilon\Big)=0
  \qquad\text{for every }\epsilon>0 .
\end{equation}
With $U,J$ from the proof of \autoref{prop:secular}, Woodbury's identity gives
\[
\begin{aligned}
 \ip{\mathbf r}{R_{\Cbar_N}(z)\mathbf r}
 ={}&\ip{\mathbf r}{R_E\mathbf r}\\
 &+s_N(\mathbf r^\top R_EU)J
   (I_2-s_NU^\top R_EUJ)^{-1}(U^\top R_E\mathbf r),\\
 \det(I_2-s_NU^\top R_EUJ)&=F_E(z).
\end{aligned}
\]
It suffices to bound this expression at $z_\varsigma:=z_{\mathrm{edge}}+i\varsigma$, uniformly for small fixed $\varsigma$ after taking $N\to\infty$.

First fix $\varsigma>0$. By \autoref{prop:maskedheight} and the inputs of \autoref{app:setting},
\[
 U^\top R_E(z_\varsigma)U\toP
 \begin{pmatrix}m_1(z_\varsigma)&0\\0&m_2(z_\varsigma)\end{pmatrix},
 \qquad F_E(z_\varsigma)\toP f_0(z_\varsigma).
\]
The true-direction entries obey
\[
\begin{gathered}
 \ip{\mathbf u_0}{R_E\mathbf u_0}
   =c_N^{-1}\ip{u_0}{(w_NI-G)^{-1}u_0}+o_P(1),\\
 \mathbf u_0^\top R_EU
   =\big[c_N^{-1}\ip{u_0}{(w_NI-G)^{-1}\han_N},\ 0\big]+o_P(1),\\
 w_N=w_N(z_\varsigma)\to w_0(z_\varsigma),
 \qquad c_N^{-1}\to w_0(z_\varsigma)/z_\varsigma.
\end{gathered}
\]
To control these limits as $\varsigma\downarrow0$, use the fold at $w_R$. The map $T_0$ is real analytic there; \eqref{eq:curvaturebound} gives $T_0(w)-z_{\mathrm{edge}}^2=\chi(w)^2$ locally, with $\chi$ conformal, $\chi(w_R)=0$ and $\chi'(w_R)>0$. For $\zeta\in\mathbb C^+$ close to the edge, the two nearby inverse values have imaginary parts of opposite sign. The upper one gives $m=-\alpha_y/w\in\mathbb C^+$ solving \eqref{eq:silverstein0}; uniqueness in \citet[(6.1.5)]{baisilverstein2010} identifies it with $m_0(\zeta)$. Consequently,
\[
\begin{aligned}
 w_0(z_\varsigma)&\longrightarrow w_R,\\
 f_0(z_\varsigma)
 &=1-s_0^2\alpha_y\Phi_0(w_0(z_\varsigma))
 \longrightarrow1-s_0^2\Dbar_0(z_{\mathrm{edge}})
 =1-\theta^2/\theta_c^2>0.
\end{aligned}
\]
The last identity is \eqref{eq:edgemargin}. Choose $\varsigma_0>0$ so that, for $0<\varsigma\le\varsigma_0$,
\[
 |w_0(z_\varsigma)-w_R|\le\gamma/4,
 \quad |w_0(z_\varsigma)/z_\varsigma|\le\frac{3w_R}{2z_{\mathrm{edge}}},
 \quad |f_0(z_\varsigma)|\ge\tfrac12(1-\theta^2/\theta_c^2),
\]
and $w_0(z_\varsigma)\in U_K$ for $K=[w_R-\gamma/2,w_R+\gamma/2]$ in \autoref{ass:closure}.

For each such fixed $\varsigma$ and $N\ge N_0(\varsigma)$, take
$|w_N-w_0(z_\varsigma)|\le\min\{\gamma/4,w_R/2\}$. On the confinement event, \eqref{eq:uniformmargins} gives
\[
\begin{gathered}
 \lambda_{\max}(G)\le w_G+\gamma/4,
 \qquad \operatorname{dist}(w_N,\spec G)\ge\gamma/4,
 \qquad |c_N^{-1}|\le\frac{2w_R}{z_{\mathrm{edge}}},\\
 \big|c_N^{-1}\ip{u_0}{(w_NI-G)^{-1}u'}\big|
 \le\frac{|c_N^{-1}|}{\operatorname{dist}(w_N,\spec G)}
 \le\frac{8w_R}{\gamma z_{\mathrm{edge}}}
 \quad(u'\in\{u_0,\han_N\}).
\end{gathered}
\]
The remaining directional limits are also bounded, and
$|F_E(z_\varsigma)|\ge\tfrac14(1-\theta^2/\theta_c^2)$ with probability tending to $1$. The $2\times2$ inverse in Woodbury's formula is bounded by $C/|F_E|$. We have therefore obtained constants $\varsigma_0,C_*>0$, independent of the chosen $\varsigma$, such that
\[
 \Prob\big(|\ip{\mathbf r}{R_{\Cbar_N}(z_\varsigma)\mathbf r}|\le C_*\big)
 \longrightarrow1
 \quad\text{for each fixed }0<\varsigma\le\varsigma_0.
\]

For $|\lambda-z_{\mathrm{edge}}|\le\varsigma$, the Poisson kernel obeys $\varsigma/((\lambda-z_{\mathrm{edge}})^2+\varsigma^2)\ge1/(2\varsigma)$. Integrating gives
\[
\begin{aligned}
 \mu_{\mathbf r,N}([z_{\mathrm{edge}}-\varsigma,z_{\mathrm{edge}}+\varsigma])
 &\le-2\varsigma\operatorname{Im}\ip{\mathbf r}{R_{\Cbar_N}(z_\varsigma)\mathbf r}\\
 &\le2\varsigma|\ip{\mathbf r}{R_{\Cbar_N}(z_\varsigma)\mathbf r}|.
\end{aligned}
\]
The convergence $\sigma_1(\Cbar_N)\toP z_{\mathrm{edge}}$ excludes spectrum above $z_{\mathrm{edge}}+\varsigma$ with probability tending to $1$. Hence, for each fixed $\varsigma\le\varsigma_0$, the mass of $[z_{\mathrm{edge}}-\varsigma,\infty)$ is at most $2\varsigma C_*$ with probability tending to $1$. First take $N\to\infty$, then $\varsigma\downarrow0$; choosing $\varsigma<\epsilon/(2C_*)$ proves \eqref{eq:edgemass}.

Put $P_{N,\varsigma}:=\one_{[z_{\mathrm{edge}}-\varsigma,\infty)}(H_{\Cbar_N})$, and let $x_N=2^{-1/2}(\hu;\hv)$ be a unit eigenvector of $H_C=H_{\Cbar_N}+H_{\Rres_N}$ for $\lambda_N:=\sigma_1(C)$. Weyl's inequality gives $|\lambda_N-\sigma_1(\Cbar_N)|\le\opnorm{\Rres_N}=o_P(1)$, hence $\lambda_N\toP z_{\mathrm{edge}}$. On $\{\lambda_N\ge z_{\mathrm{edge}}-\varsigma/2\}$, the gap on $\range(I-P_{N,\varsigma})$ is at least $\varsigma/2$. Projecting the eigenvalue equation therefore gives
\[
\begin{gathered}
 (\lambda_N-H_{\Cbar_N})(I-P_{N,\varsigma})x_N
   =(I-P_{N,\varsigma})H_{\Rres_N}x_N,\\
 \norm{(I-P_{N,\varsigma})x_N}\le\frac{2\opnorm{\Rres_N}}\varsigma,\\
 |\ip{\mathbf r}{x_N}|
 \le\norm{P_{N,\varsigma}\mathbf r}+\frac{2\opnorm{\Rres_N}}\varsigma,
 \qquad
 \norm{P_{N,\varsigma}\mathbf r}^2
   =\mu_{\mathbf r,N}([z_{\mathrm{edge}}-\varsigma,\infty)).
\end{gathered}
\]
Let $N\to\infty$ and then $\varsigma\downarrow0$, using \eqref{eq:edgemass}. Finally,
\[
\begin{gathered}
 (\ip{u_0}{\hu})^2=2|\ip{\mathbf u_0}{x_N}|^2,\qquad
 (\ip{\han_N}{\hu})^2=2|\ip{\hat{\mathbf a}_N}{x_N}|^2,\\
 (\ip{v_0}{\hv})^2=2|\ip{\mathbf v_0}{x_N}|^2.
\end{gathered}
\]
Thus all three overlaps vanish for every leading unit singular pair of $C$.

\subsection{Completion of the proof of
\texorpdfstring{\thmref{thm:finitesignal}}{the finite-signal theorem}}
\label{app:assembly}

\begin{proof}[Proof of \autoref{thm:finitesignal}]
Below threshold, \autoref{app:subcritical} gives the edge limit and all three vanishing overlaps. Weyl's inequality transfers the singular-value limit, since $|\sigma_1(C)-\sigma_1(\Cbar_N)|\le\opnorm{\Rres_N}=o_P(1)$. Above threshold, \autoref{cor:outlierclosure} gives the outlier, its overlap limits and a fixed positive gap. Weyl's inequality again transfers the singular value, while Wedin's theorem \citep{wedin1972} gives
\[
 \max\{\sin\angle(\hu,\bar u),\sin\angle(\hv,\bar v)\}
 \le\frac{2\opnorm{\Rres_N}}{\text{gap}}=o_P(1).
\]
Thus the overlap limits transfer to $C$ as well. The reduction in \autoref{lem:reduction} removes the conditioning and proves \autoref{thm:finitesignal} for $\rhox$ in a compact subset of $(0,1)$.
\end{proof}

\paragraph{Unmasked design: \texorpdfstring{$\rhox=1$}{rho x = 1}.} Here $G=I_{\Dx}$, $a_N=u_0$, $n_a=1$, $s_N=\theta\sqrt{\rhoy}$, $\nu_N=\alpha_{x,N}^{-1}\delta_1\Rightarrow\alpha_x^{-1}\delta_1=\nu_0$, and \autoref{ass:closure} holds by its $\rhox=1$ clause. Everything above applies with $T_0(w)=w/\alpha_y+w/(\alpha_x(w-1))$, $w_R-1=\sqrt{\alpha_y/\alpha_x}$, $z_{\mathrm{edge}}=\alpha_y^{-1/2}+\alpha_x^{-1/2}$, the closed forms \eqref{eq:rhoxone}--\eqref{eq:rhoxonecurves} of \autoref{prop:limits}(i), and $w_R-w_G=\sqrt{\alpha_y/\alpha_x}$ in place of $\gamma$ in \autoref{app:subcritical}.

\subsection{The Gaussian ensemble}
\label{app:gaussens}

With $Q_{\mathrm G},\WG$ standard Gaussian of sizes $N\times\Dy$ and $\Dx\times\Dy$, independent of $(S_x,S_y,Z)$, and conditionally on $A$, $A^\top Q_{\mathrm G}/\sqrt N\eqd G^{1/2}\WG/\sqrt N$, so the model of \autoref{lem:gaussresp} is
\begin{equation}\label{eq:gaussens}
  C_{\mathrm G}:=s_N\,\han_Nv_0^\top+\frac1{\sqrt N}\,G^{1/2}\WG .
\end{equation}
At $\rhoy=1$ it has exactly the conditional law of $\Cbar_N$ (\autoref{cor:specialise}); for $\rhoy<1$ it matches the conditional covariance. Under hypotheses (a) and (c) of \autoref{thm:outlier}, hypothesis (b) holds for $Q_{\mathrm G}$ by the remark after \autoref{prop:noiseedge}, and the rectangle continuation and the sign-change argument apply to $E_{\mathrm G}$ without change, so $C_{\mathrm G}$ has the same conditional outlier and residue-form overlaps up to $o_P(1)$. Fixed-height comparison alone does not suffice; spectral exclusion and the exterior argument are needed. The finite-$N$ laws for $\rhoy<1$ are not identified; \autoref{app:replica} analyzes this ensemble.

\clearpage
\section{Nested flat orthogonal designs}
\label{app:rmt}

This appendix verifies \autoref{ass:closure} for the nested flat orthogonal class of \autoref{cor:flatdesigns}: $\abs{(\Xs)_{ki}}=1$ and $\Xs^\top\Xs=NI_{\Dx}$, the designs being the $N\times\Dx$ corners of one fixed sign array, with the Sylvester corners at $N=2^m$ as the standard example, up to row and column signed permutations with $u_0$ transformed accordingly, and satisfying \autoref{ass:incoherence}. Flatness gives the comparison at fixed positive height (\autoref{app:gaussdesign}, \autoref{lem:comparison}); nesting within one fixed array enters through the spectral confinement of \autoref{app:confine}; \autoref{app:closureproof} continues the estimates to the real exterior. 

The design is fixed and only the mask $S_x$ is random; the retention $\rhox\in(0,1)$ and the limit ratio $\alpha_x$ are fixed, and confinement is established at each fixed $(\alpha_x,\rhox)$ rather than uniformly in the parameters. At $\rhox=1$, $G=I_{\Dx}$ and there is nothing to prove.

\begin{proposition}[Closure for nested flat orthogonal designs]
\label{prop:flatclosure}
Let $(\Xs)$ be a sequence of flat orthogonal designs satisfying \autoref{ass:incoherence} that is in addition \emph{nested}: after row and column signed permutations, each $\Xs$ is the $N\times\Dx$ corner of one fixed infinite unimodular array with orthogonal corners. The standard example is the Sylvester array \eqref{eq:sylvester} at $N=2^m$. Then the leverage is exactly equal, with $\tau_N=1/\alpha_{x,N}$, and for every compact $K\subset(w_G,\infty)$ there is a fixed open complex neighborhood $U_K\supset K$, with $\overline U_K$ separated from the Gram support, on which
\[
  \sup_{w\in U_K}\big\{\abs{\psi_N(w)-\psi_0(w)}+\abs{\Phi_N(w)-\Phi_0(w)}+\abs{\Phix_N(w)-\Phix_0(w)}\big\}\toP0,
\]
and $\lambda_{\max}(G)\le w_G+o_P(1)$, the functionals being those of \eqref{eq:resolvents} and the limits those of \eqref{eq:closing}. In particular \autoref{ass:closure} holds. At $\rhox=1$ the conclusion is immediate from $G=I_{\Dx}$.
\end{proposition}

\subsection{Setting}
\label{app:designsetting}

We use the design notation of \autoref{app:mde} and the scalar templates \eqref{eq:MblocksE} and \eqref{eq:finiteNblocks}. By \autoref{prop:infoplusnoise}, flatness gives independent centered entries of $\Theta$, with $\Var\Theta_{ki}=(1-\rhox)/N$ and $\abs{\Theta_{ki}}\le(N\rhox)^{-1/2}$.

The trace identity from \eqref{eq:linblocks} is
\[
  \psi_N(w)=\frac1N\big(\xi\Tr\mathbf R_{22}(\xi)-\Dx\big),\qquad w=\xi^2.
\]
It suffices to control $N^{-1}\Tr\mathbf R_{22}$, $\ip{\hat b}{\mathbf R_{11}\hat b}$ and $\ip{u_0}{\mathbf R_{21}\hat b}$ in deterministic directions. The trace and push identities \eqref{eq:pushids} recover the three functionals from these forms, with normalization by $\norm{a_N}^2$ and $\norm{a_N}$ for the directional targets.

\subsection{The Gaussian design calculation}
\label{app:gaussdesign}

\begin{lemma}[Gaussian design model]
\label{lem:gaussdesign}
For the Gaussian benchmark of \autoref{app:mde}, let $0<\rhox<1$ and write $\mathbf R_{\mathrm G}(\xi):=(\xi I-H_{A_{\mathrm G}})^{-1}$.
\begin{enumerate}
\item For every compact $\mathcal K\subset\mathbb C^+$, with $\eta:=\min_{\mathcal K}\operatorname{Im}\xi$: $\operatorname{Im}g_{\mathrm G,N}\le0$, $\operatorname{Im}d_1(g_{\mathrm G,N};\xi)\ge\operatorname{Im}\xi$, $\E\mathbf R_{\mathrm G,22}(\xi)=g_{\mathrm G,N}(\xi)\,I_{\Dx}$ exactly, $\sup_{\mathcal K}\opnorm{\E\mathbf R_{\mathrm G}-\mathcal M_N(g_{\mathrm G,N};\xi)}=O(N^{-1})$, and
\[
  \sup_{\xi\in\mathcal K}\ \sup_{\norm r=\norm s=1}\Var\big(\ip{r}{\mathbf R_{\mathrm G}(\xi)s}\big)\le\frac{4(1-\rhox)}{N\eta^4}.
\]
\item For every compact $\mathcal K\subset\mathbb C^{++}$, $g_{\mathrm G,N}\to g$ uniformly on $\mathcal K$, where $g$ is the physical solution of \eqref{eq:ghsystem} at ratio $\alpha_x$; hence $g_{\mathrm G,N}-g_{0,N}\to0$, $\mathcal M_N(g_{\mathrm G,N};\xi)-M_{0,N}(\xi)\to0$ and $M_{0,N}(\xi)-M_0(\xi)\to0$, uniformly on $\mathcal K$ and in operator norm.
\end{enumerate}
\end{lemma}

\begin{proof}
Write $\Theta_{\mathrm G}=\sqrt{1-\rhox}\,N^{-1/2}Z_{\mathrm G}^\top$, where $Z_{\mathrm G}$ is $\Dx\times N$ with iid standard Gaussian entries. Suppress the argument $\xi$ and write $d_1=d_1(g_{\mathrm G,N};\xi)$.

For orthogonal $V\in\R^{\Dx\times\Dx}$, set $U:=\rhox^{-1}PVP^\top+I_N-\Pi$. The identities $P^\top P=\rhox I_{\Dx}$ and $PP^\top=\rhox\Pi$ give $UU^\top=I_N$ and $UP=PV$. Gaussian bi-invariance therefore gives
\[
\begin{aligned}
  UA_{\mathrm G}V^\top&=P+U\Theta_{\mathrm G}V^\top\eqd A_{\mathrm G},\\
  H_{UA_{\mathrm G}V^\top}
  &=\diag(U,V)H_{A_{\mathrm G}}\diag(U,V)^\top.
\end{aligned}
\]
Hence $\E\mathbf R_{\mathrm G,22}$ commutes with every $V$; taking traces gives $\E\mathbf R_{\mathrm G,22}=g_{\mathrm G,N}I_{\Dx}$. The resolvent sign gives $\operatorname{Im}g_{\mathrm G,N}\le0$ and $\operatorname{Im}d_1\ge\operatorname{Im}\xi$.

Using standard basis vectors, differentiation in entry $(k,i)$ and summation over the Gaussian entries gives, for unit $r,s$,
\[
\begin{aligned}
  \mathcal E_{ki}&:=\begin{pmatrix}0&e_ke_i^\top\\e_ie_k^\top&0\end{pmatrix},
  \qquad e_k\in\R^N,\ e_i\in\R^{\Dx},\\
  \partial_{(Z_{\mathrm G})_{ik}}\ip r{\mathbf R_{\mathrm G}s}
  &=\sqrt{\frac{1-\rhox}{N}}\,
    r^\top\mathbf R_{\mathrm G}\mathcal E_{ki}\mathbf R_{\mathrm G}s,\\
  \sum_{k,i}\abs{r^\top\mathbf R_{\mathrm G}\mathcal E_{ki}\mathbf R_{\mathrm G}s}^2
  &\le4\norm{\mathbf R_{\mathrm G}^\top r}^2\norm{\mathbf R_{\mathrm G}s}^2
  \le\frac4{\eta^4}.
\end{aligned}
\]
Gaussian Poincar\'e gives the variance bound. Set $\hat e:=(1-\rhox)N^{-1}\Tr\mathbf R_{\mathrm G,22}$, so $\E\hat e=(1-\rhox)\alpha_{x,N}^{-1}g_{\mathrm G,N}$ and $d_1=\xi-\E\hat e$. Minkowski's inequality then gives
\[
  \norm{\hat e-\E\hat e}_{L^2}
  \le\frac{(1-\rhox)\Dx}{N}
    \sup_{\norm r=1}\Var\big(\ip r{\mathbf R_{\mathrm G}r}\big)^{1/2}
  =O(N^{-1/2}).
\]

Differentiate using $\partial_{\Theta_{ki}}\mathbf R_{\mathrm G}=\mathbf R_{\mathrm G}\mathcal E_{ki}\mathbf R_{\mathrm G}$. Multiply by the entry variance and sum over $i$; symmetry $\mathbf R_{\mathrm G,12}=\mathbf R_{\mathrm G,21}^\top$ identifies the product term. Repeating for the lower-right block gives the last identity:
\[
\begin{aligned}
  \partial_{\Theta_{ki}}(\mathbf R_{\mathrm G,21})_{i\ell}
  &=(\mathbf R_{\mathrm G,21})_{ik}(\mathbf R_{\mathrm G,21})_{i\ell}
   +(\mathbf R_{\mathrm G,22})_{ii}(\mathbf R_{\mathrm G,11})_{k\ell},\\
  \big(\E[\Theta_{\mathrm G}\mathbf R_{\mathrm G,21}]\big)_{k\ell}
  &=\frac{1-\rhox}{N}\sum_i
      \E\big[\partial_{\Theta_{ki}}(\mathbf R_{\mathrm G,21})_{i\ell}\big],\\
  \E[\Theta_{\mathrm G}\mathbf R_{\mathrm G,21}]
  &=\E[\hat e\,\mathbf R_{\mathrm G,11}]+\frac{1-\rhox}N\E[\mathbf R_{\mathrm G,12}\mathbf R_{\mathrm G,21}],\\
  \E[\Theta_{\mathrm G}\mathbf R_{\mathrm G,22}]
  &=\E[\hat e\,\mathbf R_{\mathrm G,12}]+\frac{1-\rhox}N\E[\mathbf R_{\mathrm G,12}\mathbf R_{\mathrm G,22}].
\end{aligned}
\]

Centering both factors and applying Cauchy--Schwarz, for real unit vectors of the appropriate block dimensions,
\[
\begin{aligned}
  \big|u^\top\E[(\hat e-\E\hat e)\mathbf R_{\mathrm G,ab}]v\big|
  &\le\norm{\hat e-\E\hat e}_{L^2}
    \Var(u^\top\mathbf R_{\mathrm G,ab}v)^{1/2}=O(N^{-1}).
\end{aligned}
\]
Real unit tests control the complex operator norm up to a factor $2$. The product remainders in these identities are also $O(N^{-1})$, since $\opnorm{\mathbf R_{\mathrm G}}\le\eta^{-1}$. Insert these estimates into the $(1,1)$ and $(1,2)$ blocks of $(\xi I-H_{A_{\mathrm G}})\mathbf R_{\mathrm G}=I$. Since $\abs{d_1}\ge\eta$, divide the second relation by $d_1$, transpose, and substitute into the first, using $PP^\top=\rhox\Pi$:
\[
\begin{aligned}
  d_1\,\E\mathbf R_{\mathrm G,11}-P\,\E\mathbf R_{\mathrm G,21}&=I+O(N^{-1}),\\
  d_1\,\E\mathbf R_{\mathrm G,12}-g_{\mathrm G,N}P&=O(N^{-1}),\\
  \E\mathbf R_{\mathrm G,12}
  &=\frac{g_{\mathrm G,N}}{d_1}P+O(N^{-1}),\\
  \E\mathbf R_{\mathrm G,11}
  &=\frac1{d_1}I_N+\frac{\rhox g_{\mathrm G,N}}{d_1^2}\Pi+O(N^{-1}).
\end{aligned}
\]
Together with the exact lower-right block, these are the blocks of $\mathcal M_N(g_{\mathrm G,N};\xi)$, uniformly on $\mathcal K$, proving (1).

The Gaussian Gram matrix is
\[
  G_{\mathrm G}:=A_{\mathrm G}^\top A_{\mathrm G}
  =\frac1N(\sqrt{\rhox}\,\Xs^\top+\sqrt{1-\rhox}\,Z_{\mathrm G})
    (\sqrt{\rhox}\,\Xs^\top+\sqrt{1-\rhox}\,Z_{\mathrm G})^\top.
\]
Its deterministic population is $\rhox\Xs^\top\Xs/N=\rhox I_{\Dx}$, its noise variance is $1-\rhox$, and $\Dx/N\to\alpha_x^{-1}$. Thus \citet[Theorem 1.1]{dozier2007} gives almost-sure convergence of its empirical Stieltjes transform $s_{G_{\mathrm G}}(\zeta)$ to $s_\rho(\zeta)$ for $\zeta\in\mathbb C^+$. Since $\abs{s_{G_{\mathrm G}}(\zeta)}\le(\operatorname{Im}\zeta)^{-1}$, bounded convergence gives
\[
\begin{aligned}
  g_{\mathrm G,N}(\xi)
  &=\frac\xi{\Dx}\E\Tr(\xi^2I-G_{\mathrm G})^{-1}
   =-\xi\,\E s_{G_{\mathrm G}}(\xi^2)\\
  &\longrightarrow-\xi s_\rho(\xi^2)=g(\xi),
  \qquad \xi\in\mathbb C^{++}.
\end{aligned}
\]
The last equality uses \autoref{app:onevar}: the Dozier--Silverstein equation is \eqref{eq:ghsystem}, and its Stieltjes solution selects the physical branch at ratio $\alpha_x$ (branch and support: \citet[Section~3]{dozier2007analysis}, \autoref{lem:branch}). The bounds $\abs{g_{\mathrm G,N}}\le\eta^{-1}$ and $\abs{g_{\mathrm G,N}'}\le\eta^{-2}$ upgrade the convergence to uniform convergence on $\mathcal K\subset\mathbb C^{++}$.

The exact finite-aspect solution $g_{0,N}$ also tends uniformly to $g$, by \eqref{eq:baselinecompare} and the positive-height continuity in \autoref{app:baseline}. The triangle inequality therefore gives $\sup_{\mathcal K}\abs{g_{\mathrm G,N}-g_{0,N}}\to0$.
On $\{\operatorname{Im}\xi\ge\eta\}$, the bounds $\operatorname{Im}g\le0$, $\abs g\le\eta^{-1}$ imply $\abs{d_1(g;\xi)}\ge\eta$. Hence the template is uniformly Lipschitz in $g$ and the aspect ratio on these sets. It follows that
\[
\begin{aligned}
  \sup_{\mathcal K}\opnorm{\mathcal M_N(g_{\mathrm G,N};\xi)-M_{0,N}(\xi)}&\longrightarrow0,\\
  \sup_{\mathcal K}\opnorm{M_{0,N}(\xi)-M_0(\xi)}&\longrightarrow0.
\end{aligned}
\]
These substitutions give $o(1)$ errors; the $O(N^{-1})$ estimate in (1) compares the expected Gaussian resolvent with its template before either substitution. This proves (2).
\end{proof}

\subsection{Spectral confinement}
\label{app:confine}

\begin{proposition}[Confinement]
\label{prop:confinement}
For the nested flat orthogonal class of \autoref{prop:flatclosure} and fixed $\rhox\in(0,1)$, there is $K_0>w_G$ with $\Prob(\opnorm G\le K_0)\to1$, and $\lambda_{\max}(G)\le w_G+o_P(1)$.
\end{proposition}

\begin{proof}
Set $(\mathcal Z_N)_{ik}:=(\Xs)_{ki}((S_x)_{ki}-\rhox)/\sqrt{\rhox(1-\rhox)}$. These entries are independent, centered, of unit variance, and bounded by $\max(\rhox,1-\rhox)/\sqrt{\rhox(1-\rhox)}$. The entrywise identity gives the matrix form:
\[
\begin{gathered}
  \sqrt{\rhox}\,(\Xs)_{ki}+\sqrt{1-\rhox}\,(\mathcal Z_N)_{ik}
  =\frac{(\Xs)_{ki}(S_x)_{ki}}{\sqrt{\rhox}},\\
  G=\frac1N\big(\sqrt{\rhox}\,\Xs^\top+\sqrt{1-\rhox}\,\mathcal Z_N\big)
  \big(\sqrt{\rhox}\,\Xs^\top+\sqrt{1-\rhox}\,\mathcal Z_N\big)^\top.
\end{gathered}
\]

For fixed unit $u\in\R^{\Dx}$ and $v\in\R^N$, Hoeffding's inequality gives $\Prob(\abs{\ip u{\mathcal Z_Nv}}>t)\le2e^{-ct^2}$. Two $1/4$-nets have at most $9^{\Dx}$ and $9^N$ points, and their largest bilinear form controls $\opnorm{\mathcal Z_N}$ up to a factor $2$. A union bound therefore gives, for sufficiently large fixed $C_Z$,
\[
  \Prob(\opnorm{\mathcal Z_N}>C_Z\sqrt N)
  \le2\,9^{N+\Dx}e^{-cC_Z^2N/4}\le e^{-c'N}.
\]
Thus $\opnorm G^{1/2}\le\sqrt{\rhox}+\sqrt{1-\rhox}\,\opnorm{\mathcal Z_N}/\sqrt N\le\sqrt{\rhox}+\sqrt{1-\rhox}\,C_Z$ on the complementary event. Take $K_0:=(\sqrt{\rhox}+\sqrt{1-\rhox}\,C_Z)^2$, enlarged to $K_0\ge w_G+1$.

Apply \citet[Theorem 1.1]{baisilverstein2012} with row dimension $\Dx$, column dimension $N$, deterministic part $\sqrt{\rhox}\,\Xs^\top$ and noise $\sqrt{1-\rhox}\,\mathcal Z_N$. The aspect ratio is $\Dx/N=\alpha_{x,N}^{-1}\to\alpha_x^{-1}>0$. For hypothesis (a), use the fixed unimodular array $\mathcal A$ of \autoref{prop:flatclosure} and one infinite array of independent mask variables to define
\[
  \mathcal Z_{ik}:=\mathcal A[k,i]\,
  \frac{(S_x)_{ki}-\rhox}{\sqrt{\rhox(1-\rhox)}}.
\]
Then $\mathcal Z_N$ is a corner of the same array for every selected dimension. Its entries need not be identically distributed. This coupling preserves the law at each $N$, hence every conclusion in probability. The standard example is the Sylvester array
\begin{equation}\label{eq:sylvester}
  \mathcal A[k,i]=(-1)^{\operatorname{popcount}(k\,\&\,i)},
  \qquad k,i\ge0,
\end{equation}
whose top-left $2^m\times2^m$ corner is the Sylvester matrix of order $2^m$; at $N=2^m$ the design consists of its first $\Dx$ columns.

Bounded entries give hypotheses (b)--(c). For (d)--(f), whitening \eqref{eq:whitening} gives
\[
  \frac{\rhox}{N}\Xs^\top\Xs=\rhox I_{\Dx},
  \qquad \opnorm{\sqrt{\rhox/N}\,\Xs^\top}=\sqrt{\rhox}.
\]
Thus the population law is exactly $\delta_{\rhox}$; the norm bound, aspect-ratio limit and matrix form are all as required.

For hypothesis (g), fix $\varepsilon\in(0,1]$ and choose
\[
\begin{aligned}
  [a_1,a_2]&:=[w_G+\varepsilon,K_0],\\
  (b_1,b_2)&:=(w_G+\varepsilon/2,K_0+1)
             \supset[a_1,a_2].
\end{aligned}
\]
The limiting law has population $\delta_{\rhox}$, noise variance $1-\rhox$ and ratio $\alpha_x^{-1}$ (\autoref{app:onevar}); its support top is $w_G$ (\autoref{lem:branch}). For the finite-$N$ laws, \autoref{prop:ipnsupport} gives $w_{G,N}=W_N(\omega_{G,N})\to w_G$, hence $w_{G,N}<b_1$ eventually. Thus $(b_1,b_2)$ avoids both supports as required by (1.7).

Every $x>w_{G,N}$ has the unique exterior representation $x=W_N(\omega)$ with $0<\omega<\omega_{G,N}$, and $x_N'(\mathfrak b)>0$ at $\mathfrak b=(1+\omega)^{-1}$ by \eqref{eq:supportorient}. The converse stated after (1.6) in \citet[Section~1]{baisilverstein2012}, at population law $\delta_{\rhox}$, noise level $1-\rhox$ and ratio $\alpha_{x,N}^{-1}$, then identifies $\mathfrak b$ with their $b_n^0(x)$, so the population coordinate, their $w_n^0$ from (1.3), is
\[
  q_N(x):=\mathfrak b^2x-\mathfrak b\,(1-\rhox)(1-\alpha_{x,N}^{-1}).
\]
Multiplying the first display of \autoref{prop:ipnsupport} by $\mathfrak b^2$ shows that $q_N(x)$ is the value of $m_{\delta_{\rhox}}^{-1}$ computed in its proof, so
\[
\begin{aligned}
  q_N\big(W_N(\omega)\big)
  &=\rhox+\frac{1-\rhox}{\alpha_{x,N}\omega}\\
  &\ge\rhox+\frac{1-\rhox}{\alpha_{x,N}\omega_{G,N}}
   \ge\rhox+\frac{1-\rhox}{2\alpha_x\omega_G}
   =:\rhox+\delta
\end{aligned}
\]
for all large $N$, because $\alpha_{x,N}\omega_{G,N}\to\alpha_x\omega_G$. Take $\hat\varepsilon:=\varepsilon/4$, shrinking it if needed below the margin in the source's (1.8). Then
\[
  (a_1-\hat\varepsilon,a_2+\hat\varepsilon)
  =(w_G+\varepsilon-\hat\varepsilon,K_0+\hat\varepsilon)
  \subset(b_1,b_2),
\]
so its image under $q_N$ lies in $[\rhox+\delta,\infty)$.

Deleting column $j$ from $\sqrt{\rhox}\,\Xs^\top=\sqrt N\,P^\top$ gives the population matrix
\[
  \frac{N}{N-1}P^\top(I_N-e_je_j^\top)P
  \preceq\frac{N}{N-1}P^\top P
  =\frac{\rhox N}{N-1}I_{\Dx}.
\]
Since $\rhox N/(N-1)<\rhox+\delta$ for all large $N$, the spectral bound is uniform in $j$. No deleted-column population matrix has an eigenvalue in the required image; condition (1.10) holds for all $N$ columns.

The theorem applies along the nested dimension sequence: its proof uses the row dimension through summable moment bounds and almost-sure ESD convergence, which remain valid on this subsequence. On the chosen coupling it excludes eigenvalues in $[w_G+\varepsilon,K_0]$ eventually almost surely. Thus, in probability for the original masks,
\[
\begin{aligned}
  \Prob\big(\lambda_{\max}(G)>w_G+\varepsilon\big)
  &\le\Prob(\opnorm G>K_0)\\
  &\quad+\Prob\big(\spec G\cap[w_G+\varepsilon,K_0]\ne\varnothing\big)
  \longrightarrow0.
\end{aligned}
\]
Since $\varepsilon$ is arbitrary, $\lambda_{\max}(G)\le w_G+o_P(1)$.
\end{proof}

\subsection{Verification of the deterministic equivalent}
\label{app:closureproof}

\begin{proof}[Proof of \autoref{prop:flatclosure}]
Clause \textbf{(i)} of \autoref{ass:closure} is \autoref{prop:confinement}. For clause \textbf{(ii)}, fix a compact $K\subset(w_G,\infty)$, choose $K'$ with $K\subset(w_G,K']$, and put
\[
  \mathcal S_K:=\big\{\xi:\ \sqrt{w_G}+\varepsilon\le\operatorname{Re}\xi\le\sqrt{K'}+1,
           \ \abs{\operatorname{Im}\xi}\le\eta_1\big\}.
\]
Choose $\varepsilon,\eta_1>0$ small enough that $\{\xi^2:\xi\in\mathcal S_K\}$ contains an open $U_K\supset K$ whose closure is separated from $[0,w_G]$.

Apply \autoref{lem:comparison} with the following substitutions. The entries of $X$ are independent, centered and bounded; those of $Y$ are iid Gaussian, with matching variances:
\[
\begin{gathered}
  d=n=N,\quad m=\Dx,\quad P_0=P,\quad B=I_N,\quad
  X=\sqrt N\,\Theta,\quad Y=\sqrt N\,\Theta_{\mathrm G},\\
  Z_X=A,\quad Z_Y=A_{\mathrm G},\quad
  X_{ki}=\frac{(\Xs)_{ki}\big((S_x)_{ki}-\rhox\big)}{\sqrt{\rhox}},\quad
  \E X_{ki}^2=\E Y_{ki}^2=1-\rhox.
\end{gathered}
\]
The Gaussian calculation of \autoref{lem:gaussdesign} supplies the limiting forms. Since $M_0$ is complex symmetric and bounded at fixed height, \autoref{cor:comparisontrace} applies with the test matrix $\diag(0_N,I_{\Dx})$. For every compact $\mathcal K\subset\mathbb C^{++}$ we obtain
\[
\begin{gathered}
  \sup_{\xi\in\mathcal K}\ \sup_{\norm r=\norm s=1}
  \E\abs{\ip r{\mathbf R(\xi)s}-\ip r{M_0(\xi)s}}^2\longrightarrow0,\\
  \sup_{\xi\in\mathcal K}\norm{N^{-1}\Tr\big(\mathbf R_{22}(\xi)-M_{0,22}(\xi)\big)}_{L^2}
  \longrightarrow0.
\end{gathered}
\]
The test vectors are deterministic, and their supremum remains outside the expectation.

By \autoref{prop:confinement}, the event $\{\lambda_{\max}(G)\le w_G+\varepsilon^2/4\}$ has probability tending to $1$. On it,
\[
  \spec(H_A)\subset[-\sqrt{w_G}-\varepsilon/2,\sqrt{w_G}+\varepsilon/2],
  \qquad \sup_{\xi\in\mathcal S_K}\opnorm{\mathbf R(\xi)}\le\frac2\varepsilon.
\]
The deterministic matrix $M_0$ is analytic and bounded on a sufficiently thin $\mathcal S_K$. Indeed, the integral formula in \autoref{app:onevar} extends $g$ to $\{\xi^2\notin[0,w_G]\}$, and $d_1=\xi/(1+\omega)>0$ on the real exterior branch. Compactness and continuity keep $d_1$ away from zero throughout the rectangle, so the blocks of $M_0$ and their derivatives are bounded there.

For $0<\eta\le\eta_1$ and $\xi=x+iy$ with $0\le y\le\eta$, compare with $x+i\eta$. The resolvent identity gives
\[
\begin{aligned}
  \opnorm{\mathbf R(\xi)-\mathbf R(x+i\eta)}
  &\le\abs{\eta-y}\,\opnorm{\mathbf R(\xi)}\opnorm{\mathbf R(x+i\eta)}
   \le\frac{4\eta}{\varepsilon^2},\\
  \opnorm{M_0(\xi)-M_0(x+i\eta)}&\le C\eta.
\end{aligned}
\]
Therefore each unit bilinear form satisfies
\[
\begin{aligned}
  \abs{\ip r{(\mathbf R(\xi)-M_0(\xi))s}}
  &\le\abs{\ip r{(\mathbf R(x+i\eta)-M_0(x+i\eta))s}}\\
  &\quad+\frac{4\eta}{\varepsilon^2}+C\eta.
\end{aligned}
\]
The normalized trace obeys the same bound up to the bounded factor $\Dx/N$; negative $y$ follows by conjugation.

For fixed $\eta$, $\mathcal S_K\cap\{\operatorname{Im}\xi\ge\eta\}$ is a compact subset of $\mathbb C^{++}$. A finite net, the positive-height comparison and the resolvent Lipschitz bound $\abs{\xi-\xi'}\eta^{-2}$ give uniform convergence in probability there for $r,s\in\{(\hat b;0),(0;u_0)\}$ and for $N^{-1}\Tr\mathbf R_{22}$. First let $N\to\infty$ at fixed $\eta$; then let $\eta\downarrow0$ in the displayed bound. The three forms therefore converge to their $M_0$ counterparts uniformly on all of $\mathcal S_K$ in probability.

Use $\Pi\hat b=\hat b$, $\norm{\hat b}=1$ and $P^\top\hat b=\sqrt{\rhox}\,u_0$ in \eqref{eq:MblocksE}. With $d_1=\xi-(1-\rhox)\alpha_x^{-1}g$, the deterministic values are
\[
\begin{aligned}
  \ip{\hat b}{M_{0,11}\hat b}
  &=d_1^{-1}+\rhox g d_1^{-2}=h,\\
  \ip{u_0}{M_{0,21}\hat b}
  &=d_1^{-1}g\,\ip{u_0}{P^\top\hat b}
   =\frac{\sqrt{\rhox}\,g}{d_1},\\
  \frac1N\Tr M_{0,22}&=\frac{\Dx}{N}g=\frac{g}{\alpha_{x,N}}.
\end{aligned}
\]
Flatness and $\norm b^2=N$ also give the exact row-norm parameter in \eqref{eq:taudef}:
\[
  \tau_N=\frac1{N^2}\sum_k b_k^2\norm{(\Xs)_{k\cdot}}^2
  =\frac{\Dx}{N^2}\norm b^2=\frac1{\alpha_{x,N}}.
\]
Thus \autoref{prop:geometry} gives $\norm{a_N}^2\toP n_a>0$, and $a_N\ne0$ with probability tending to $1$. Substitute the block limits into the trace identity of \autoref{app:designsetting} and the push identities \eqref{eq:pushids}:
\[
\begin{aligned}
  \psi_N(w)
  &=\frac{\xi}{N}\Tr\mathbf R_{22}-\frac1{\alpha_{x,N}}
    \ \longrightarrow\ \frac{\xi g-1}{\alpha_x},\\
  \Phi_N(w)
  &=\frac{\xi\ip{\hat b}{\mathbf R_{11}\hat b}-1}{\norm{a_N}^2}
    \ \longrightarrow\ \frac{\xi h-1}{n_a},\\
  \Phix_N(w)
  &=\frac{\ip{u_0}{\mathbf R_{21}\hat b}}{\norm{a_N}}
    \ \longrightarrow\ \frac{\sqrt{\rhox}\,g}{d_1\sqrt{n_a}}.
\end{aligned}
\]
Here $w=\xi^2$, and the convergence is uniform on $U_K$ in probability. These are exactly \eqref{eq:closing}: $e_1=(1-\rhox)\alpha_x^{-1}g$, and $h=d_1^{-1}+\rhox g d_1^{-2}$ is the second continued fraction of \eqref{eq:ghsystem} on its solution. This proves clause \textbf{(ii)} and the row-norm assertion.
\end{proof}

\begin{remark}[Signed permutations]
The randomized designs of \autoref{app:flatvalidation} transfer from the corners of the fixed array by distributional equivariance. For a row signed permutation $U$, mask exchangeability and $b\mapsto Ub$ give $A\mapsto UA$ in law together with $\hat b\mapsto U\hat b$, which preserves the joint law of $(G,a_N)$ and hence of the functionals; for a column signed permutation $V$, $u_0\mapsto V^\top u_0$, $G\mapsto V^\top GV$ and $a_N\mapsto V^\top a_N$, so the functionals are invariant when the planted direction is transformed with the design. This includes $N$-dependent row and column randomizations.

Hence signed-permuted designs are reduced to the nested corners before the fixed-array theorem is invoked in \autoref{prop:confinement}, and confinement is inherited at the same retention. \autoref{ass:incoherence} must hold for the pair $(\Xs,u_0)$; selecting columns of a larger flat orthogonal matrix preserves flat orthogonality but not automatically the score bound \autoref{ass:incoherence}(iv).
\end{remark}

\clearpage
\section{The replica derivation}
\label{app:replica}

This appendix studies the covariance-matched Gaussian ensemble \eqref{eq:gaussens}, conditionally on $(\Xs,S_x)$. At $\rhoy=1$, this ensemble has the conditional law of $\Cbar_N$. For $\rhoy<1$, it is only a formal covariance-matched replacement for the actual response-masked noise. Its positive-integer Gaussian average is exact. The continuation to $n\to0$, replica symmetry, limit exchanges, saddle selection, and source differentiation remain formal under assumptions \textbf{(R1)}--\textbf{(R5)} below. Subject to those assumptions, the saddle equations are algebraically equivalent to the finite-$N$ conditional deterministic-equivalent secular and overlap equations obtained by evaluating the exact identities of \autoref{app:secular} on the conditional transforms. Evaluation at the scalar limits from Appendices~\ref{app:mde}--\ref{app:baseline} then gives the candidate formulas of \autoref{prop:curves}. This calculation neither proves \autoref{ass:closure} nor contributes to the proof of \autoref{thm:finitesignal}.

We work on the mask event $a_N\ne0$, so $\han_N$ is defined and $G\ne0$. Conditionally on the design and $S_x$, the functionals of \eqref{eq:resolvents} are deterministic. We keep the subscript $N$ on them, write the transforms of \eqref{eq:transforms} without it, and retain $\alpha_{x,N}=N/\Dx$ and $\alpha_{y,N}=N/\Dy$ in every finite-$N$ identity. Only after the deterministic limit do these become $\alpha_x$ and $\alpha_y$. The one-variable forms $\psi(\omega)$ and $\Phi(\omega)$ of \autoref{def:maskavg} do not appear here. The response overlap is denoted by $r_y$. The spherical multiplier $\zeta$ and susceptibility $\chi$ are independent variational quantities; \autoref{prop:rsequiv} relates them to the spectral coordinates after imposing stationarity.

\paragraph{Replica assumptions.} The calculation below is conditional on the design and $S_x$ and assumes: \textbf{(R1)} analytic continuation of the integer-replica expression to $n\to0$.

\textbf{(R2)} replica symmetry in the right-sector order parameters and the spherical multipliers $\zeta_a$.

\textbf{(R3)} self-averaging of the quenched free energy.

\textbf{(R4)} the scaling $1-q_v=\chi/\beta$ and the stated zero-temperature/order-of-limits exchange.

\textbf{(R5)} an isolated differentiable saddle when source fields are differentiated.

These assumptions are not used by the exact rank-one identities (\autoref{app:secular}) or the finite-signal proof of \autoref{app:finitesignal}.

For a matrix $M\in\R^{\Dx\times\Dy}$ and source fields $h=(h_x,h_p,h_y)$, define the one-copy partition function
\begin{equation}\label{eq:gibbsmain}
\begin{aligned}
  Z_{\beta,N}(M;h):=\int_{\norm u=\norm v=1}
  \exp\Big(\beta N\big[&\ip{u}{Mv}+h_x\ip{u}{u_0}\\
    &+h_p\ip{u}{\han_N}+h_y\ip{v}{v_0}\big]\Big)
    \,\dd\sigma(u)\,\dd\sigma(v).
\end{aligned}
\end{equation}
Here $\dd\sigma$ is normalized uniform measure on the indicated unit sphere. The fields couple to the planted design direction, the corrupted direction, and the planted response direction, respectively. Let $F_{\beta,N}(M;h):=(\beta N)^{-1}\log Z_{\beta,N}(M;h)$. The residual bound of \autoref{lem:residual} also controls these free energies. The source terms are identical for $C$ and $\Cbar_N$, and $\abs{\ip{u}{Cv}-\ip{u}{\Cbar_Nv}}\le\opnorm{\Rres_N}$ on the unit spheres, so
\[
  e^{-\beta N\opnorm{\Rres_N}}\,Z_{\beta,N}(\Cbar_N;h)
  \;\le\;Z_{\beta,N}(C;h)
  \;\le\;e^{\beta N\opnorm{\Rres_N}}\,Z_{\beta,N}(\Cbar_N;h),
\]
hence, exactly and uniformly in the source fields, $\abs{F_{\beta,N}(C;h)-F_{\beta,N}(\Cbar_N;h)}\le\opnorm{\Rres_N}$, and the expected free energies differ by $o(1)$ after averaging over the masks and noise at fixed $\theta$. This bound concerns free-energy values; identifying their limiting source derivatives still requires the formal assumptions. At $\rhoy=1$, the conditional law of $\Cbar_N$ is exactly that of $C_{\mathrm G}$ in \eqref{eq:gaussens}. For $\rhoy<1$, using $C_{\mathrm G}$ is the formal Gaussian replacement described above. All Gaussian averages below evaluate \eqref{eq:gibbsmain} at $M=C_{\mathrm G}$.

\subsection{Replica-symmetric calculation}

The derivation runs from the replicated partition function to the free energy \eqref{eq:freeenergymain} and the stationarity conditions \eqref{eq:S1main}--\eqref{eq:S3main}, conditionally on the mask, so that in \eqref{eq:gaussens} the quadruple $(G,\han_N,u_0,s_N)$ is deterministic and only $\WG$ is random. For a positive integer $n$, the power $Z_{\beta,N}(C_{\mathrm G};h)^n$ introduces $n$ unit-vector pairs $(u^a,v^a)$, indexed by $a=1,\ldots,n$.

\begin{lemma}[Source observables within the RS formalism]
\label{lem:sources}
Within the replica-symmetric formalism, let $f(h):=\lim_{\beta\to\infty}\lim_{N\to\infty} \E_{\WG}F_{\beta,N}(C_{\mathrm G};h)$, with the design and mask conditioned on. Define the limiting source observables by
\begin{equation}\label{eq:sources}
  r_x^{\rm RS}:=\lim_{h\to0}\frac{\partial f}{\partial h_x},
  \qquad
  r_a^{\rm RS}:=\lim_{h\to0}\frac{\partial f}{\partial h_p},
  \qquad
  r_y^{\rm RS}:=\lim_{h\to0}\frac{\partial f}{\partial h_y},
\end{equation}
where $h_y\downarrow0$ and $h_x,h_p=o(h_y)$. At fixed $N$, taking the zero-temperature limit before sending the sources to zero selects the top singular pair of $C_{\mathrm G}$ with positive response overlap. The derivatives of the expected free energy give its conditional expected overlaps. Identifying these finite-$N$ derivatives with the derivatives of the limiting free energy is the formal step covered by \textbf{(R1)}--\textbf{(R5)}.
\end{lemma}

\begin{proof}
At fixed $N$ and $\beta$ the integral runs over the compact product of two spheres with an integrand smooth in the sources, so differentiation passes under the integral sign and returns the Gibbs average of the corresponding overlap, $\partial_{h_x}\E_{\WG}F_{\beta,N}(C_{\mathrm G};h) =\E_{\WG}\avg{\ip{u}{u_0}}_\beta$, and likewise for $h_p$ and $h_y$, exactly at every $N$, $\beta$ and $h$. As $\beta\to\infty$, Laplace asymptotics make $F_{\beta,N}(C_{\mathrm G};h)$ converge to the maximum of $\ip{u}{C_{\mathrm G}v}+h_x\ip{u}{u_0}+h_p\ip{u}{\han_N}+h_y\ip{v}{v_0}$ over the two spheres, with the Gibbs measure concentrating on the maximizing set. At $h=0$ the maximum is $\sigma_1(C_{\mathrm G})$, attained by its top singular pair and the simultaneous sign reversal. This singular value is almost surely simple: the Gaussian noise is absolutely continuous on $\range(G)\otimes\R^{\Dy}$, which carries $\han_N$ since $a_N\neq0$ and $a_N\in\range(A^\top)=\range(G)$. But the $h=0$ integrand is invariant under $(u,v)\mapsto(-u,-v)$ while each of the three overlaps is odd, so every Gibbs average vanishes by exact cancellation, and a source must break the tie. The degeneracy is one global sign, so the three overlaps flip together: $h_y>0$ selects the representative with positive response overlap, whereas $h_x$ alone would select a positive design overlap, which need not have the same sign as the response overlap. Hence the lemma takes the path $h_x,h_p=o(h_y)$.

At finite $N$, the formal saddle potential $\Psi_h$ in \eqref{eq:freeenergyh} depends on the sources explicitly, through the coefficients $c_i$ of \eqref{eq:Mimain} and the term $h_yr_y$, and implicitly, through its stationary point $(r_y,\zeta,\chi)(h)$. The envelope theorem removes the latter contribution because the derivatives with respect to these three variables vanish there. The remaining partial derivatives give finite-$N$ saddle predictions. Their limits give \eqref{eq:sources} under the stated formal assumptions.
\end{proof}

The replica overlaps needed for the Gaussian average are
\begin{equation}\label{eq:orderparams}
  (Q_u^G)_{ab}:=\ip{u^a}{Gu^b},
  \qquad (Q_v)_{ab}:=\ip{v^a}{v^b},
  \qquad (r_y)_a:=\ip{v^a}{v_0}.
\end{equation}
The average over $\WG$ of the noise factor in $Z_{\beta,N}(C_{\mathrm G};h)^n$ is exact for every positive integer $n$:
\begin{equation}
  \E_{\WG}\exp\Big(\beta\sqrt N\sum_{a=1}^n\ip{u^a}{G^{1/2}\WG v^a}\Big)
  = \exp\Big(\frac{\beta^2N}{2}\Tr\big(Q_u^GQ_v\big)\Big).
  \label{eq:disordermain}
\end{equation}
The covariance $G$ therefore enters the coupling between left replicas.

\begin{proof}[Proof of \eqref{eq:disordermain}]
At every positive integer $n$, this is the Gaussian moment-generating identity for the matrix $\sum_aG^{1/2}u^a(v^a)^\top$, whose squared Frobenius norm is $\sum_{a,b}\ip{u^a}{Gu^b}\,\ip{v^a}{v^b}=\Tr(Q_u^GQ_v)$.
\end{proof}

We impose the right-sector overlaps in \eqref{eq:orderparams} by delta functions and assume replica symmetry: $Q_v=(1-q_v)I_n+q_v\one\one^\top$ and $(r_y)_a=r_y$ for every $a$. Thus $r_y$ without a replica index denotes their common scalar value. Every other factor of the replicated integrand, the disorder average \eqref{eq:disordermain}, the signal term and the sources, depends on the right replicas only through $Q_v$ and the response overlaps $(r_y)_a$, so the $v$-integral at fixed order parameters is a constrained volume.

For the right-sector entropy, split $v^a=(r_y)_av_0+\tilde v^a$ with $\tilde v^a\perp v_0$. The residual Gram matrix has entries $(\tilde Q_v)_{ab}=(Q_v)_{ab}-(r_y)_a(r_y)_b$. To leading order in $N$ at fixed integer $n$, the Gram--Wishart Jacobian contributes $(2\alpha_{y,N})^{-1}\log\det\tilde Q_v$ per unit of $N$ \citep{mezard1987spin}. Under RS its eigenvalues are $1-q_v$ on $\one^\perp$ and $1-q_v+n(q_v-r_y^2)$ along $\one$. Thus, on $r_y^2\le q_v<1$,
\begin{equation}
  \det\tilde Q_v=(1-q_v)^{n-1}\big(1-q_v+n(q_v-r_y^2)\big),
  \qquad
  \lim_{n\to0}\frac{\log\det\tilde Q_v}{n}
  =\phi(q_v,r_y),                                      \label{eq:ventropyapp}
\end{equation}
where
\[
  \phi(q,r):=\log(1-q)+\frac{q-r^2}{1-q}.
\]
The last identity is the first-order expansion of the displayed determinant. Its continuation in $n$ is part of (R1).

To integrate the left coordinates, represent each constraint $\norm{u^a}^2=1$ by a delta function and a conjugate multiplier $\zeta_a$. The inner coordinate integral then uses Lebesgue measure; the multiplier contours and normalized-sphere factors remain outside it. Under replica symmetry $\zeta_a=\zeta$, the signal, sources, and Gaussian average \eqref{eq:disordermain} give the left action
\begin{equation}\label{eq:uaction}
\begin{aligned}
  \mathcal A_u
  \;=\;&\ \beta Ns_N\,r_y\sum_a\ip{u^a}{\han_N}
  \;+\;\frac{\beta^2N}{2}\sum_{a,b}(Q_v)_{ab}\,\ip{u^a}{Gu^b}\\
  &+\beta N\sum_a\big(h_x\ip{u^a}{u_0}+h_p\ip{u^a}{\han_N}\big)
  \;-\;\frac{\beta N\zeta}{2}\sum_a\big(\norm{u^a}^2-1\big).
\end{aligned}
\end{equation}
Let $G=\sum_i\lambda_i\varphi_i\varphi_i^\top$ be an orthonormal eigendecomposition, with $\han_{N,i}:=\ip{\varphi_i}{\han_N}$ and $u_{0,i}:=\ip{\varphi_i}{u_0}$. Every bilinear form in \eqref{eq:uaction} is a sum over these eigencoordinates. Let $u_i\in\R^n$ collect the replica coordinates along $\varphi_i$. Under RS, the diagonal part of the coupling contributes $\beta^2\lambda_i(1-q_v)=\beta\chi_\beta\lambda_i$, where $\chi_\beta:=\beta(1-q_v)$. It combines with the spherical multiplier to give the curvature $\zeta-\chi_\beta\lambda_i$. The common overlap $q_v$ produces a rank-one coupling. Thus
\begin{equation}\label{eq:uactiondiag}
  \mathcal A_u
  \;=\;\frac{\beta N\zeta n}{2}
  +\sum_i\Big({-\tfrac12}\,\ip{u_i}{M_iu_i}+\beta Nc_i\,\ip{\one}{u_i}\Big),
\end{equation}
where
\begin{equation}
  M_i = \beta N\big[(\zeta-\chi_\beta\lambda_i)I_n-\beta q_v\lambda_i\one\one^\top\big],
  \qquad
  c_i := s_N\,r_y\han_{N,i}+h_xu_{0,i}+h_p\han_{N,i} ,
  \label{eq:Mimain}
\end{equation}
the three linear terms of \eqref{eq:uaction} collecting into the single source $c_i$. Eigencoordinate $i$ enters only through the triple $(\lambda_i,\han_{N,i},u_{0,i})$.

The matrix $M_i/(\beta N)$ has eigenvalue $\zeta-\chi_\beta\lambda_i$ on $\one^\perp$ and $\zeta-\chi_\beta\lambda_i-n\beta q_v\lambda_i$ along $\one$. Hence, for every positive integer $n$,
\[
\begin{aligned}
  \det\frac{M_i}{\beta N}
  &\;=\;(\zeta-\chi_\beta\lambda_i)^{\,n-1}
    \big(\zeta-\chi_\beta\lambda_i-n\beta q_v\lambda_i\big),\\
  \ip{\one}{M_i^{-1}\one}
  &\;=\;\frac{n}{\beta N\big(\zeta-\chi_\beta\lambda_i-n\beta q_v\lambda_i\big)},
\end{aligned}
\]
valid for the Gaussian integral when $\zeta-\chi_\beta\lambda_i-n\beta q_v\lambda_i>0$ for every $i$. Since $q_v\ge r_y^2\ge0$ and $G$ is positive semidefinite, this condition makes every $M_i$ positive definite. Completing the square centers the inner integral at $\beta Nc_iM_i^{-1}\one$ and gives
\[
  \int_{\R^n}\exp\Big(-\tfrac12\ip{u_i}{M_iu_i}
      +\beta Nc_i\ip{\one}{u_i}\Big)\,\dd u_i
  =\frac{(2\pi)^{n/2}}{\sqrt{\det M_i}}\,
    \exp\Big(\tfrac12(\beta Nc_i)^2\ip{\one}{M_i^{-1}\one}\Big).
\]
This identity is exact for positive integer $n$ on the stated positivity domain. On the continued branch $\zeta>\chi_\beta\lambda_{\max}(G)$,
\begin{equation}\label{eq:nlimits}
\begin{aligned}
  \lim_{n\to0}\frac1n\log\det\frac{M_i}{\beta N}
  &= \log(\zeta-\chi_\beta\lambda_i)-\frac{\beta q_v\lambda_i}{\zeta-\chi_\beta\lambda_i},\\
  \ip{\one}{M_i^{-1}\one}
  &= \frac{n}{\beta N(\zeta-\chi_\beta\lambda_i)}+O(n^2).
\end{aligned}
\end{equation}
These are the first-order expansions of the determinant and contraction above. Continuing from the integer-$n$ positivity domain to this branch remains a formal step.

\subsection{The zero-temperature potential}

We suppress the $N$ dependence of the saddle potentials below. They use the finite-$N$ spectral data and the leading large-$N$ entropy terms; their identification with the quenched free energy remains formal. Taking the logarithm of the inner Gaussian integral above gives $-\tfrac12\log\det(M_i/2\pi)+\tfrac12(\beta Nc_i)^2\ip{\one}{M_i^{-1}\one}$ per eigencoordinate. The source term passes through the $v$-integration as $\beta Nn\,h_yr_y$. Adding it, the multiplier constant $\beta N\zeta n/2$, and the entropy \eqref{eq:ventropyapp}, dividing by $\beta Nn$, dropping additive constants independent of the order parameters and sources, and inserting \eqref{eq:nlimits} at $n\to0$ leaves the finite-temperature potential
\begin{equation}\label{eq:finiteT}
\begin{aligned}
  \Psi_\beta(r_y,q_v,\zeta)
  \;=\;\frac{\zeta}{2}
  &-\frac{1}{2\alpha_{x,N}\beta}\avg{\log(\zeta-\chi_\beta\lambda)}
  +\frac{1}{2\alpha_{x,N}}\avg{\frac{q_v\lambda}{\zeta-\chi_\beta\lambda}}\\[2pt]
  &+\frac12\sum_i\frac{c_i^2}{\zeta-\chi_\beta\lambda_i}
  +\frac{\phi(q_v,r_y)}{2\alpha_{y,N}\beta}
  +h_yr_y,
\end{aligned}
\end{equation}
where $\avg{g(\lambda)}:=\Dx^{-1}\sum_{i=1}^{\Dx}g(\lambda_i)$. The factors $\Dx/N=1/\alpha_{x,N}$ and $\Dy/N=1/\alpha_{y,N}$ appear outside the spectral averages and the entropy, respectively. Each contribution is $Nn$ times a function of the order parameters up to $o(Nn)$, so the $n\to0$ limits may be inserted before the $N\to\infty$ Laplace evaluation, whose maximum becomes a stationary point under the continuation. The exchange is formal. The Gaussian prefactors combine with the normalized-sphere factors before the large-$N$ limit; their remaining contribution vanishes at zero temperature. Within the formalism, $f(h)=\lim_{\beta\to\infty}\lim_{N\to\infty}\stat\Psi_\beta$. Now let $\beta\to\infty$ with $q_v\to1$ and $\chi:=\lim\beta(1-q_v)$ finite, so that $\chi_\beta\to\chi>0$. At a fixed interior point $\zeta>\chi\lambda_{\max}(G)$, the logarithmic average in \eqref{eq:finiteT} stays bounded, so its contribution is $O(1/\beta)$. This estimate is uniform on compact sets separated from the spectral boundary. The adjacent average tends to $\tfrac{1}{2\alpha_{x,N}}\avg{\lambda/(\zeta-\chi\lambda)}$, and substituting $1-q_v=\chi/\beta$ supplies the remaining power of $\beta$ in the entropy term,
\[
  \frac{\phi(q_v,r_y)}{2\alpha_{y,N}\beta}
  \;=\;\frac{\log(\chi/\beta)}{2\alpha_{y,N}\beta}
  +\frac{1-\chi/\beta-r_y^2}{2\alpha_{y,N}\chi}
  \;\xrightarrow[\ \beta\to\infty\ ]{}\;
  \frac{1-r_y^2}{2\alpha_{y,N}\chi}.
\]
What remains is the zero-temperature potential
\begin{equation}
  \Psi_h(r_y,\zeta,\chi)
  = \frac{\zeta}{2}
  +\frac{1}{2\alpha_{x,N}}\avg{\frac{\lambda}{\zeta-\chi\lambda}}
  +\frac12\sum_i\frac{c_i^2}{\zeta-\chi\lambda_i}
  +\frac{1-r_y^2}{2\alpha_{y,N}\chi}
  +h_yr_y ,
  \label{eq:freeenergyh}
\end{equation}
on the domain $\{|r_y|\le1,\ \chi>0,\ \zeta>\chi\lambda_{\max}(G)\}$, with the overlap bound inherited from $r_y^2\le q_v<1$. At $h=0$ the source is $c_i=s_N r_y\han_{N,i}$ and the sum $\sum_ic_i^2/(\zeta-\chi\lambda_i)$ collapses to $s_N^2r_y^2F(\zeta,\chi)$ with $F(\zeta,\chi):=\ip{\han_N}{(\zeta I-\chi G)^{-1}\han_N}$, leaving the zero-source free energy
\begin{equation}
  \Psi(r_y,\zeta,\chi)
  = \frac{\zeta}{2}
  +\frac{1}{2\alpha_{x,N}}\avg{\frac{\lambda}{\zeta-\chi\lambda}}
  +\frac{s_N^2r_y^2}{2}\,F(\zeta,\chi)
  +\frac{1-r_y^2}{2\alpha_{y,N}\chi}
  \label{eq:freeenergymain}
\end{equation}
on the same domain. At $G=I$, every denominator in the spectral average and in $F$ becomes $\zeta-\chi$.

Stationarity of \eqref{eq:freeenergymain} in $(\zeta,\chi,r_y)$ gives
\begin{align}
  \frac{1}{\alpha_{x,N}}\avg{\frac{\lambda}{(\zeta-\chi\lambda)^2}}
    + s_N^2r_y^2\,\ip{\han_N}{(\zeta I-\chi G)^{-2}\han_N} &= 1,
    \tag{S1}\label{eq:S1main}\\[1pt]
  \frac{1}{\alpha_{x,N}}\avg{\frac{\lambda^2}{(\zeta-\chi\lambda)^2}}
    + s_N^2r_y^2\,\ip{\han_N}{G(\zeta I-\chi G)^{-2}\han_N}
    &= \frac{1-r_y^2}{\alpha_{y,N}\chi^2},
    \tag{S2}\label{eq:S2main}\\[1pt]
  r_y\Big(s_N^2F(\zeta,\chi)-\frac{1}{\alpha_{y,N}\chi}\Big) &= 0 .
    \tag{S3}\label{eq:S3main}
\end{align}
These equations follow by direct differentiation of \eqref{eq:freeenergymain}. At $r_y=0$ and $h=0$, every $c_i$ vanishes, so all three source derivatives of $\Psi_h$ are zero. This is the zero-overlap RS branch. To locate it, let $w=\zeta/\chi$ and use $T_N$, $w_{R,N}$, and $z_{\mathrm{edge},N}$ from \eqref{eq:randomfp} and \eqref{eq:edgetransfer}. Multiplying \eqref{eq:S2main} by $\chi^2$ gives
\[
  \frac1N\sum_i\frac{\lambda_i^2}{(w-\lambda_i)^2}
  =-\partial_w\big[w\psi_N(w)\big]=\frac1{\alpha_{y,N}},
  \qquad \partial_wT_N(w)=0.
\]
Since $G\ne0$, the convexity formula \eqref{eq:TNconvex} and the endpoint signs give a unique root above $\lambda_{\max}(G)$, namely $w_{R,N}$. Equation~\eqref{eq:S1main} separately determines the susceptibility:
\[
  \chi^2=\frac1N\sum_i\frac{\lambda_i}{(w-\lambda_i)^2}
  =-\psi_N'(w)>0.
\]
Using $\partial_wT_N(w)=\alpha_{y,N}^{-1}+\psi_N(w)+w\psi_N'(w)=0$ then gives
\[
  \zeta^2=w^2\chi^2=-w^2\psi_N'(w)
  =T_N(w_{R,N})=z_{\mathrm{edge},N}^2.
\]
Since $\zeta>0$, this saddle sits at the noise edge. The combination \eqref{eq:zScombo} gives $\Psi=\zeta$ on both zero-source branches, as shown in the saddle-value calculation below. The replica prescription selects the branch with the larger $\zeta$, the supercritical root whenever it exists. This selection remains part of the formal assumptions (R1)--(R5).

The derivation produces the transforms \eqref{eq:transforms} as eigenvalue sums of the deformed model. Writing $d_i(z):=z^2-\lambda_i/(\alpha_{y,N}(1+\delta))$, the change of variables $w=\alpha_{y,N}(1+\delta)z^2$ gives
\[
  \mone(z)=z\sum_i\frac{\han_{N,i}^2}{d_i(z)},
  \qquad
  \mcross(z)=z\sum_i\frac{\han_{N,i}u_{0,i}}{d_i(z)}.
\]
The third transform follows directly from the fixed point. Since $w-\lambda_i=\alpha_{y,N}(1+\delta)d_i(z)$, that equation reads $N^{-1}\sum_i\lambda_i/d_i(z)=-\delta$. Since $\alpha_{y,N}\Dy=N$,
\[
  \sum_i\frac1{d_i(z)}+\frac{\Dy-\Dx}{z^2}
  =\frac1{z^2}\left[\Dy+
    \frac{\sum_i\lambda_i/d_i(z)}{\alpha_{y,N}(1+\delta)}\right]
  =\frac{\Dy}{z^2(1+\delta)}.
\]
Multiplying by $z/\Dy$ gives, identically on the branch,
\begin{equation}
  \frac{z}{\Dy}\Bigg[\sum_{i=1}^{\Dx}\frac{1}{d_i(z)}
  +\frac{\Dy-\Dx}{z^2}\Bigg]
  \;=\;\frac{1}{z(1+\delta)}\;=\;\mtwo(z) .
  \label{eq:companion}
\end{equation}
The left side is the deterministic-equivalent right-sector co-resolvent trace of the deformed model.

One derivative identity will be used below. Differentiating $z^2=T_N(w)$ from \eqref{eq:randomfp} gives $w'(z)\,\partial_wT_N(w)=2z$. Since $K=w/z$ and $\partial_wT_N(w)=w(1+K^2\psi_N'(w))/K^2$, this becomes
\begin{equation}
  w'(z)\,\big(1+K^2\psi_N'(w)\big) \;=\; 2K . \label{eq:wprime}
\end{equation}

\subsection{Equivalence with the secular system}

At a stationary point, we identify the spectral coordinates of \eqref{eq:fixedpoint} through
\begin{equation}
  \hz := \zeta,
  \qquad
  \delta := \frac{1}{\alpha_{y,N}\chi\zeta}-1,
  \qquad\text{equivalently}\quad
  \chi = 1/K(\hz).
  \label{eq:saddlecoordinates}
\end{equation}

\begin{proposition}[Algebraic equivalence of the positive RS saddle and the increasing-branch secular solution]
\label{prop:rsequiv}
Let $(r_y,\zeta,\chi)$ with $r_y>0$, $\chi>0$, and $\zeta>\chi\lambda_{\max}(G)$ solve the stationarity system \eqref{eq:S1main}--\eqref{eq:S3main}, and define $(\hz,\delta)$ by \eqref{eq:saddlecoordinates}. Then $\delta\in(-1,0)$ solves \eqref{eq:fixedpoint} at $z=\hz$, the value $\hz$ solves $\Dbar(\hz)=1/s_N^2$, and at this saddle,
\begin{equation}
\begin{aligned}
  r_y^2 &= \frac{-2\mtwo(\hz)}{s_N^2\,\Dbar'(\hz)},\\
  \left.\partial_{h_p}\Psi_h\right|_{h=0}
  &=s_N\,r_y\,\mone(\hz),
  \qquad
  \left.\partial_{h_x}\Psi_h\right|_{h=0}
  =s_N\,r_y\,\mcross(\hz).
\end{aligned}
  \label{eq:saddleoverlaps}
\end{equation}
The derivative with respect to $h_y$ is $r_y$. Under \textbf{(R1)}--\textbf{(R5)}, the limits of these finite-$N$ saddle derivatives give the source observables in \eqref{eq:sources}. Conversely, every real solution of \eqref{eq:fixedpoint} and $\Dbar(\hz)=1/s_N^2$ with $\hz>0$, $\delta\in(-1,0)$, $w>\lambda_{\max}(G)$, and $1+K^2\psi_N'(w)>0$ arises in this way.
\end{proposition}

\begin{proof}[Proof of \autoref{prop:rsequiv}]
Equation~\eqref{eq:saddlecoordinates} gives $1/(\alpha_{y,N}\chi)=\zeta(1+\delta)$, hence, with $w=\alpha_{y,N}(1+\delta)\hz^2$ and $K=\hz\,\alpha_{y,N}(1+\delta)$,
\begin{equation}
  \zeta-\chi\lambda_i
  = \frac{\hz^2\alpha_{y,N}(1+\delta)-\lambda_i}{\hz\,\alpha_{y,N}(1+\delta)}
  = \frac{w-\lambda_i}{K}, \label{eq:denomK}
\end{equation}
so $(\zeta I-\chi G)^{-1}=K(wI-G)^{-1}$ and every spectral form of \eqref{eq:S1main}--\eqref{eq:S3main} becomes a resolvent form of $G$ at $w$:
\begin{equation}
\begin{gathered}
  F(\zeta,\chi) = K\,\Phi_N(w)=\mone(\hz),
  \qquad
  \ip{\han_N}{(\zeta I-\chi G)^{-2}\han_N} = -K^2\Phi_N'(w),\\
  \frac{1}{\alpha_{x,N}}\avg{\frac{\lambda}{(\zeta-\chi\lambda)^2}} = -K^2\psi_N'(w),
  \qquad
  \frac{1}{\alpha_{x,N}}\avg{\frac{\lambda}{\zeta-\chi\lambda}} = K\,\psi_N(w).
\end{gathered}
\label{eq:wtrans}
\end{equation}

\emph{(a) The secular equation.} On the branch $r_y>0$, \eqref{eq:S3main} reads $s_N^2F(\zeta,\chi)=1/(\alpha_{y,N}\chi)$. By \eqref{eq:wtrans} and $1/(\alpha_{y,N}\chi)=\hz(1+\delta)=K/\alpha_{y,N}$ this is $s_N^2K\Phi_N(w)=K/\alpha_{y,N}$, that is $\alpha_{y,N}s_N^2\Phi_N(w)=1$. Since $\Dbar=\mone\mtwo=\alpha_{y,N}\Phi_N(w)$ by \eqref{eq:transforms}, and step (b) identifies $\delta$ with the branch solution of \eqref{eq:fixedpoint} at $\hz$, this is $\Dbar(\hz)=1/s_N^2$.

\emph{(b) The fixed point.} The identity
\[
  \zeta(\zeta I-\chi G)^{-2}
  -\chi G(\zeta I-\chi G)^{-2}=(\zeta I-\chi G)^{-1}
\]
contracts the two signal terms in $\zeta\cdot$\eqref{eq:S1main}$-\chi\cdot$\eqref{eq:S2main}. It also applies to the spectral averages, giving
\begin{equation}
  \frac{1}{\alpha_{x,N}}\avg{\frac{\lambda}{\zeta-\chi\lambda}}
  +s_N^2r_y^2\,F(\zeta,\chi)
  = \zeta-\frac{1-r_y^2}{\alpha_{y,N}\chi} . \label{eq:zScombo}
\end{equation}
On the branch $r_y>0$, \eqref{eq:S3main} gives $s_N^2F=1/(\alpha_{y,N}\chi)$ exactly, so the left $r_y^2$ term is $r_y^2/(\alpha_{y,N}\chi)$ and cancels its counterpart on the right, leaving a relation between $\zeta$ and $\chi$ alone,
\begin{equation}
  \frac{1}{\alpha_{x,N}}\avg{\frac{\lambda}{\zeta-\chi\lambda}}
  = \zeta-\frac{1}{\alpha_{y,N}\chi} . \label{eq:Ttilde}
\end{equation}
Translating through \eqref{eq:wtrans}, the left side is $K\psi_N(w)$ and the right side is $\hz-\hz(1+\delta)=-\hz\,\delta$. Dividing $K\psi_N(w)=-\hz\delta$ by $\hz$ and using $K/\hz=\alpha_{y,N}(1+\delta)$ gives $\delta=-\alpha_{y,N}(1+\delta)\,\psi_N(w)$, which is \eqref{eq:fixedpoint} at $z=\hz$. The branch is the required one: $\chi,\zeta>0$ give $1+\delta=1/(\alpha_{y,N}\chi\zeta)>0$, while $\zeta>\chi\lambda_{\max}(G)$ puts $w$ above the spectrum by \eqref{eq:denomK}, so $\psi_N(w)>0$ and $\delta\in(-1,0)$.

\emph{(c) The right overlap.} Translating \eqref{eq:S1main} through \eqref{eq:wtrans}, $-K^2\psi_N'(w)-s_N^2r_y^2K^2\Phi_N'(w)=1$, that is
\begin{equation}
  r_y^2 = \frac{1+K^2\psi_N'(w)}{-s_N^2K^2\Phi_N'(w)} . \label{eq:ryS1}
\end{equation}
Along the branch $\Dbar(z)=\alpha_{y,N}\Phi_N(w(z))$, so \eqref{eq:wprime} gives $\Dbar'(\hz)=2K\alpha_{y,N}\Phi_N'(w)/(1+K^2\psi_N'(w))$, while $\mtwo(\hz)=1/(\hz(1+\delta))=\alpha_{y,N}/K$ by \eqref{eq:transforms}. Hence
\[
  \frac{-2\mtwo(\hz)}{s_N^2\,\Dbar'(\hz)}
  = \frac{-2\alpha_{y,N}/K}{2K\alpha_{y,N}s_N^2\Phi_N'(w)/(1+K^2\psi_N'(w))}
  = \frac{1+K^2\psi_N'(w)}{-s_N^2K^2\Phi_N'(w)} = r_y^2 ,
\]
by \eqref{eq:ryS1}, the first claim of \eqref{eq:saddleoverlaps}.

\emph{(d) The sources.} The only $h_x$-, $h_p$-dependence of \eqref{eq:freeenergyh} enters through $c_i=s_N r_y\han_{N,i}+h_xu_{0,i}+h_p\han_{N,i}$ of \eqref{eq:Mimain}, and by \autoref{lem:sources} together with the envelope theorem only this explicit dependence survives at the saddle. Differentiating $\tfrac12\sum_ic_i^2/(\zeta-\chi\lambda_i)$ at $h=0$,
\[
  \left.\partial_{h_p}\Psi_h\right|_{h=0}
  = \sum_i\frac{s_N r_y\han_{N,i}\cdot\han_{N,i}}{\zeta-\chi\lambda_i}
  = s_N\,r_y\,F(\zeta,\chi) = s_N\,r_y\,\mone(\hz),
\]
and with $\han_{N,i}$ replaced by $u_{0,i}$ in the differentiated factor, $\left.\partial_{h_x}\Psi_h\right|_{h=0} =s_N\,r_y\,K\Phix_N(w)=s_N\,r_y\,\mcross(\hz)$. Also, $\partial_{h_y}\Psi_h=r_y$. These are the remaining claims of \eqref{eq:saddleoverlaps}. Squaring them and using $\Dbar=\mone\mtwo=1/s_N^2$ at the root gives the formal squared overlap formulas, with the factor $(\mcross/\mone)^2$ as the squared ratio of the two source derivatives.

\emph{Converse.} Let $(\hz,\delta)$ solve \eqref{eq:fixedpoint} and $\Dbar(\hz)=1/s_N^2$, with $\hz>0$, $\delta\in(-1,0)$, and $w=\alpha_{y,N}(1+\delta)\hz^2>\lambda_{\max}(G)$ on the increasing branch $1+K^2\psi_N'(w)>0$. Let $\zeta:=\hz$ and $\chi:=1/K$, and define $r_y>0$ by \eqref{eq:ryS1}, positive there because $\Phi_N'(w)<0$. Since $w>\lambda_{\max}(G)$, \eqref{eq:denomK} gives $\zeta>\chi\lambda_{\max}(G)$, the domain of \eqref{eq:freeenergymain}. Reading (a) backwards recovers \eqref{eq:S3main}. The fixed point translated back through \eqref{eq:wtrans} is \eqref{eq:Ttilde}, which with \eqref{eq:S3main} restores the combination \eqref{eq:zScombo}, and \eqref{eq:ryS1} is \eqref{eq:S1main} itself, so \eqref{eq:S2main} follows from the combination on dividing by $\chi>0$. Its left side is strictly positive since $G\ne0$, so $r_y^2<1$, as required for a spherical overlap. By \eqref{eq:wprime} the branch condition says $w'(\hz)>0$, so $\hz$ lies on the increasing arc of the branch and the correspondence between saddle and root is a bijection.
\end{proof}

\subsection{The saddle value}

The same calculation gives the stationary value on both zero-source branches. Multiplying \eqref{eq:S1main} by $\zeta$ and subtracting $\chi$ times \eqref{eq:S2main} gives \eqref{eq:zScombo}. This identity uses only the first two stationarity equations, so it holds at both zero and positive response overlap. Substituting it into \eqref{eq:freeenergymain},
\[
  \Psi=\frac{\zeta}{2}
  +\frac12\Big(\zeta-\frac{1-r_y^2}{\alpha_{y,N}\chi}\Big)
  +\frac{1-r_y^2}{2\alpha_{y,N}\chi}=\zeta.
\]
On the positive branch $\zeta=\hz$, the candidate outlier location; at zero overlap $\zeta=z_{\mathrm{edge},N}$ by the calculation above. Under the formal limit exchanges, \autoref{lem:sources} identifies $f(0)$ with $\lim_{N\to\infty}\E_{\WG}\sigma_1(C_{\mathrm G})$. The positive saddle therefore predicts $f(0)=\lim_{N\to\infty}\hz$, where $\hz$ here is the finite-$N$ conditional secular root. The zero-overlap branch gives the corresponding limit of $z_{\mathrm{edge},N}$. The squared overlaps do not depend on the sign choice of \autoref{lem:sources}. The signed identities in \eqref{eq:saddleoverlaps} do. At $G=I$ and $\han_N=u_0$ the RS system reduces to the standard rectangular spiked-matrix formulas \eqref{eq:closedforms}, evaluated in \autoref{prop:limits}(i).

\clearpage
\section{Additional experiments}
\label{app:robustness}

\subsection{Simulation configurations}
\label{app:configs}

We distinguish the asymptotic spectral threshold $\theta_c$ of \eqref{eq:boundary}, a detection half-crossing, and a recovery-level crossing. The detection half-crossing is the signal strength at which the probability that the top singular value exceeds a fixed null quantile reaches one half. A recovery-level crossing is the interpolated signal strength at which a trial-mean recovery curve last crosses a specified level. We apply the same crossing statistic to each available candidate curve; otherwise, we report its offset from the predicted transition.

A \emph{supercritical} mean absolute error averages over grid points with $\theta>1.1\,\theta_c$. We exclude $(\theta_c,1.1\theta_c]$ because finite-size rounding dominates this band. Unless stated otherwise, bands and error bars show $\pm1$ SD across trials. Confidence intervals use percentile bootstraps over whole trial curves, with $2000$ resamples.

For the response-retention collapse in \autoref{fig:theory_validation}(b), we use Gaussian-QR designs with $N=2000$, $(\alpha_x,\alpha_y)=(5,20)$, $\rhox=0.7$, and $\rhoy\in\{1,0.8,0.6,0.4,0.2\}$, with $60$ trials per point on a grid expressed in $\theta\sqrt{\rhoy}$. The angular study in \autoref{fig:geometry} uses one fixed flat orthogonal design with $N=1024$, $\Dx=512$, $\Dy=256$, $(\alpha_x,\alpha_y)=(2,4)$, $\rhox\in[0.07,1]$, and $\rhoy=1$. It uses $60$ mask and noise draws per $(\rhox,\theta)$ grid point, with $\theta\in\{1,2,8\}$.

\subsection{Theorem-class validation on nested flat orthogonal designs}
\label{app:flatvalidation}

\autoref{cor:flatdesigns} covers nested Sylvester/Walsh--Hadamard designs, so a simulation inside the theorem class is available at every sample size of the form $N=2^m$, the orders at which the nested construction exists. We build the design from a Walsh--Hadamard matrix of that order, take $D_x$ of its columns, and randomize by independent row and column permutations and row and column sign changes, all of which the corollary allows. Every such design satisfies $\abs{(\Xs)_{ki}}=1$ and $\Xs^\top\Xs=N I_{\Dx}$ to machine precision, so its leverage is exactly equal and $\tau_N=1/\alpha_{x,N}$ exactly, and it meets the entry bound of \autoref{ass:incoherence} with $\norm{\Xs}_{\max}=1$.

\autoref{fig:theory_validation}(a) runs the transition protocol on such a design and on a Gaussian-QR design of the same size, against the same curves, at $N=2048$, $\Dx=\Dy=1024$, $(\alpha_x,\alpha_y)=(2,2)$, $\rhox=0.3$ and $\rhoy=0.6$. The grid has $25$ equally spaced signal strengths in $[0.5,2.5]\theta_c$ plus the isotropic reference threshold, with $100$ trials per point. Within each trial, the design, both masks, and the noise are shared across the signal grid.

Supercritical mean absolute error is 0.001 for $R_x^2$ and 0.002 for $R_y^2$ on the flat orthogonal design, and 0.001 and 0.004 on Gaussian-QR. The recovery-level crossing, averaged over the levels $\{0.05,0.10,0.15,0.20\}$ on the trial-mean $R_x^2$ curve with a percentile bootstrap over whole trial curves, is 1.619 and 1.626 against 1.629 for the limiting curve of \autoref{thm:finitesignal} under the same rule, and both intervals cover it. At this resolution both design classes track the same limiting curves. This supports the universality prediction in \autoref{rem:prediction}, suggesting that the deterministic equivalent may extend beyond the class in which we verify it.

\subsection{The phase diagram}
\label{app:phasediagram}

\autoref{fig:phase_diagram}(a) maps recovery over the whole $(\theta,\rho)$ plane. The grid is $90\times90$ over $\theta\in[0.2,2.0]$ and $\rho\in[0.05,0.95]$ at $N=2000$, $D_x=300$, $D_y=240$, with $30$ trials per cell for $\theta\in[0.6,1.8]\,\theta_c(\rho)$ and $6$ trials elsewhere. Every cell reuses the same trial seeds, so the Monte Carlo errors of different cells are correlated. Each joint retention is realized symmetrically, $\rho_x=\rho_y=\sqrt\rho$, so this is one slice of the two-retention plane, and \autoref{fig:theory_validation}(b) isolates the response direction. Over the grid the supercritical mean absolute error against the prediction is 0.007.

A separate run resolves the boundary itself, on a band of $\pm35\%$ around $\theta_c(\rho)$ at every fourth retention with 60 trials per point. Over its $460$ points the mean absolute deviation from the predicted curves is 0.029, the larger part of it below the threshold, where the prediction is identically zero and the deviation is the finite-$N$ floor.

\begin{figure}[!t]
\centering
\includegraphics[width=\textwidth]{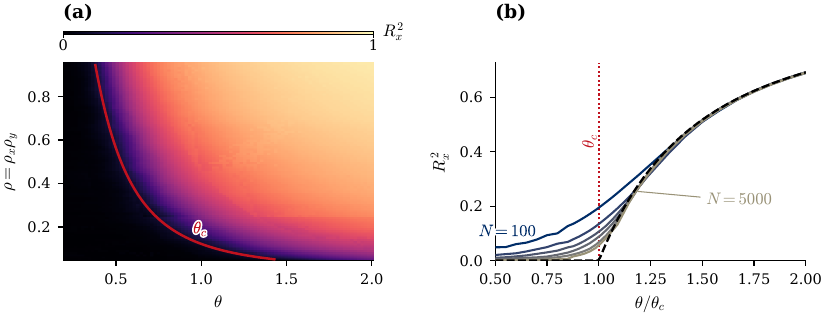}
\caption{%
\textbf{Recovery across the $(\theta,\rho)$ plane.} \textbf{(a)} Measured $R_x^2$ on a $90\times90$ grid at $N=2000$, $\Dx=300$, $\Dy=240$, with $30$ trials per cell for $\theta\in[0.6,1.8]\,\theta_c(\rho)$ and $6$ elsewhere, and the predicted boundary $\theta_c(\rho)$ of \eqref{eq:boundary} in red. Each joint retention is realized symmetrically, $\rhox=\rhoy=\sqrt\rho$. The change in trial count at $1.8\,\theta_c(\rho)$ and the shared trial seeds leave faint seams in the field, along that curve and near $\rho=0.25$. \textbf{(b)} The transition sharpens with sample size at $\alpha_x=\alpha_y=2.5$ and $\rhox=\rhoy=0.8$, against the limiting curve in black. The unlabeled curves are $N=250$, $500$, $1000$ and $2000$, in that order between the two named. Repeats fall from $800$ at $N=100$ to $25$ at $N=5000$. }
\label{fig:phase_diagram}
\end{figure}
\newpage

Panel~(b) shows the transition sharpening as $N$ grows. This finite-size sweep uses $(\alpha_x,\alpha_y)=(2.5,2.5)$, $\rhox=\rhoy=0.8$, and $32$ signal strengths in $[0.5,2]\theta_c$. The sample sizes are $N\in\{100,250,500,1000,2000,5000\}$, with trial counts decreasing from $800$ to $25$.

\subsection{Robustness to non-Gaussian noise}
\label{app:nongaussian}

Our analysis assumes Gaussian response noise. We test the predicted transition and recovery under changes in the noise tails and column variances.

\paragraph{Experimental setup.} We use $N=1000$, $\Dx=200$, $\Dy=150$, $(\alpha_x,\alpha_y)=(5,6.67)$, and $\rhox=\rhoy=0.7$ ($\rho=0.49$), with $100$ trials at each of $25$ signal strengths in $[0.6,2.0]\theta_c$.

The five i.i.d.\ noise families have unit variance: standard Gaussian noise (excess kurtosis $0$), standardized Student-$t$ noise with $\nu=5$, $4.5$, and $3$ (excess kurtosis $6$, $12$, and infinite, respectively), and standardized Laplace noise (excess kurtosis $3$).

For the sixth, column-heteroskedastic setting, we draw independent $V_j\sim\mathrm{Uniform}[0.5,1.5]$ and then $Z_{ij}\mid V_j\sim\mathcal N(0,V_j)$. It has unit average variance and marginal excess kurtosis $0.25$. The shared column scales also change the covariance structure, so this setting tests anisotropy as well as the marginal noise distribution.

\paragraph{Results.} In \autoref{fig:nongaussian} the i.i.d.\ noise families with finite fourth moments stay close to the Gaussian-model prediction curves, while Student-$t(\nu=3)$ degrades in the supercritical regime. Supercritical mean absolute errors are 0.008 (Gaussian), 0.007 ($t_5$), 0.006 ($t_{4.5}$), 0.239 ($t_3$), 0.005 (Laplace) and 0.046 (column-heteroskedastic). The column-heteroskedastic setting also deviates and changes the response covariance structure. These results do not establish a monotone effect of kurtosis.

\begin{figure}[!t]
\centering
\includegraphics[width=\textwidth]{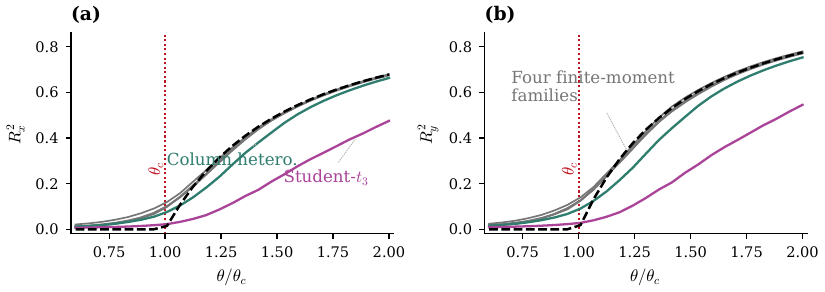}
\caption{%
\textbf{Robustness of the phase transition to non-Gaussian noise.} Black is the Gaussian-model prediction and red is $\theta_c$. Gaussian, Student-$t$ at $\nu=5$ and $\nu=4.5$, and Laplace lie within $0.025$ of one another, no more than the trial standard deviation, so their curves share a gray color and a label. \textbf{(a)} $R_x^2$, with separate labels for the two families that deviate, and \textbf{(b)} $R_y^2$, with the same ordering of the noise families. Parameters: $N=1000$, $\Dx=200$, $\Dy=150$, $\rho_x=\rho_y=0.7$, 100 trials. }
\label{fig:nongaussian}
\end{figure}
\newpage

\subsection{Comparison with baseline imputation methods under MCAR}
\label{app:baselines}

This experiment asks whether imputation improves on the zero-filled estimator under MCAR. We select all baseline hyperparameters on held-out observed entries to avoid arbitrary choices that could affect the comparison.

\paragraph{Methods.} We compare seven estimators.

(1)~\textbf{PLS-SVD:} zero-fill both views and form $C=(N\sqrt{\rho})^{-1}X^\top Y$.

(2)~\textbf{Mean imputation:} replace missing entries by column means, then center both views.

(3)~\textbf{Rank-one response imputation:} initialize both views by column means, then alternate PLS-SVD estimation with rank-one prediction of missing response entries. The design remains mean-imputed and observed response entries are retained. This alternating method has no tuning parameters.

(4)~\textbf{Iterative SVD:} alternate rank-truncated reconstruction and restoration of observed entries \citep{troyanskaya2001svdimpute}. The rank is selected separately for each view from $\{1,2,5,10,20\}$ by squared error on a held-out tenth of its observed entries.

(5)~\textbf{Nuclear-norm imputation:} use soft-impute \citep{mazumder2010softimpute}, which alternates soft-thresholded SVD reconstruction with restoration of observed entries, selecting the threshold by squared error on a held-out tenth of each view's observed entries. We denote by $\lambda_0$ the largest singular value of that view's zero-filled training matrix after the held-out entries have been removed. The threshold grid includes $\lambda_0$, which gives the zero solution in the selection fit, and smaller thresholds. \autoref{app:baselineimpl} gives the full grid and the iteration limits.

(6)~\textbf{Complete-data reference:} apply PLS-SVD to unmasked $(X_\star,Y_\star)$.

(7)~\textbf{Prewhitened PLS-SVD:} re-whiten the masked design before forming the cross-covariance, then map the left direction back before scoring. This estimator is outside the scope of \autoref{thm:ceiling}.

\paragraph{Selection protocol.} Selection uses a realization independent of the evaluation trials, so no trial both chooses and scores its own hyperparameters. Selecting at each signal strength gives the baselines information a practitioner would not have, which makes the comparison conservative for the zero-filled estimator. The iterative-SVD rank selected on each view is $1$ at every signal strength in both configurations.

\paragraph{Experimental setup.} \autoref{fig:baselines} shows both configurations, sweeping $\theta$ from $0.5\theta_c$ to $2.0\theta_c$ (i) and to $3.5\theta_c$ (ii), with $20$ and $24$ grid points, respectively, and $30$ trials per point: (i) $N=1000$, $\Dx=200$, $\Dy=150$, $\rho_x=\rho_y=0.7$; and (ii) $N=800$, $\Dx=400$, $\Dy=200$ ($\alpha_x=2$, $\alpha_y=4$), $\rho_x=\rho_y=0.3$.

\begin{figure}[!t]
\centering
\includegraphics[width=\textwidth]{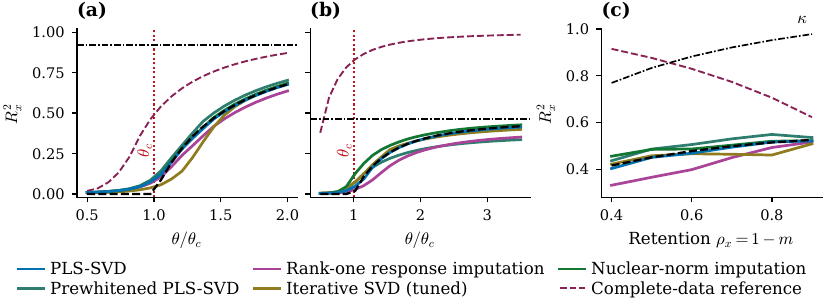}
\caption{%
\textbf{Baseline imputation methods under MCAR.} The black dashed curve is the prediction for the zero-filled estimator, and the black dash-dot line is its ceiling $\kappa$, shown for reference and not as a bound on alternatives. Red marks $\theta_c$. Centered mean imputation agrees with the zero-filled estimator to within 0.004, so we omit its curve. \textbf{(a,b)} Recovery curves in a moderate- and a low-retention configuration. The prewhitened direction is mapped back to the original design coordinates before scoring. Every imputation baseline has its hyperparameters selected on held-out observed entries. Complete-data PLS-SVD is a reference rather than an upper bound. Bands over trials are omitted for legibility. \textbf{(c)} Recovery at $\theta=1.5\theta_c(\rho_x)$ versus the retention $\rho_x=\rho_y=1-m$ over $\{0.4,0.5,\dots,0.9\}$, with $m$ the missing fraction of each view, the sizes of configuration (i) and hyperparameters selected at each retention. Parameters: configuration (i), $N=1000$, $\Dx=200$, $\Dy=150$, $\rho_x=\rho_y=0.7$; configuration (ii), $N=800$, $\Dx=400$, $\Dy=200$, $\rho_x=\rho_y=0.3$; 30 trials. }
\label{fig:baselines}
\end{figure}

\paragraph{Results.} The zero-filled estimator tracks its predicted curves in both configurations. Under light masking (i) nuclear-norm and mean imputation give recovery similar to zero-fill, with mean supercritical $R_x^2$ of 0.499 and 0.491 against 0.492. The tuned iterative SVD gives 0.450, and re-whitening the design does better at 0.523.

Under heavy masking (ii), where the ceiling falls to $\kappa=0.462$, the methods differ more: nuclear-norm imputation reaches 0.361 against 0.331 for zero-fill and 0.329 for the tuned iterative SVD, and re-whitening falls to 0.276.

Mean imputation matches zero-fill in both configurations once the imputed views are centered; without centering, the cross-product retains the rank-one mean term $N\hat\mu_x\hat\mu_y^\top$, with $\hat\mu_x,\hat\mu_y$ the column means of the imputed views. The tuned nuclear-norm baseline improves on the zero-filled estimator by 0.030 in the strongly masked regime. This paper makes no optimality claim for the zero-filled estimator; it locates that estimator's transition and the overlap it can recover, at both retentions.

On the whitened design view, the hold-out selector chooses the zero solution at every sampled signal strength in both configurations. On the response view, which carries the planted spike, the selector returns the zero solution at a fraction 0.17 of the supercritical grid in configuration (i), with interior thresholds in $[0.85, 0.97]\,\lambda_0$ otherwise, while in configuration (ii) interior thresholds in $[0.41, 0.90]\,\lambda_0$ win throughout its supercritical grid. Up to the scalar $\sqrt{\rho_x\rho_y}$, the zero-filled cross-covariance is unbiased for $X_\star^\top Y_\star/N$ under MCAR \citep{lounici2014}; whether imputation improves on it is the question the comparison answers. Rank-one response imputation gives 0.450 and 0.263, so centering the initial fill does not resolve its poor recovery under heavy masking.

\subsection{Robustness to correlated noise}
\label{app:correlated_noise}

The model of \autoref{sec:setup} assumes i.i.d.\ Gaussian noise. We examine structured noise covariances using a Toeplitz model $(\Sigma_Z)_{jk}=r^{\abs{j-k}}$ with correlation parameter $r$ and a factor model $\Sigma_Z=I+\sigma^2 FF^\top$. The factor matrix $F\in\mathbb{R}^{D_y\times5}$ has independent standard Gaussian entries and is drawn once and held fixed. All runs use $N=1000$, $\Dx=200$, $\Dy=150$, $\rhox=\rhoy=0.7$, $30$ signal strengths in $[0.5,4]\theta_c$, and $400$ trials per point. \autoref{fig:correlated} shows the Toeplitz model at twenty correlations from $0$ to $0.95$. Separate diagnostic runs use $r\in\{0,0.3,0.6,0.9\}$ and factor variances $\sigma^2\in\{0,1,3,10\}$; their errors are reported below.

\begin{figure}[!t]
\centering
\includegraphics[width=\textwidth]{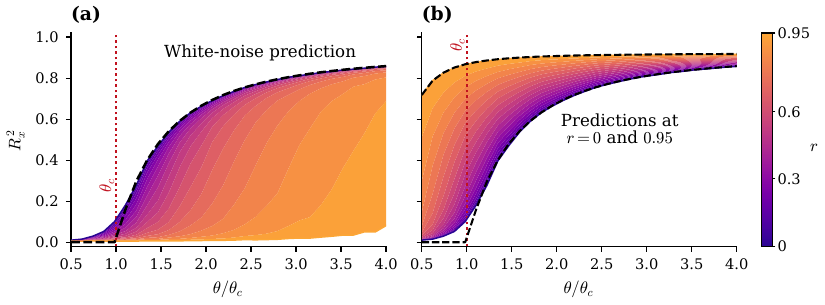}
\caption{%
\textbf{Phase transition under correlated noise.} Color encodes the Toeplitz correlation $r$ across twenty simulated levels from $0$ to $0.95$, and red marks $\theta_c$. \textbf{(a)} Without whitening, recovery falls steadily as $r$ rises, with the strongest correlations suppressing recovery across most of the displayed range. The black dashed curve is the white-noise prediction. \textbf{(b)} After oracle whitening by $\Sigma_Z^{-1/2}$ every level tracks its own prediction, the same curve evaluated at that trial's adjusted spike, with supercritical MAE at most 0.004. The two black dashed curves are the predictions at $r=0$ and $r=0.95$, and they bracket the field. Both panels use the same unadjusted horizontal axis, which is why the whitened field shifts left rather than collapsing. Parameters: $N=1000$, $\Dx=200$, $\Dy=150$, $\rhox=\rhoy=0.7$, with $400$ trials per signal strength. }
\label{fig:correlated}
\end{figure}

\paragraph{Naive PLS-SVD.} Without response whitening, recovery degrades sharply with correlation. At Toeplitz $r=0.9$ ($\mathrm{cond}(\Sigma_Z) \approx 350$), the naive curve differs substantially from the white-noise prediction (MAE 0.547). The factor model also inflates the noise variance. At $\sigma^2=1$, $3$ and $10$ the mean diagonal entry of $\Sigma_Z$ is $6.2$, $16.5$ and $52.8$, with $\mathrm{cond}(\Sigma_Z)\approx205$, $612$ and $2039$. The naive MAE is 0.689, 0.690 and 0.690, and naive $R_x^2$ stays below $0.01$ at every signal strength.

\paragraph{Oracle whitening.} As an oracle diagnostic, we apply the known transformation $\Sigma_Z^{-1/2}$ to the complete response $Y_\star$ before masking and compute the adjusted spike $\theta\,\|\Sigma_Z^{-1/2} v_0\|$ per trial. This checks whether the theory, evaluated at the per-trial adjusted spike, describes the post-whitening effective model. It is not an in-sample estimator of $\Sigma_Z$. Across the four Toeplitz diagnostic levels and the four factor variances the whitened MAE is at most 0.004 (Toeplitz $r=0.9$: 0.002).

\subsection{Discriminating the finite-signal threshold}
\label{app:belowiso}

\autoref{fig:below_isotropic} gives the detailed threshold check underlying \autoref{fig:theory_validation}(c). At $(\alpha_x,\alpha_y)=(2,2)$ and $\rhoy=1$ we sweep $\rhox$ on an eighteen-point grid from $0.10$ to $0.95$, spanning regimes where the masked and isotropic predictions are well separated, nearly coincident, and reversed. Each configuration sweeps $44$ signal strengths over $N\in\{500,1000,2000\}$ at $\Dx=\Dy=N/2$, with $100$ trials per point and $500$ null draws per configuration. Each null draw uses a new design. The signal grid is densified around both predictions so that neither is favored by its resolution.

Both crossings of \autoref{app:configs} are used. The detection half-crossing uses the $0.95$ quantile of the simulated $\theta=0$ top-singular-value distribution as its fixed null quantile. The recovery-level crossing averages the last crossings of the trial-mean $R_x^2$ curve through $L\in\{0.05,0.10,0.15\}$. Because a positive recovery-level crossing is offset from the asymptotic threshold, each candidate is scored by applying the same crossing statistic to its own predicted curve. At $N=2000$ the measurements lie $+0.13$, $-0.44$ and $-1.20$ interval half-widths from the masked curve at $\rhox=0.15$, $0.30$ and $0.80$; \autoref{fig:below_isotropic}(b) shows the comparison with the reference curve across sample sizes and over the whole retention grid. For $\rhox\ge0.70$ the measurement at $N=2000$ lies $1.3$ to $2.1\%$ below the crossing of the masked curve, and every interval excludes that crossing, narrowly at $\rhox=0.70$. The shortfall is $2.1$ to $3.6\%$ at $N=1000$ and $5.4$ to $6.9\%$ at $N=500$. At $\rhox=0.95$ the $N=2000$ interval contains the comparator's crossing. The~detection half-crossing has a positive finite-$N$ offset that is smaller at $N=2000$ than at $N=500$ at every retention and agrees with neither candidate: at $\rhox=0.15$ it lies $+5.3$ interval half-widths from $\theta_c$ and $-24.1$ from the comparator. Comparing recovery-level crossings accounts for the offset between a positive-level crossing and the asymptotic threshold. The near-coincident $\rhox=0.50$ case serves as a calibration rather than a discriminating test.

\begin{figure}[!t]
\centering
\includegraphics[width=\textwidth]{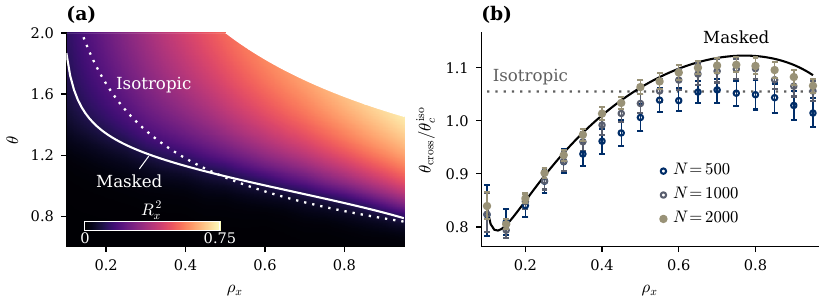}
\caption{%
\textbf{Detailed finite-signal threshold validation.} \textbf{(a)} Mean $R_x^2$ over the $(\rhox,\theta)$ plane at $N=2000$ and eighteen retentions, with the level crossing $\theta_{\mathrm{cross}}$ predicted by each candidate: solid for the masked theory, dotted for the isotropic comparator. The masked curve is the field's $0.1$ contour to within $0.02$, near the edge of the recovered region. At $\rhox=0.25$ the comparator lies at twice that recovery level, inside the recovered region. For $\rhox>0.5$ the field is truncated at the end of the signal grid. \textbf{(b)} Measured $\theta_{\mathrm{cross}}$ in units of $\theta_c^{\mathrm{iso}}$ for three sample sizes, with $95\%$ bootstrap intervals. \autoref{fig:theory_validation}(c) shows the same comparison on the original scale at $N=2000$. The comparator's own crossing is $1.055$, rather than $1$, because a positive-level crossing lies above the asymptotic threshold of its curve. At $N=2000$ the measurement stays within $1.6$ interval half-widths of the masked curve. It lies below that curve at the highest retentions, with larger shortfalls at smaller sample sizes, and up to a quarter below the comparator where the two predictions separate. All runs use $\Dx=\Dy=N/2$, $\rhoy=1$, and $100$ trials per grid point. }
\label{fig:below_isotropic}
\end{figure}

The null runs also give the first direct check of the deformed edge. The reported singular values are normalized to the scale of $C=X^\top Y/(N\sqrt{\rhoxy})$ in \eqref{eq:crosscov}. On that scale, the median top singular value under the null approaches $z_{\mathrm{edge}}$ from below as $N$ grows, at $\rhox=0.15$ by $0.93\%$, $0.64\%$ and $0.46\%$ at $N=500$, $1000$ and $2000$, with the same monotone pattern at every retention.

\subsection{Direction-specific ceilings on engineered and real designs}
\label{app:ceilings}
\label{app:leverage}
\label{app:biodirections}

\begin{figure}[!t]
\centering
\includegraphics[width=\textwidth]{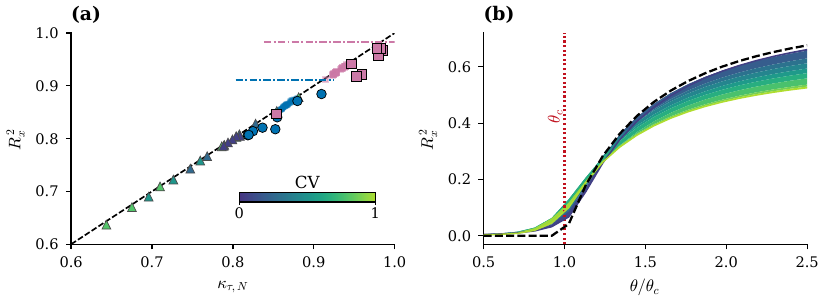}
\caption{%
\textbf{Direction-specific ceilings on engineered and biological designs.} \textbf{(a)} Measured $R_x^2$ at $\theta=30\,\theta_c$ against $\kappa_{\tau,N}$, with $\theta_c$ computed under equal leverage. On the dashed identity line, measured recovery equals the ceiling of \autoref{thm:ceiling}. Triangles are engineered designs ($N=2000$, $\Dx=\Dy=500$, $\rhox=0.5$, $\rhoy=1$, $40$ trials), colored by the coefficient of variation (CV) of their sample leverages. Large markers are the eight leading directions of the biological designs TCGA (blue circles) and PBMC (pink squares), and small markers are $30$ random directions in each ($\rhox=\rhoy=0.7$, $100$ trials). Dash-dot lines show the equal-leverage value $\kappa$ of each biological design. Standard errors are below marker size. \textbf{(b)} Measured $R_x^2$ on engineered designs at ten CV levels from $0.05$ to $1.0$, against the equal-leverage curve of \autoref{thm:finitesignal} (black dashed). The region between adjacent CV levels is shaded. }
\label{fig:ceilings}
\end{figure}

\autoref{thm:ceiling} holds for column-orthogonal designs that are incoherent; it does not require equal leverage, and the design enters only through the coupling $\tau_N$. The finite-signal analysis of \autoref{app:maskavg}, by contrast, is reduced to one scalar variable under equal leverage. \autoref{fig:ceilings} separates the two statements on designs with deliberately unequal leverage, whitened to machine precision, and on the two biological designs.

Rows are drawn with a two-point row-scale pattern and the columns are then whitened, which makes $\Xs^\top\Xs=NI_{\Dx}$ exact to machine precision. The whitening map compresses the nominal contrast, since the squared row norms are $N$ times the diagonal of a hat matrix, so each resulting design is reported by its realized coefficient of variation rather than by its nominal one; no squared row norm can exceed $\alpha_{x,N}$ times the mean $\Dx$. The planted direction is drawn toward the leading or the trailing eigenvector of $\Xs^\top\diag(\norm{(\Xs)_{k\cdot}}^2)\Xs$, which in these realizations reaches couplings on both sides of the equal-leverage value $1/\alpha_x$; the construction does not guarantee couplings below the equal-leverage value. Parameters: $N=2000$, $\Dx=\Dy=500$, $(\alpha_x,\alpha_y)=(4,4)$, $\rhox=0.5$, $\rhoy=1$, $40$ trials per cell.

At thirty times the equal-leverage threshold used for this experiment, the realized coupling spans 0.133 to 0.558 around the equal-leverage value $0.25$, and the implied level $\kappa_{\tau,N}$ spans 0.642 to 0.883 around the equal-leverage $\kappa=0.800$. At that signal strength the equal-leverage finite-signal shortfall is 0.001. The equal-leverage configurations provide a control confirming that it is negligible: their mean lies within 0.003 of $\kappa$, inside the shortfall plus twice the standard error. Raw recovery, with no finite-signal adjustment, follows the coupling each trial realizes: the mean absolute deviation from $\kappa_{\tau,N}$ is 0.011, against 0.049 from the equal-leverage ceiling, and a regression on the diagonal has slope 1.021 and intercept $-0.019$; the shortfall-adjusted deviations, reported as a secondary diagnostic, are 0.011 and 0.048. The high-signal limit therefore tracks the measured coupling on designs well outside the equal-leverage case, as \autoref{thm:ceiling} asserts.

The finite-signal equal-leverage approximation degrades in a measurable way. A separate finite-signal sweep uses $20$ signal strengths in $[0.5,2.5]\theta_c$ across $10$ row-scale ratios, with $40$ trials per cell. At the strongest coupling, the supercritical error of the equal-leverage curves grows from 0.018 at equal leverage (realized coefficient of variation 0.055) to 0.099 at a realized coefficient of variation of 0.998. This quantifies where the equal-leverage reduction remains usable, and connects the verified class to the biological designs of \autoref{sec:results}, which are neither flat nor equal in leverage.

\paragraph{A mask-only oracle.} We remove response noise and take a leading left singular vector of $M=X^\top\diag(b)S_y\diag(v_0)$, where $X=S_x\odot\Xs$ and $b=\Xs u_0$. Its squared overlap with $u_0$ is the mask-only oracle recovery. The exact identity $M/(N\sqrt{\rhoxy})=p_Nv_0^\top+D_N$ retains the finite-$N$ residual defined after \autoref{prop:deltabound}. Under \autoref{ass:incoherence}, \autoref{app:ceiling} gives oracle recovery $\kappa_{\tau,N}+o_P(1)$ as $N\to\infty$. For each biological dataset, we use $\Dx=\Dy=200$, $\rhox=\rhoy=0.7$, $100$ trials, and signal strength $30\theta_c$, with the equal-leverage reference threshold. The two datasets use independent simulations. We pair the oracle and measured recovery on each trial's masks as a finite-sample diagnostic.

The couplings vary widely across directions. Raw recovery follows them: the mean absolute deviation from $\kappa_{\tau,N}$ is 0.017 on TCGA and 0.018 on PBMC for the leading complete-data directions, against 0.078 and 0.046 from the equal-leverage value $\kappa$, and 0.003 and 0.002 for random directions in the same designs.

The random comparison directions are drawn once and held fixed across trials. The deviation from $\kappa_{\tau,N}$ is $-0.017$ $\pm$ 0.0007 and $-0.018$ $\pm$ 0.0011, formed per trial since the eight directions of a dataset share their masks. The paired mask-only oracle accounts for most of the shortfall below $\kappa_{\tau,N}$. After pairing, both datasets retain small residuals that are distinguishable from zero at the reported Monte Carlo precision.

Their magnitudes are comparable to the equal-leverage finite-signal gaps. These gaps are direction-independent, so we report them only for scale and do not subtract them from the per-direction residuals. No direction-specific finite-signal prediction is available, so the comparison does not establish zero error or an exact deterministic equivalent.

\paragraph{Detection.} We compare the detection half-crossings with matched Haar controls of the same size and aspect ratio. On TCGA the detection half-crossing is $1.130\,\theta_c$ at $\rhox=\rhoy=0.7$ against $1.129\,\theta_c$ for its control, and $1.071\,\theta_c$ against $1.143$ at $\rhox=\rhoy=0.5$. On PBMC it lies above its control by $0.097$ at $\rhox=\rhoy=0.7$ and by $0.238$ at $0.5$, in units of $\theta_c$. These finite-sample detection crossings therefore show a larger departure from the matched control as retention falls; they do not directly estimate the asymptotic threshold.

\autoref{tab:kappasurvey} reports $\kappa_{\tau,N}$ at $\rhox\in\{0.9,0.7,0.5\}$ for the eight leading complete-data directions of each dataset. These values are calculated directly from the designs, without Monte Carlo trials.

\begin{table}[!ht]
\centering
\footnotesize
\caption{Direction-specific ceilings and finite-$N$ incoherence readings on the two biological designs. For each leading complete-data direction the coupling $\tau_N$ fixes $\kappa_{\tau,N}$ at any retention, with no Monte Carlo. Columns (i)--(iv) are the finite-$N$ readings of \autoref{ass:incoherence} at this one size, each divided by its logarithmic bound; they describe the designs rather than test the asymptotic condition.}
\label{tab:kappasurvey}
\begin{tabular}{llcccccccc}
\toprule
& & & \multicolumn{3}{c}{$\kappa_{\tau,N}$ at $\rhox$} & \multicolumn{4}{c}{Incoherence reading} \\
\cmidrule(lr){4-6}\cmidrule(lr){7-10}
Design & Direction & $\tau_N$ & $0.9$ & $0.7$ & $0.5$ & (i) & (ii) & (iii) & (iv) \\
\midrule
PBMC ($N=5000$) & 1 & 0.036 & 0.996 & 0.985 & 0.966 & 2.57 & 2.52 & 8.08 & 0.25 \\
 & 2 & 0.098 & 0.989 & 0.960 & 0.911 & 2.34 & 2.47 & 8.08 & 0.93 \\
 & 3 & 0.116 & 0.987 & 0.953 & 0.896 & 2.42 & 2.38 & 8.08 & 1.35 \\
 & 4 & 0.049 & 0.995 & 0.980 & 0.953 & 2.60 & 2.57 & 8.08 & 0.56 \\
 & 5 & 0.401 & 0.957 & 0.853 & 0.714 & 1.68 & 1.41 & 8.08 & 5.23 \\
 & 6 & 0.043 & 0.995 & 0.982 & 0.959 & 2.27 & 2.14 & 8.08 & 0.70 \\
 & 7 & 0.053 & 0.994 & 0.978 & 0.950 & 1.67 & 1.72 & 8.08 & 0.70 \\
 & 8 & 0.132 & 0.986 & 0.947 & 0.884 & 1.97 & 1.03 & 8.08 & 1.66 \\
& equal leverage & 0.040 & 0.996 & 0.983 & 0.962 & n/a & n/a & n/a & n/a \\
\midrule
TCGA ($N=873$) & 1 & 0.404 & 0.957 & 0.852 & 0.712 & 2.15 & 1.50 & 2.70 & 0.49 \\
 & 2 & 0.232 & 0.975 & 0.910 & 0.812 & 1.66 & 1.65 & 2.70 & 0.45 \\
 & 3 & 0.396 & 0.958 & 0.855 & 0.716 & 1.48 & 1.19 & 2.70 & 0.97 \\
 & 4 & 0.319 & 0.966 & 0.880 & 0.758 & 1.16 & 0.87 & 2.70 & 0.78 \\
 & 5 & 0.456 & 0.952 & 0.837 & 0.687 & 1.03 & 1.21 & 2.70 & 1.42 \\
 & 6 & 0.518 & 0.946 & 0.818 & 0.659 & 0.78 & 0.89 & 2.70 & 1.71 \\
 & 7 & 0.495 & 0.948 & 0.825 & 0.669 & 0.78 & 0.94 & 2.70 & 1.87 \\
 & 8 & 0.514 & 0.946 & 0.820 & 0.661 & 0.70 & 1.06 & 2.70 & 1.24 \\
& equal leverage & 0.229 & 0.975 & 0.911 & 0.814 & n/a & n/a & n/a & n/a \\
\bottomrule
\end{tabular}
\end{table}

\subsection{Semi-synthetic preprocessing}
\label{app:preprocessing}

\begin{table}[!ht]
\centering
\caption{Datasets used in semi-synthetic experiments, with the measured score--leverage couplings $\tau_N$ of \autoref{def:tau} for biological and random signal directions, the implied ceilings $\kappa_{\tau,N}$, and the equal-leverage value $\kappa$, all at $\rhox=0.7$. The biological columns give the leading complete-data direction, direction $1$ of \autoref{tab:kappasurvey}. The random columns give the mean $\pm$ SD of $\tau_N$ and the mean of $\kappa_{\tau,N}$ over the $30$ random directions of \autoref{fig:ceilings}(a).}
\label{tab:datasets}
\footnotesize
\setlength{\tabcolsep}{2.5pt}
\begin{tabular}{l l c c c c c c c}
\toprule
Dataset & Views & $N$ & $D_x=D_y$ & $\tau_N^{\mathrm{bio}}$ & $\kappa_{\tau,N}^{\mathrm{bio}}$ & $\tau_N^{\mathrm{rand}}$ & $\kappa_{\tau,N}^{\mathrm{rand}}$ & $\kappa$ \\
\midrule
TCGA BRCA & RNA-seq / methylation & 873 & 200 & 0.404 & 0.852 & $0.361 \pm 0.021$ & 0.866 & 0.911 \\
PBMC Multiome & scRNA-seq / scATAC-seq & 5000 & 200 & 0.036 & 0.985 & $0.171 \pm 0.020$ & 0.932 & 0.983 \\
\bottomrule
\end{tabular}
\end{table}

For TCGA BRCA we use the \texttt{HiSeqV2} gene-level expression matrix ($\log_2$-transformed normalized counts, as distributed) and the \texttt{HumanMethylation450} beta-value matrix, both from the TCGA Breast Cancer (BRCA) cohort on the UCSC Xena public hub \citep{goldman2020xena}, transposed to samples by features and paired by sample identifier ($N=873$). Features containing missing values are dropped, each remaining feature is standardized, each view is reduced to $200$ principal components, and the component scores are eigen-whitened: centered, rotated to the eigenbasis of the empirical covariance, and scaled so that $X_w^\top X_w = NI$.

For PBMC Multiome we use the filtered feature matrix of the 10k PBMC Multiome dataset from which granulocytes were removed by cell sorting \citep{tenxgenomics2021pbmcmultiome}. The RNA view is filtered to cells with at least $200$ detected genes and genes detected in at least $3$ cells, depth-normalized to $10^4$ counts per cell, log-transformed ($\log(1+x)$), and restricted to the $2000$ most highly variable genes. The ATAC view is filtered with the same thresholds and restricted to its first $5000$ peak features in file order, with no further transformation before standardization. The file orders peaks by chromosome and then by position, so these $5000$ peaks span positions $9790$ to $68{,}214{,}518$ of \texttt{chr1} in GRCh38. Cells present in both views are paired by barcode and subsampled to $N=5000$. Both views are then standardized, reduced to $200$ principal components, and eigen-whitened as above.

All preprocessing, including whitening, uses complete data, and the MCAR masks of the protocol in \autoref{sec:protocol} are applied afterwards, to the whitened design and the synthetic response.

\subsection{Aspect-ratio sensitivity}
\label{app:aspect_ratios}

\autoref{fig:aspect_ratios} tests the predicted curves of \autoref{thm:finitesignal} on Gaussian-QR designs across extreme full-column-rank aspect ratios.

\begin{figure}[!ht]
\centering
\includegraphics[width=\textwidth]{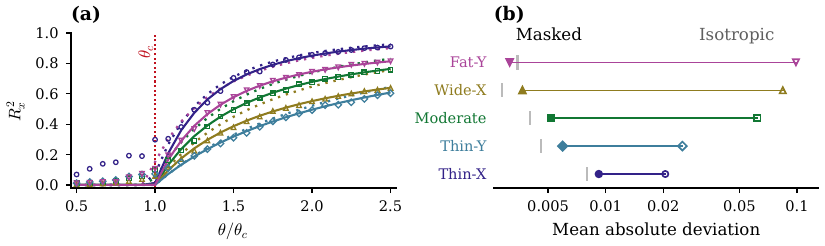}
\caption{%
\textbf{Aspect-ratio sensitivity.} \textbf{(a)} Five aspect-ratio regimes ($\alpha_x,\alpha_y\ge2$), one color each: markers show measurements over $100$ trials, solid curves show the masked prediction, and dotted curves show the isotropic comparator. Below the threshold both predictions are zero while measured recovery has begun to increase. \textbf{(b)} Mean absolute deviation from the measurement over $\theta>1.1\theta_c$: filled for the masked prediction, open for the comparator, gray for the trial-noise level, which the masked prediction stays near in every regime. All runs use $N=1000$ and $\rhox=\rhoy=0.7$. }
\label{fig:aspect_ratios}
\end{figure}

\paragraph{Extreme aspect ratios.} We use $N=1000$, $\rho_x=\rho_y=0.7$ ($\rho=0.49$), and $100$ trials at each of $25$ signal strengths in $[0.5,2.5]\theta_c$. The five full-column-rank configurations have the following supercritical MAEs: Thin-X ($\alpha_x=50$, $\alpha_y=6.7$, MAE $0.009$), Moderate ($\alpha_x=5$, $\alpha_y=6.7$, MAE $0.005$), Wide-X ($\alpha_x=2$, $\alpha_y=6.7$, MAE $0.004$), Thin-Y ($\alpha_x=5$, $\alpha_y=50$, MAE $0.006$), and Fat-Y ($\alpha_x=5$, $\alpha_y=2$, MAE $0.003$). All achieve $r \ge 0.996$. The isotropic comparator deviates most in the Wide-X and Fat-Y regimes, up to MAE $0.099$: Wide-X has the largest ceiling gap and Fat-Y the largest edge gap.

\subsection{Implementation of the baseline estimators}
\label{app:baselineimpl}

Rank-one response imputation centers both mean-imputed views, recenters the response after each fill, and stops after $50$ updates or once the absolute cosine between successive estimates of each direction exceeds $1-10^{-6}$. Iterative SVD starts from column means and stops after $50$ iterations or once the relative change of the filled matrix in Frobenius norm falls below $10^{-4}$. Nuclear-norm imputation starts from zero and stops after $200$ iterations or once the relative change of its low-rank estimate falls below $10^{-4}$. Its threshold grid is $\{1,0.99,0.97,0.95,0.92,0.90,0.85,0.80\}\,\lambda_0$ followed by $14$ geometrically spaced values from $0.70\lambda_0$ to $0.02\lambda_0$, and the selection fits them in decreasing order, each starting from the previous fit. Iterative SVD and nuclear-norm imputation use these limits in selection and in evaluation, and both apply PLS-SVD to the filled views without centering.

\clearpage
\section{Limitations}
\label{app:limitations}

The model assumes a rank-one signal, fixed MCAR retentions, i.i.d.\ Gaussian response noise, and whitening before masking, with $\Dx\le N$. The finite-signal theorem treats fixed $\theta\ne\theta_c$, leaving critical-window behavior open. For design masking with $\rho_x<1$, the finite-signal curves for QR/Haar designs remain universality predictions; proving the deterministic equivalent there requires control of a random variance profile. The finite-signal curves for biological designs are equal-leverage approximations, and persistent leverage variation requires a different deterministic equivalent. The high-signal theorem applies more generally under incoherence. The biological experiments combine measured designs with synthetic responses and known directions, so they test recovery within the specified response model. Structured missingness and multiple signal components remain open. Further priorities are whitening estimated from incomplete observations and characterizing recovery thresholds and ceilings in latent-factor models with additive noise in both views.

In the latent-factor model with additive noise in both views, a shared score generates the design and the response. The shared score couples the noisy design, corrupted direction, and product noise, so this specification does not supply the conditional structure used here. The present guarantees should therefore not be transferred to that setting without new analysis.

\section{Broader impacts}
\label{app:impacts}

Multimodal datasets with missing entries are the norm in computational biology and healthcare, and shared components extracted from them enter downstream analyses that are expensive to validate. For specified model parameters and a specified candidate direction, the model threshold and the direction-specific ceiling $\kappa_{\tau,N}$ are computable from the design, and \autoref{tab:kappasurvey} shows that on real designs the ceiling at moderate retention varies materially across the leading complete-data directions. Both are model-based reference quantities rather than guarantees: passing the threshold says recovery is possible within the model, not that the model matches the data.

The theory assumes MCAR masking and a design whitened before masking. Missingness in biomedical data is often structured, so applying the MCAR threshold there requires additional justification. The design-side ceiling $\kappa_\tau$, strictly below $1$ when $\rho_x<1$ and the coupling $\tau$ is positive, means design-side loadings are not consistent estimates even at arbitrarily strong signal.

\end{document}